\documentclass{article}

\PassOptionsToPackage{numbers, square}{natbib}
    \usepackage[preprint]{neurips_2025}

\usepackage[utf8]{inputenc} % allow utf-8 input
\usepackage[T1]{fontenc}    % use 8-bit T1 fonts
\usepackage{hyperref}       % hyperlinks
\usepackage{url}            % simple URL typesetting
\usepackage{booktabs}       % professional-quality tables
\usepackage{amsfonts}       % blackboard math symbols
\usepackage{nicefrac}       % compact symbols for 1/2, etc.
\usepackage{microtype}      % microtypography
\usepackage{xcolor}         % colors
\usepackage{amsmath}
\usepackage{amssymb}
\usepackage{graphicx}
\usepackage{subcaption}
\usepackage{algorithm}    
\usepackage{algorithmic}
\usepackage{amsmath}
\usepackage{amsthm}

\newcommand{\mc}[1]{\mathcal{#1}}
\newcommand{\mbf}[1]{\mathbf{#1}}
\newcommand{\mbb}[1]{\mathbb{#1}}

\newcommand{\E}{\mathbb{E}}
\newcommand{\R}{\mathbb{R}}
\newcommand{\N}{\mathbb{N}}

\newtheorem{theorem}{Theorem}[section]

\newtheorem{lemma}[theorem]{Lemma}
\newtheorem{corollary}[theorem]{Corollary}

\newtheorem{assumption}[theorem]{Assumption}
\newtheorem{remark}[theorem]{Remark}

\title{Asynchronous Decentralized SGD under Non-Convexity: A Block-Coordinate Descent Framework}

\author{%
  Yijie Zhou \\
  School of Data Science\\
  The Chinese University of Hong Kong, Shenzhen\\
  \texttt{yijiezhou@link.cuhk.edu.cn} \\
  \And
  Shi Pu \\
  School of Data Science\\
  The Chinese University of Hong Kong, Shenzhen\\
  \texttt{shipu@cuhk.edu.cn} \\
}

\begin{document}

\maketitle

\begin{abstract}
Decentralized optimization has become vital for leveraging distributed data without central control, enhancing scalability and privacy. However, practical deployments face fundamental challenges due to heterogeneous computation speeds and unpredictable communication delays. This paper introduces a refined model of Asynchronous Decentralized Stochastic Gradient Descent (ADSGD) under practical assumptions of bounded computation and communication times. To understand the convergence of ADSGD, we first analyze Asynchronous Stochastic Block Coordinate Descent (ASBCD) as a tool, and then show that ADSGD converges under computation-delay-independent step sizes. The convergence result is established without assuming bounded data heterogeneity. Empirical experiments reveal that ADSGD outperforms existing methods in wall-clock convergence time across various scenarios. With its simplicity, efficiency in memory and communication, and resilience to communication and computation delays, ADSGD is well-suited for real-world decentralized learning tasks.
\end{abstract}

\section{Introduction}
\label{sec:intro}
In the era of deep learning, especially with the dominance of Large Language Models, training datasets get larger and sometimes are spatially distributed. Consequently, centralized training is often not desired and even impossible due to either memory constraints or the decentralized nature of data. Decentralized optimization (DO), therefore, becomes a perfect remedy \cite{tang2023fusionai}. It aims to minimize the sum of local objective functions, i.e.,
\begin{equation}\label{eq:dec_opt}
    \min_{x\in \R^d} f(x) = \sum_{i=1}^n f_i(x),
\end{equation}
where $n$ is the number of agents and $d$ is the dimension of the problem. The optimization process is decentralized in that each agent only has access to the local objective function $f_i$. In deep learning, one typical form of $f_i$ is $f_i(x)\triangleq \E_{\xi\sim \mc{D}_i} F_i(x;\xi)$, where $\mc{D}_i$ represents the local data distribution of agent $i$, and $F_i$ is the loss function.

Most decentralized methods \cite{pu2020push,nedic2017achieving} use synchronous updates, suffering from stragglers in heterogeneous systems. Asynchronous approaches avoid this bottleneck and often perform better in practice \cite{samarakoon2019distributed}. However, existing asynchronous methods \cite{lian2018asynchronous,niwa2021asynchronous,bornstein2022swift,koloskova2020unified} typically require either partial synchronization or activation assumptions (e.g., independent sampling with fixed probabilities). For example, ADPSGD~\cite{lian2018asynchronous} imposes strict synchronization requirements: (1) agents must maintain identical update frequencies, and (2) neighbor synchronization is mandatory during updates. These constraints create significant waiting times during execution.

In this work, we propose Asynchronous Decentralized SGD (ADSGD) method and perform analysis under the assumptions that require only bounded computation/communication delays (Section \ref{subsec:assum}). The analysis connects ADSGD to Asynchronous Stochastic Block Coordinate Descent (ASBCD), providing new convergence guarantees for non-convex objectives.

\subsection{Related Work}
\label{sec:related_work}
This section reviews the literature on Asynchronous Block Coordinate Descent (ABCD) and asynchronous decentralized optimization algorithms. 
Note that most of the asynchronous algorithms make probabilistic assumptions regarding update patterns \cite{lian2018asynchronous,bornstein2022swift,koloskova2022sharper,liu2015asynchronous,peng2016arock,leblond2017asaga}. While these assumptions simplify theoretical analysis, they may not accurately approximate real-world scenarios. Here, we focus exclusively on algorithms that align with the same asynchrony assumptions as those adopted in this work.

% \subsubsection{ABCD methods}
\textbf{ABCD methods.} Readers might refer to \cite{sun2017asynchronous} for a slightly outdated review. More recently, several works \cite{kazemi2019asynchronous, ubl2022faster, zhou2018distributed} investigated the proximal block coordinate descent method. The work in \cite{ubl2022faster} considers the convex setting. While \cite{zhou2018distributed} extends to the non-convex case, additional assumptions such as the Luo-Tseng error bound condition \cite{tseng1991rate} is required.
% , which is too strong to even hold for $f(x) = |x|^\lambda, \lambda >2$,  and the KL-property. Unfortunately,
Note that these studies do not provide a convergence rate. In \cite{sun2017asynchronous}, a convergence rate of $o(\frac{1}{\sqrt{k}})$ is established for non-convex problems. The paper \cite{kazemi2019asynchronous} proposed an accelerated algorithm and achieved similar results as in \cite{sun2017asynchronous}. However, they have not considered the stochastic gradient setting.

% \subsubsection{Asynchronous decentralized methods}
\textbf{Asynchronous decentralized optimization methods.} Most asynchronous decentralized optimization methods study deterministic gradients, primarily using tracking-based approaches. For example, the works in \cite{cannelli2020asynchronous} and \cite{zhang2019fully} both achieve linear convergence, where APPG \cite{zhang2019fully} assumed P\L-condition and \cite{cannelli2020asynchronous} used the Luo-Tseng error
bound condition \cite{tseng1991rate}. The work in \cite{tian2020achieving} achieves sublinear convergence for general non-convex functions, but suffers from: (1) step sizes scaling as $\mc{O}(\underline w^{(2n-1)B+nD})$\footnote{$\underline w$ is the lower bound of the weights in the weight matrix}, $B, D$ are the bounds of computation and communication delays, respectively. (2) heavy memory/communication overhead due to gradient tracking.

While the paper \cite{wu_delay-agnostic_2023} proves delay-agnostic convergence for asynchronous DGD with exact gradients under strongly convex objectives, extending the result to non-convex problems with stochastic gradients is non-trivial. Specifically, the max-block pseudo-contractive analysis fails for stochastic gradients due to non-commutativity of max and expectation operations.

Under the adopted asynchrony assumptions of bounded computation/communication delays, existing methods using stochastic gradients either focus on strongly convex objectives \cite{spiridonoff2020robust} or impose strong constraints for non-convex cases. For example, under a simple case - a 3-agent fully-connected network with no delays and 1-smooth loss - the tracking-based methods in \cite{zhu2023robust} and \cite{kungurtsev2023decentralized} theoretically require step sizes below $2.2\times10^{-27}$ and $3.5\times 10^{-54}$. In contrast, the theoretical step size for ADSGD has a clear and simple dependency on $D$ and $K$ (total iteration number) only. Moreover, tracking-based methods require over three times as much memory in practice and double the communication budget compared with ADSGD. Readers can refer to the appendix for a detailed comparison.

\subsection{Main Contributions}
\label{sec:contribution}
For non-convex objectives, this paper proposes the ADSGD method and analyze its performance under bounded computation/communication delays, generalizing the equivalence between DGD and BCD \cite{zeng2018nonconvex} to asynchronous stochastic gradient settings. The key contributions of this work include:

\begin{itemize}
    \item We establish that ASBCD converges under non-convexity at rate $\mathcal{O}(1/K^{1/2})$, matching that of standard SGD. This is the first such result for asynchronous coordinate descent with stochastic gradients.

    \item We introduce ADSGD and prove that, as a special case of ASBCD, the method converges with computation-delay-independent step sizes for non-convex functions - the first such guarantee. Compared to existing methods, ADSGD reduces per-iteration communication costs by 50\% and memory usage by 70\%. Moreover, its convergence does not depend on the commonly assumed bounded data heterogeneity assumption.

    \item We demonstrate empirically that ADSGD converges faster than all baselines (including synchronous methods) regardless of stragglers, showing that the method is delay-resilient, communication/memory-efficient, and simple to implement, ideal for practical deployment.
\end{itemize}

\section{The algorithms}
This section details the ASBCD and ADSGD algorithms.
\subsection{ASBCD}
Consider the optimization problem
\begin{equation}\label{eq:optimization}
    \min_{\mbf x\in \R^{d'}} \mbf{f}(\mbf x),
\end{equation}
where $\mbf x = (x_1^T,...,x_n^T)^T, x_i \in \R^{d_i}$, and $\sum_{i=1}^n d_i = d'$. 

In ASBCD, each block stores part of the global model $x_i$ and a buffer $\mathcal{B}_i$, containing all blocks of the global model $\{x_{ij}\}_{j\in [n]}$. Specifically, $x_i$ is the current local iterate of block $i$ and $x_{ij}$ records the most recent $x_j$ it received from block $j\in\mc{N}_i$. As in Algorithm \ref{alg:ASBCD}, each block keeps estimating $\nabla_{i} \mbf{f}(\cdot)$ with the global model in the buffer. Once the gradient estimation is available, block $i$ updates as follows: 
\begin{equation}
    x_i \leftarrow x_i - \alpha g^{\mbf f}_i(\mbf{x}_i),
    \label{eq:ASBCD_update_SYMBOLIC}
\end{equation}
where $\mbf{x}_i=(x_{i1}^T,...,x_{in}^T)^T$ and $g^\mbf{f}_{i}(\cdot)$ is a stochastic estimator of $\nabla_{i} \mbf f(\cdot)$. Then block $i$ sends the updated block to all other blocks and repeats. 

For a clear mathematical representation, we introduce the virtual iteration index $k\in\N_0$, which is only for analysis and need not be known by any block. The index $k$ is increased by 1 whenever some block is updated.
The updating rule can be written as
% \begin{equation}
% \label{eq:ABCD_update}
%     x^{k+1}_{i_k} = x^k_{i_k} - \alpha g^f_{i_k}(\hat{\mbf{x}}^k_{i_k}),
% \end{equation}

\begin{equation}\label{eq:ABCD_update}
    x_i^{k+1} = \begin{cases}
        x^k_{i} - \alpha g^\mbf{f}_{i}(\hat{\mbf{x}}^k), & i=i_k,\\
        x_i^k, & \text{otherwise},
    \end{cases}
\end{equation}
where $i_k$ is the active block at step $k$ and $\hat{\mbf{x}}^k$ is the global model held by block $i_k$ at iteration $k$. Note that the read of the global model is done prior to gradient calculation, hence $\hat{\mbf{x}}^k$ is the available global model to $i_k$ when it begins the gradient estimation before iteration $k$. For instance, suppose a block starts gradient estimation at iteration 2 and takes 3 iterations to finish. Then $\hat{\mbf{x}}^{5}$ is the global model available to $i_{5}$ at iteration 2. Specifically, $\hat{\mbf{x}}^k = ((x^{s_{i_k 1}^k}_1)^T,...,(x^{s_{i_k n}^k}_n)^T)^T$ and $s_{ij}^k \le k$ is the largest iteration index (but smaller than or equal to $k$) of the most recent version of $x_j$ available to block $i$ when it starts its last gradient estimation prior to iteration $k$. In other words, $s_{ij}^k$ is the minimal value of the current iteration index $k$ and the iteration when block $j$ conducts its next update. Note that $s_{ij}^k$ degenerates to a simpler term under certain circumstances, e.g. $s_{ii}^k=k$.

% We specify the definition of $s_{ij}^k$ as follows. For any $i,j,k$, let $k'$ be the iteration at which $x_j^{s_{ij}^k}$ is updated and $k''$ be the iteration when $j$ conducts its next update. Namely, $x_j^{s_{ij}^k} = x_j^t$, for $ k'\le t < k''$. We define $s_{ij}^k \triangleq \min\{k'',k\}$.
% % Such a definition guarantees that the mapping $s_{ij}^k$ is a bijection and 

Fig. \ref{fig:schematic_ASBCD_1} and \ref{fig:schematic_ASBCD_2} provide examples of $s_{ij}^k$. In both figures,  block $i_9$ starts calculating gradient at iteration 2 and 4 and finishes at iteration 4 and 9. We first focus on $s_{i_91}^9$, which is the largest iteration index of the most recent $x_1$ available when $i_9$ starts its last gradient estimation at iteration 4. It is shown in both figures that $x_1$ is available at iteration 2. However, block 1 updates at iteration 5 and 3 in Fig. \ref{fig:schematic_ASBCD_1} and \ref{fig:schematic_ASBCD_2}, respectively. Thus $s_{i_91}^9=5$ for Fig. \ref{fig:schematic_ASBCD_1} and $s_{i_91}^9=3$ for Fig. \ref{fig:schematic_ASBCD_2}. Next, consider $s_{i_4 1}^4$, which is the largest iteration index of the most recent $x_1$ available at iteration 2. Now, $s_{i_4 1}^4=\min\{5,4\}=4$ for Fig. \ref{fig:schematic_ASBCD_1} because block 1 remains unchanged until iteration 5. Similarly, $s_{i_4 1}^4=\min\{3,4\}=3$ for Fig. \ref{fig:schematic_ASBCD_2}.

\begin{figure*}[t]
\vspace{-0.3cm}
\centering
\begin{subfigure}[t]{0.28\textwidth}
\centering
\includegraphics[width=\textwidth]{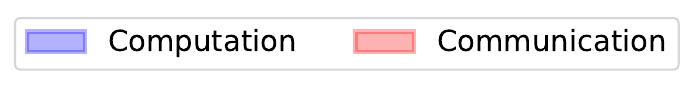}
\end{subfigure}
\vspace{-0.1cm}
\\
% \hfill
\hspace{-0.4cm}
\begin{subfigure}[t]{0.265\textwidth}
\centering
\includegraphics[width=\textwidth]{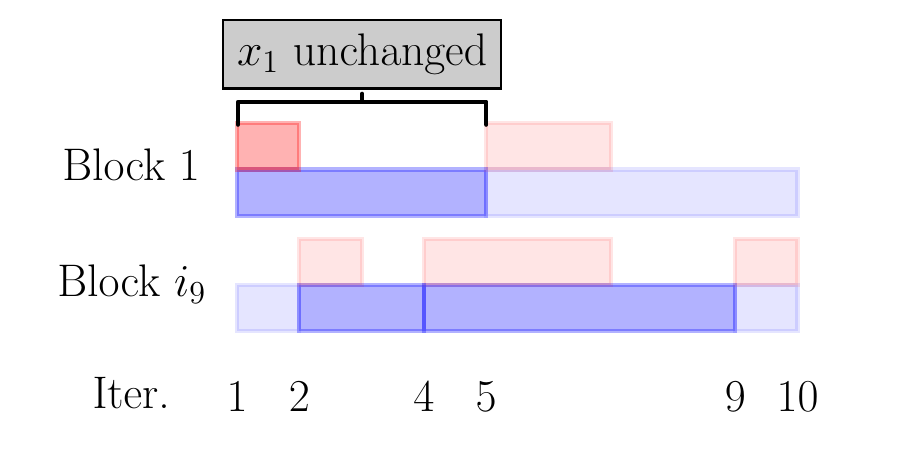}
\vspace{-0.6cm}
\caption{ASBCD: $s_{i_91}^9=5$}
\label{fig:schematic_ASBCD_1}
\end{subfigure}
% \hfill
\hspace{-0.5cm}
\begin{subfigure}[t]{0.265\textwidth}
\centering
\includegraphics[width=\textwidth]{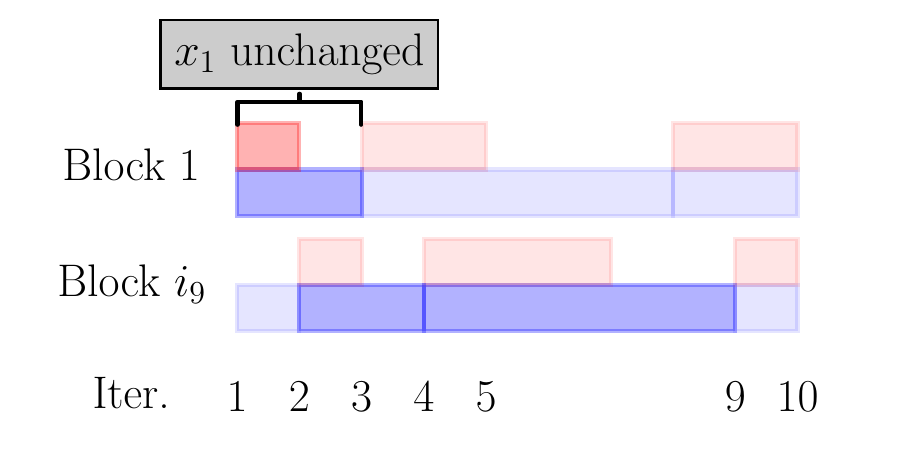}
\vspace{-0.6cm}
\caption{ASBCD: $s_{i_91}^9=3$}
\label{fig:schematic_ASBCD_2}
\end{subfigure}
% \hfill
\hspace{-0.5cm}
\begin{subfigure}[t]{0.265\textwidth}
\centering
\includegraphics[width=\textwidth]{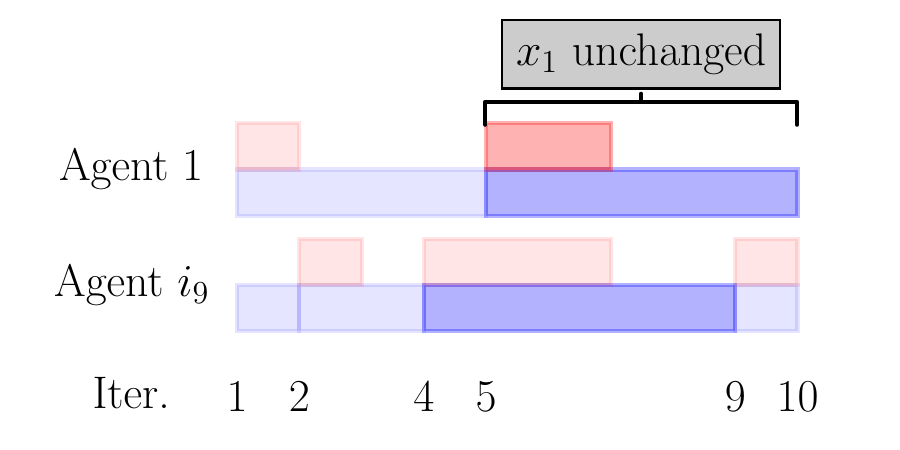}
\vspace{-0.6cm}
\caption{ADSGD: $s_{i_91}^9=9$}
\label{fig:schematic_ADSGD_1}
\end{subfigure}
\hspace{-0.5cm}
\begin{subfigure}[t]{0.265\textwidth}
\centering
\includegraphics[width=\textwidth]{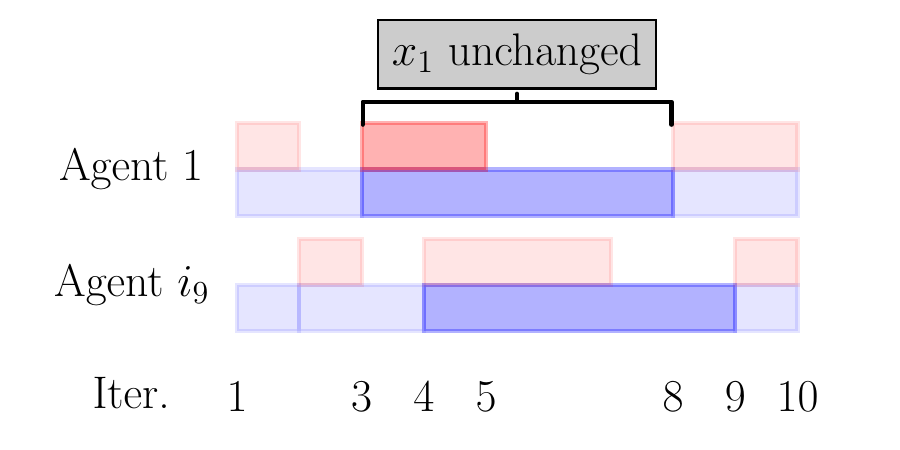}
\vspace{-0.6cm}
\caption{ADSGD: $s_{i_91}^9=8$}
\label{fig:schematic_ADSGD_2}
\end{subfigure}
% \hfill
% \vspace{-0.3cm}
\caption{Schematics of $s_{ij}^k$ for ASBCD (left) and ADSGD (right), with the computation and communication of primary focus highlighted in a darker shade.}
\label{fig:schematic}
  % \vspace{-0.4cm}
\end{figure*}

\begin{remark}
    $s_{ij}^k$ here captures both computation delay and communication delay. The term $k- s_{ij}^k$ will be large when either $i$ is slow at computing (while $j$ is fast) or the communication link from $j$ to $i$ is slow. If $\{i_k\}=[n]$ (all blocks update together) and $s_{ij}= k$ for all $i,j,k$, \eqref{eq:ABCD_update} degenerates to standard stochastic gradient descent.
\end{remark}

\begin{algorithm}[tb]
\caption{ASBCD}
\label{alg:ASBCD}
\begin{algorithmic}[1]
\STATE \textbf{Initialization:}  All blocks agree on $\alpha>0$. 
\STATE Each block chooses $x_i$, creates a local buffer $\mc{B}_i$, shares $x_i$, and calculates $g^{\mbf{f}}_i(\mbf{x}_i)$.
% \vspace{0.5\baselineskip} 
\STATE \textbf{All Blocks Do In Parallel:}
\WHILE{The termination criterion is met}
    % \STATE Read all other blocks of the global model and save as $\{x_{ij}\}_{j\in [n]}$ in its buffer.
    % \STATE Calculate $g_{i}^f ((x_{i1},...,x_{in}))$.
    % \STATE Update $x_i$ according to \eqref{eq:ASBCD_update_SYMBOLIC}.
    \REPEAT
        \STATE Keep receiving $x_j$ from other blocks.
        \STATE Let $x_{ij} = x_j$ and store $x_{ij}$ in $\mc{B}_i$
    \UNTIL{$g^{\mbf{f}}_i(\mbf{x}_i)$} is available.
    \STATE Update $x_i$ according to \eqref{eq:ASBCD_update_SYMBOLIC}. 
    \STATE Send $x_i$ to every other block.
    \STATE Calculate $g^{\mbf{f}}_i(\mbf{x}_i)$.
\ENDWHILE

\end{algorithmic}
\end{algorithm}

% The read version
% \begin{algorithm}[tb]
% \caption{ASBCD}
% \label{alg:ASBCD}
% \begin{algorithmic}[1]
% \STATE \textbf{Initialization:}  All blocks agree on $\alpha>0$. Each block chooses $x_i\in\mathbb{R}^d$ and creates a local buffer $\mc{B}_i$.
% \vspace{0.5\baselineskip} 
% \STATE \textbf{All Blocks Do In Parallel:}
% \WHILE{The termination criterion is met}
%     \STATE Read all other blocks of the global model and save as $\{x_{ij}\}_{j\in [n]}$ in its buffer.
%     \STATE Calculate $g_{i}^f ((x_{i1},...,x_{in}))$.
%     \STATE Update $x_i$ according to \eqref{eq:ASBCD_update_SYMBOLIC}.

% \ENDWHILE

% \end{algorithmic}
% \end{algorithm}

\subsection{ADSGD}
\label{sec:ADSGD}
ADSGD aims to solve \eqref{eq:dec_opt} over a network of $n$ agents described by an undirected, connected graph 
$\mc{G}=(\mc{V}, \mc{E})$, where $\mc{V}=\{1,\ldots,n\}$ is the vertex set and $\mc{E}\subseteq \mc{V}\times \mc{V}$ is the edge set. In the network, each agent $i$ observes a local cost function $f_i:\mathbb{R}^d\rightarrow \mathbb{R}$ and can only interact with its neighbors in $\mc{N}_i=\{j: \{i,j\}\in\mc{E}\}$. 

In asynchronous DSGD, each node behaves similarly as in ASBCD. Every node $i\in\mc{V}$ holds the model of itself $x_i$ and a buffer $\mc{B}_i$, which contains the models of its neighbors $\{x_{ij}\}_{j\in \mc{N}_i}$. Like before, $x_i$ is the current local iterate of node $i$, and $x_{ij}$ records the most recent $x_j$ it received from node $j\in\mc{N}_i$. Each agent keeps estimating $\nabla f_i(x_i)$ and updates as follows once finished,
\begin{equation}
    x_i \leftarrow w_{ii}x_i + \sum_{j\in\mc{N}_i} w_{ij} x_{ij} - \alpha g^{f_i}(x_i),
    \label{eq:ADSGD_update_SYMBOLIC}
\end{equation}
 where $g^{f_i}(x_i)$ is a stochastic estimator of $\nabla f_i(x_i)$, and $w_{ij}$ is the $(i,j)$-th entry of the weight matrix $W$. Similar to ASBCD, node $i$ broadcasts $x_i$ to all neighbors, and its neighbor $j$ overwrites $x_{ji}$ in its buffer $\mathcal{B}_j$. A detailed implementation is given in Algorithm \ref{alg:ADSGD}.

Likewise, the iterates are indexed by $k\in\mbb{N}_0$, which increases by $1$ whenever an update is performed on a local variable $x_i$ of some nodes $i\in\mc{V}$. Again, $k$ is only for analysis and need not be known by agents. Denote $\bar{\mc{N}}_i=\mc{N}_i\cup\{i\}$ for all $i\in\mc{V}$. Then, the asynchronous DSGD can be described as 
\begin{equation}\label{eq:ADSGD_update}
    x_i^{k+1} = \begin{cases}
        \sum_{j\in\bar{\mc{N}_i}} w_{ij}x_j^{s_{ij}^k} - \alpha g^{f_i}(x_i^k), & i=i_k,\\
        x_i^k, & \text{otherwise},
    \end{cases}
\end{equation}
where $s_{ij}^k\in [0, k]$ for $j\in\bar{\mc{N}}_i$ is the largest iteration index (smaller than or equal to $k$) of the most recent version of $x_j$ available to node $i$ at iteration $k$ (instead of ``when $i$ begins its final gradient estimation prior to
iteration $k$'' as in ASBCD). Still, $s_{ij}^k$ is the minimal value of the current iteration index $k$ and the time agent $j$ updates next as previously defined in ASBCD.

Similar schematics of $s_{ij}^k$ are shown in Fig. \ref{fig:schematic_ADSGD_1} and \ref{fig:schematic_ADSGD_2}.
The behaviors of both agents in Fig. \ref{fig:schematic_ADSGD_1} are identical to that of Fig. \ref{fig:schematic_ASBCD_1}, and so is Fig. \ref{fig:schematic_ADSGD_2} to Fig. \ref{fig:schematic_ASBCD_2}. Now $s_{i_91}^9$ is the largest iteration index of available $x_1$ when agent $i_9$ finishes gradient estimation at iteration 9. Note that, in Fig. \ref{fig:schematic_ADSGD_1},  the $x_1$ updated at iteration 5 is available to agent $i_9$ prior to iteration 9, and $x_1$ remains identical until iteration 10. Therefore, $s_{i_91}^9=\min\{10, 9\}=9$ for Fig. \ref{fig:schematic_ADSGD_1}. Likewise, in Fig. \ref{fig:schematic_ADSGD_2},  the $x_1$ updated at iteration 3 is available to agent $i_9$ before iteration 9. Thus, $s_{i_91}^9=\min\{8, 9\}=8$ for Fig. \ref{fig:schematic_ADSGD_2}.

% Its only difference to Fig. \ref{fig:schematic_ASBCD_1} and \ref{fig:schematic_ASBCD_2} is the time agent 1 finishes communication. Agent 1 finishes communication after the iteration when agent $i_8$ finishes its last gradient estimation (iteration 4). In the case of ASBCD, agent $i_8$ would have to use the $x_1$ available in the previous communication round. However, for ADSGD, $x_1$ is not used until iteration 8, when $i_8$ finishes the current gradient estimation. Therefore, at iteration 8, agent $i_8$ will use the $x_1$ sent to him at iteration $5$. And $s_{i_81}^8$ equals 8 or 7 depending on agent 1's next update iteration.

% \begin{figure}[t]
% \vspace{-0.3cm}
% \centering
% \begin{subfigure}[t]{0.26\textwidth}
% \centering
% \includegraphics[width=\textwidth]{pic/schematic_legend_image.pdf}
% \end{subfigure}
% \vspace{-0.1cm}
% \\
% \begin{subfigure}[t]{0.25\textwidth}
% \centering
% \includegraphics[width=\textwidth]{pic/ADSGD_schematic_1.pdf}
% \vspace{-0.6cm}
% \caption{$s_{i_81}^8=8$}
% \label{fig:schematic_ADSGD_1}
% \end{subfigure}
% \hspace{-0.55cm}
% \begin{subfigure}[t]{0.25\textwidth}
% \centering
% \includegraphics[width=\textwidth]{pic/ADSGD_schematic_2.pdf}
% \vspace{-0.6cm}
% \caption{$s_{i_81}^8=7$}
% \label{fig:schematic_ADSGD_2}
% \end{subfigure}
% \vspace{-0.25cm}
% \caption{Schematic of $s_{ij}^k$ for ADSGD}
% \label{fig:schematic_ADSGD}
%   % \vspace{-0.4cm}
% \end{figure}

\begin{remark}
    Unlike in ASBCD, $s_{ij}^k$ here captures only the communication delay. The difference originates from the time the global model is used. The global model is used before gradient estimation in ASBCD and after gradient estimation in ADSGD. 
    When $i_k = [n]$ and $s_{ij}^k=k$, for all $i,j,k$, \eqref{eq:ADSGD_update} reduces to the synchronous DSGD. 
\end{remark}

% \begin{algorithm}[tb]
% \caption{ADSGD}
% \label{alg:ADSGD}
% \begin{algorithmic}[1]
% \STATE \textbf{Initialization:}  All the nodes agree on $\alpha>0$, and cooperatively set $w_{ij}$ $\forall \{i,j\}\in\mc{E}$.
% \STATE  Each node $i\in\mc{V}$ chooses $x_i\in\mathbb{R}^d$, creates a local buffer $\mc{B}_i$, shares $x_i$ with all neighbors in $\mc{N}_i$, and calculates $g^{f_i}(x_i)$.

% \FOR{each node \(i \in \mathcal{V}\)}
%     \STATE Keep receiving $x_j$ from neighbors.
%     \STATE Let $x_{ij} = x_j$ and store $x_{ij}$ in $\mc{B}_i$.
%     \IF{$g^{f_i}(x_i)$ is available}
%         \STATE Update $x_i$ according to \eqref{eq:ADSGD_update_SYMBOLIC}.
%         \STATE Send $x_i$ to all neighbors 
%         $j\in\mc{N}_i$.
%         \STATE Calculates $g^{f_i}(x_i)$
%     \ENDIF
% \ENDFOR

% \STATE \textbf{Repeat} until a termination criterion is met.

% \end{algorithmic}
% \end{algorithm}

\begin{algorithm}[tb]
\caption{ADSGD}
\label{alg:ADSGD}
\begin{algorithmic}[1]
\STATE \textbf{Initialization:}  All the nodes agree on $\alpha>0$, and cooperatively set $w_{ij}$, $\forall \{i,j\}\in\mc{E}$.
\STATE Each node chooses $x_i$, creates a buffer $\mc{B}_i$, shares $x_i$ with neighbors, and calculates $g^{f_i}(x_i)$. 
% \vspace{0.5\baselineskip} 
\STATE \textbf{All Nodes Do In Parallel:}
\WHILE{the termination criterion is not met}
    \REPEAT
        \STATE Keep receiving $x_j$ from neighbors.
        \STATE Let $x_{ij} = x_j$ and store $x_{ij}$ in $\mc{B}_i$
    \UNTIL{$g^{f_i}(x_i)$ is available.}
    \STATE Update $x_i$ according to \eqref{eq:ADSGD_update_SYMBOLIC}. 
    \STATE Send $x_i$ to all neighbors $j\in \mc{N}_i$.
    \STATE Calculate $g^{f_i}(x_i)$.
\ENDWHILE
\end{algorithmic}
\end{algorithm}

\subsection{Notation}
With a slight abuse of notation, $[x]_i$ refers to taking the $i$-th block of a vector $x$. For instance, if $\mbf x \triangleq [x_1,...,x_n]$, then $[\mbf x]_i = x_i$, where $x_i$ can either be a scalar or a vector.
% Likewise, $[x]_{ij}$ refers to the $ij$-th block of $x$.
For a matrix $W$ with $n$ eigenvalues, $\lambda_i(W)$ denotes the $i$-th largest eigenvalue. Additionally, we define $\mbf{W}=W\otimes I_d$ and $(k)^+\triangleq \max\{0,k\}$.

\section{Convergence Analysis}
\label{sec:convergence}
\subsection{Assumptions}
\label{subsec:assum}
All assumptions are summarized below. Assumption \ref{asm:b-l_bounded-smooth} - \ref{asm:b-grad_est} are for ASBCD, while their counterparts Assumption \ref{asm:dsgd-l_bounded-smooth} - \ref{asm:dsgd-grad_est} are for ADSGD. The assumption on partial asynchrony (Assumption \ref{asm:partialasynchrony}) is shared.

\begin{assumption}\label{asm:b-l_bounded-smooth}
    $\mbf{f}$ is $L$-smooth and lower bounded by $\mbf{f}^*$.
\end{assumption}
% We denote $L_\max = \max_i {L_i}$ and $L = \sum_i L_i$. Note that $f$ is $L$-Lipschitz continuous.

\begin{assumption}\label{asm:b-grad_est}
    For each block $i$, the gradient estimator is unbiased with bounded variance. I.e., $\mbb{E}[g_i^\mbf{f}(\mbf x) - \nabla_i \mbf{f}(\mbf x)]=0$ and $\mbb{E}[\|g_i^\mbf{f}(\mbf x) - \nabla_i \mbf{f}(\mbf x)\|^2]\le \sigma^2, \forall i,\mbf x$.
\end{assumption}

\begin{assumption}[Asynchrony]\label{asm:partialasynchrony}
    There exist positive integers $B$ and $D$ such that
    \begin{enumerate}
    % \vspace{-0.28cm}
	\item For every $i\in\mc{V}$ and for every $k\ge 0$, there exists $m \in \{k,\ldots,k+B-1\}$ such that $i_m = i$.
	  \item There holds $k-D \le s_{ij}^k \le k$ for all $i\in\mc{V}$, $j\in\mc{N}_i$, and $k\in \mc{K}_i$. 
    % \vspace{-0.15cm}
    \end{enumerate}
\end{assumption}
\begin{remark}
    Assumption \ref{asm:partialasynchrony}.1 requires each agent to update at least once every $B$ steps. Therefore, $B$ resembles the computation delay. Assumption \ref{asm:partialasynchrony}.2 implies different conditions for ADSGD and ASBCD. For ADSGD, $D$ represents the maximum communication delay, as the agents mix the neighboring models after the gradient is available. Whereas for ASBCD, $D$ encompasses both computation and communication delay, as each block takes the whole model before gradient estimation. Therefore, when there is no communication delay, $D$ for ADSGD degenerates to $0$ while $D$ for ASBCD degenerates to $B$. In practice, the assumption holds as long as the computation and communication times of each agent are lower and upper bounded.
\end{remark}

\begin{assumption}\label{asm:dsgd-l_bounded-smooth}
    For each agent, the objective function $f_i$ is $L_i$-smooth and lower bounded by $f_i^*$.
\end{assumption}

\begin{assumption}\label{asm:dsgd-grad_est}
    For each agent, the gradient estimator is unbiased with bounded variance. I.e., $\mbb{E}[g^{f_i}(x) - \nabla f_i(x)]=0$ and $\mbb{E}[\|g^{f_i}(x) - \nabla f_i(x)\|^2]\le \sigma^2, \forall i,x$.
\end{assumption}

\begin{assumption}\label{asm:W}
    The mixing matrix $W$ is stochastic and symmetric, with the corresponding communication graph being undirected and connected.
\end{assumption}

\subsection{Convergence of ASBCD}

The following lemma generalizes [Theorem 1, \cite{sun2017asynchronous}] to stochastic gradients, forming the basis for ADSGD's analysis:

\begin{lemma}
\label{the:ABCD_convergence}
    For problem \ref{eq:optimization}, given Assumption \ref{asm:b-l_bounded-smooth} - \ref{asm:partialasynchrony}, and $\alpha < \frac{1}{(D + 1/2)L}$, the sequence $\{\mbf x^k\}$ generated by (\ref{eq:ABCD_update}) satisfies the following relation:
\begin{align*}
% \label{eq:ASBCD_theorem}
    &\frac{\sum_{k=0}^{K-1}\E\|\nabla \mbf{f}(\mbf x^k)\|^2}{K} \le \frac{3n(B+C_0L^2 \alpha^2)}{\alpha(1 - (\frac{L}{2}+D L)\alpha)} \frac{\mbf{f}(\mbf x^0) - \mbf{f}^*}{K} + \alpha \left(3nC_0L^2\alpha + \frac{L(D+1)}{2(1 - (\frac{L}{2}+D L)\alpha)}\right)\sigma^2,
\end{align*}
where $C_0 = D^2+3B^2(D + 2D^3)$.
\end{lemma}

Lemma \ref{the:ABCD_convergence} indicates ASBCD converges to the $\mc{O}(\alpha)$ neighborhood of its stationary points with a rate of $\mc{O}(\frac{1}{K})$, matching the standard rate for non-convex SGD. Note that the step size here depends on $D$, encompassing computation and communication delays. A corollary is provided in the appendix.

\subsection{Convergence of ADSGD}

The paper \cite{zeng2018nonconvex} pointed out that DGD is a special case of Block Coordinate Descent. In the following, such an equivalence is generalized to the asynchronous and stochastic gradient setting.

Define $F(\mbf{x}) \triangleq \sum_i f_i(x_i)$, and $L_\alpha(\mbf{x})= F(\mbf{x}) + \frac{\mbf{x}^T (I-\mbf W) \mbf{x}}{2\alpha}$. Note that $F(\mbf{x})$ and  $L_\alpha(\mbf{x})$ are both  Lipschitz smooth with $L_F\triangleq \max L_i$ and $L_L\triangleq L_F+\frac{1-\lambda_n}{\alpha}$, respectively
(note that $L_L$ is a function of the step size $\alpha$).

ADSGD on $F(\mbf{x})$ can be viewed as ASBCD on $L_\alpha(\mbf{x})$. The updating rule \eqref{eq:ADSGD_update} can be rewritten as
    \begin{equation*}
    x_i^{k+1} = \begin{cases}
        x^k_{i} - \alpha g^{L_\alpha}_{i} (\hat{\mbf{x}}^k), & i=i_k,\\
        x_i^k, & \text{otherwise},
    \end{cases}
\end{equation*}
    where $g^{L_\alpha}_{i} (\cdot)$ is an stochastic estimate of $\nabla_{i}L_\alpha(\cdot)$ and $\hat{\mbf{x}}^k = ((x^{s_{i_k 1}^k}_1)^T,...,(x^{s_{i_k n}^k}_n)^T)^T$. Note that even though the above is expressed under the framework of ASBCD, $s_{ij}^k$ should follow the one defined for ADSGD (Section \ref{sec:ADSGD}). This is because, in ADSGD, the neighbors' models are only used after gradient estimation. Therefore, $\hat{\mbf{x}}^k$ should be the global model available to $i_k$ at iteration $k$.

% As shown above, ADSGD is a special case of ASBCD, with a step-size-dependent Lipschitz constant. Such a dependency poses difficulties in deriving convergence rate for ADSGD. To be precise, the accumulation of noise now cannot be mitigated by only controlling the step size. We thus introduce a double-step size technique as a workaround.

ADSGD constitutes a special case of ASBCD with step-size-dependent Lipschitz continuity. This dependency complicates convergence analysis, as step size alone cannot control noise accumulation. To address this limitation, we propose a double-step-size technique, where an extra step size $\beta$ is introduced such that
\begin{align}
\label{eq:ADSGD_double_step size}
     x^{k+1}_{i_k} &= x^k_{i_k} - \beta g_{i_k}^{L_\alpha}(\hat{\mbf{x}}^k)\nonumber\\
    &= (1-\frac{\beta}{\alpha})x^k_{i_k} + \frac{\beta}{\alpha}[\mbf{W}\hat{\mbf{x}}^k]_{i_k} - \beta  g_{i_k}^{F}(\hat{\mbf{x}}^k)\nonumber\\
    &=[\mbf{\tilde W} \hat{\mbf{x}}^k]_{i_k} - \beta  g^{f_{i_k}}(x^k_{i_k}),
\end{align}
where $\mbf{\tilde W} \triangleq \tilde W \otimes I_d$, and $\tilde W_{ii} = (1-\frac{\beta}{\alpha})+ \frac{\beta}{\alpha}W_{ii}$, $\tilde W_{ij} = \frac{\beta}{\alpha} W_{ij}, i\ne j$. The double-step size ADSGD algorithm is effectively Algorithm \ref{alg:ADSGD} with a different weight matrix $\tilde W$ and the step size $\beta$, since $\tilde W$ satisfies Assumption \ref{asm:W}. Via the ASBCD-ADSGD correspondence, we transfer ASBCD's convergence to Algorithm~\ref{alg:ADSGD}. Theorem~\ref{the:ADSGD} shows convergence under proper $\{\alpha, \beta\}$.

\begin{theorem}\label{the:ADSGD}
    For Problem \ref{eq:dec_opt}, given Assumption \ref{asm:partialasynchrony} - \ref{asm:W} and $\beta<\frac{1}{(D+1/2)L_L}$, the sequence $\{x^k\}$ generated by \eqref{eq:ADSGD_double_step size} satisfies:
\begin{align*}
% \label{eq:ADSGD_theorem}
    &\frac{\sum_{k=0}^{K-1}\E\|\nabla f(\bar x^k)\|^2}{K} \le \frac{6n^2(B+C_0L_L^2 \beta^2)}{\beta(1 - (D+\frac{1}{2})L_L\beta)} \frac{\sum_{i=1}^n(f_i(x^0) - f_i^*)}{K}\nonumber\\
    &+ L_L \beta\left(6n^2C_0L_L\beta + \frac{n(D+1)}{1 - (D+\frac{1}{2})L_L\beta}\right)\sigma^2\\
    &+ \frac{4n(\max_i L_i)^2\alpha}{1-\lambda_2 (W)}\left(\sum_{i=1}^n (f_i(x^0) - f_i^*) + \frac{D+1}{2} KL_L\beta^2\sigma^2\right)
\end{align*}
    
    where $\bar x^k = \frac{1}{n}\sum_{i=1}^n x_i^k$ and $C_0$ is defined in Lemma \ref{the:ABCD_convergence}.
\end{theorem}
\begin{remark}
    Unlike \cite{zhu2023robust,kungurtsev2023decentralized}, the proposed step size $\beta$ is independent of computation delay $B$ and permits much larger values than prior asynchronous decentralized non-convex SGD methods. This independence arises because each agent i computes gradients solely using its local model $x_i^t$, which isn't affected by communication delays. 
\end{remark}
% The independence between step size and $B$ can be understood intuitively. In ADSGD, agent $i$ only uses the local model $x_i$, which can only be modified by itself, to estimate the gradient. Since there is no communication delay within the agent, the gradient estimation is always on a fresh model. The dependency on $D$ is from bounding the staleness (Lemma \ref{le:staleness_err_for_block_alg}) and can hardly be avoided.

% whereas for ASBCD, \hat x will be different from x, making the gradient estimation staled.
\begin{remark}
    In contrast to previous convergence analyses of DSGD~\cite{koloskova2020unified}, current analysis does not rely on bounded data heterogeneity assumption. It only requires each individual loss to be bounded below, which is almost always satisfied in practice.
\end{remark}

% \begin{corollary}
% \label{coro:ADSGD}
%      For Problem \ref{eq:dec_opt}, given Assumption \ref{asm:partialasynchrony} - \ref{asm:W} and let $\alpha = \frac{1}{K^{1/3}}, \beta = \frac{1}{K^{2/3}}$, the sequence $\{x^k\}$ generated by \eqref{eq:ADSGD_double_step size} satisfies:
%     \begin{equation*}
%         \frac{\sum_{k=0}^{K-1}\E\|\nabla f(\bar x^k)\|^2}{K} \sim O(\frac{1}{K^{1/3}}),
%     \end{equation*}
%     when $K > (2D+1)^3$.
% \end{corollary}

\begin{corollary}
\label{coro:ADSGD}
     For Problem \ref{eq:dec_opt}, given Assumption \ref{asm:partialasynchrony} - \ref{asm:W} and let $\alpha = \frac{2}{L_F K^{1/3}}, \beta = \frac{1}{4L_F(D+1/2)K^{2/3}}$, the sequence $\{x^k\}$ generated by \eqref{eq:ADSGD_double_step size} satisfies:
    \begin{align*}
    &\frac{\sum_{k=0}^{K-1} \E\|\nabla f(\bar x^k)\|^2}{K} \le \left(16n^2C_1 + \frac{8n}{1-\lambda_2(W)}\right) L_F \frac{\sum_{i=1}^n(f_i(x^0)-f_i^*)}{K^{1/3}} \\
    &+ \left(\frac{n}{D(1-\lambda_2(W))} + 2n\right)\frac{\sigma^2}{K^{1/3}} + \frac{3n^2C_0}{2D}\frac{\sigma^2}{K^{2/3}},
\end{align*}
where $C_0$ is defined in Lemma \ref{the:ABCD_convergence} and $C_1 = 6B^2D^2 + 3B^2 + 3BD + 3B +D$.
\end{corollary}

% Corollary \ref{coro:ADSGD} shoes that ADSGD converges to its stationary points with a proper step size. The rate is slightly slower than the standard convergence rate for non-convex SGD. This is primarily due to the consensus error. Indeed, ADSGD converges to the stationary point of $L_\alpha$ with a faster rate $\mc{O}(\frac{1}{\sqrt{K}})$. However, the distance between the stationary points of $L_\alpha$ and $F$ is affected by consensus error, which is $\mc{O}(\alpha)$. Thus, the overall convergence rate is slower.
% Better controlling the consensus error for ADSGD poses significant challenges. To avoid waiting and achieve full asynchrony, an agent should only be able to modify information stored by itself (see Equation (\ref{eq:ADSGD_update_SYMBOLIC})). This breaks the double stochasticity of the mixing matrix, making previous analysis illegal. In \cite{lian2018asynchronous}, neighbors synchronize partially to keep the mixing matrix doubly stochastic. However, this results in stalls and sometimes causes significant performance drops as shown in Section \ref{sec:experiments}.

Corollary~\ref{coro:ADSGD} establishes that ADSGD converges to stationary points with proper step sizes, albeit at a rate slightly slower than standard non-convex SGD due to consensus error ($\mathcal{O}(\alpha)$). While convergence to $L_\alpha$'s stationary points achieves the faster $\mathcal{O}(1/\sqrt{K})$ rate with a proper step size, the gap between $L_\alpha$ and $F$ limits the overall rate. The key challenge lies in controlling consensus error while maintaining full asynchrony. ADSGD's self-only update rule (Eq.~\ref{eq:ADSGD_update_SYMBOLIC}) breaks double stochasticity, invalidating prior analyses. Partial neighbor synchronization~\cite{lian2018asynchronous} preserves this property but introduces stalls and performance degradation (Section~\ref{sec:experiments}).

\begin{figure*}[htbp]
\vspace{-0.3cm}
\centering
\begin{subfigure}[t]{0.45\textwidth}
\centering
\includegraphics[width=\textwidth]{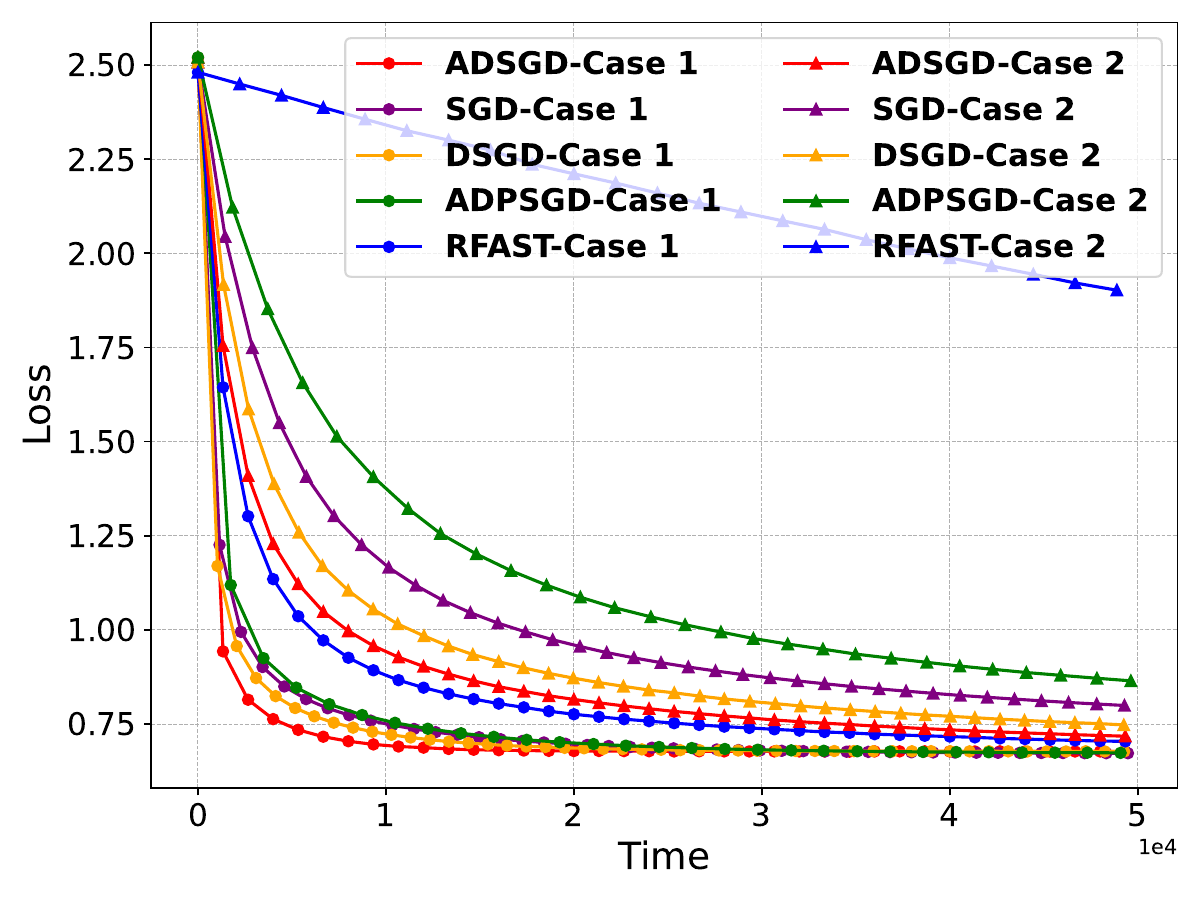}
% \vspace{-0.6cm}
\caption{No Straggler} %
\label{fig:LOG_no_stra_loss}
\end{subfigure}
% \hfill
\hspace{-0.2cm}
\begin{subfigure}[t]{0.45\textwidth}
\centering
\includegraphics[width=\textwidth]{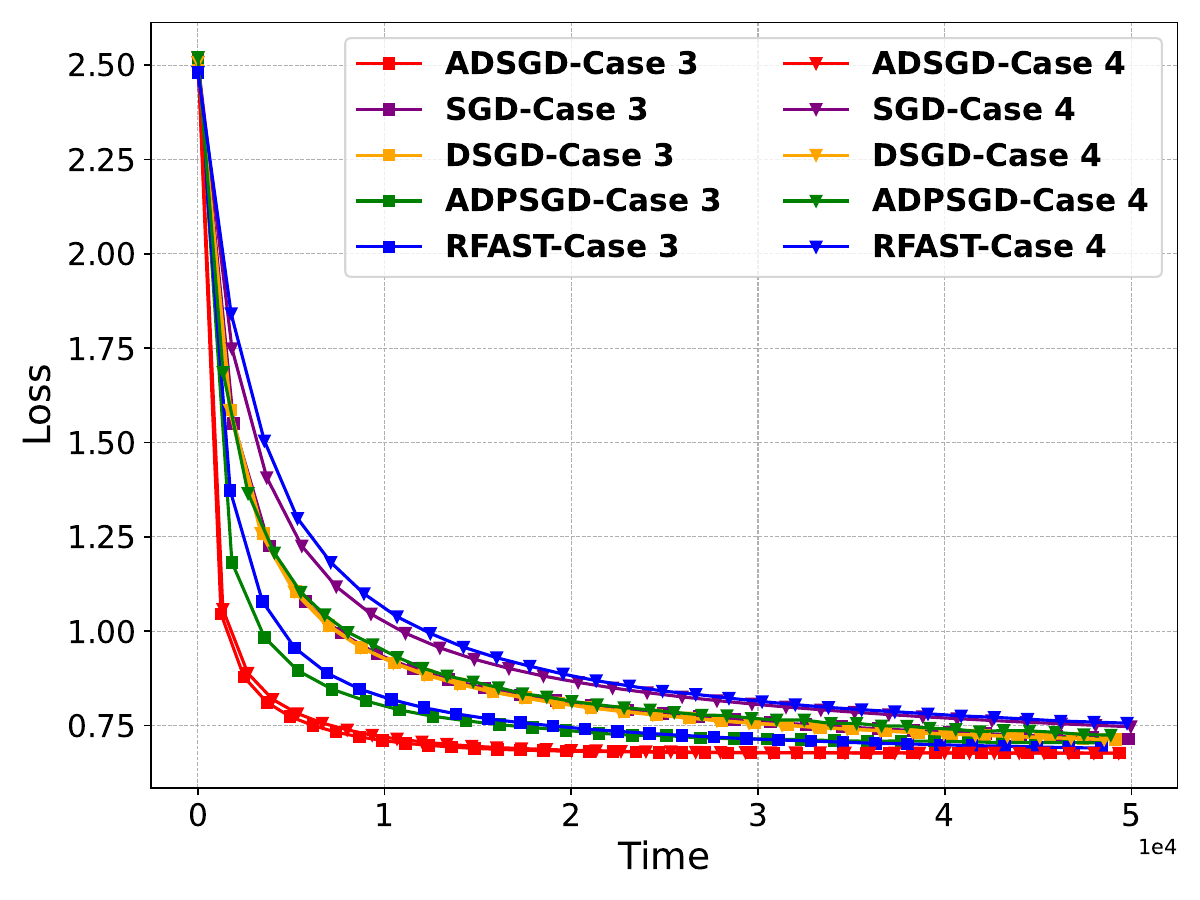}
% \vspace{-0.6cm}
\caption{One Straggler}
\label{fig:LOG_stra_loss}
\end{subfigure}
% \hfill
% \vspace{-0.25cm}
\caption{Loss plot of Logistic Regression on MNIST. Case 5 is excluded for clarity. See the appendix for details.}

\label{fig:LOG_loss}
  % \vspace{-0.4cm}
\end{figure*}

\section{Empirical Evaluation}
\label{sec:experiments}
This section presents empirical results demonstrating ADSGD's superiority over existing approaches. Three asynchronous algorithms are compared: ADSGD (ours), ADPSGD \cite{lian2018asynchronous}, and RFAST \cite{zhu2023robust}, with synchronous DSGD \cite{yuan2016convergence} and parallel SGD (ring-Allreduce implementation) serving as baselines for comparison. For all algorithms, the implementation adheres strictly to their original descriptions, without incorporating any additional acceleration techniques. These algorithms are tested on two non-convex tasks with real-world datasets. Computation and communication delays are carefully simulated to thoroughly examine their impact under various scenarios. Experiments on two non-convex tasks delays under diverse delay scenarios are conducted on a server with 8 Nvidia RTX3090 GPUs.

\textbf{Modeling delays.} We simulate system delays via random sampling for better ablation. Computation delays focus on gradient estimation (the dominant time cost), ignoring negligible model mixing/updating delays. For communication, we assume full-duplex agents with multicast and serial sending (i.e., sequential message transmission). These conservative assumptions hold in practice: full-duplex multicasting is supported by 4G, WiFi, Zigbee, and wired LANs. Serial sending improves information availability under bandwidth limits.

% We use random number sampling to simulate system delays. We do so for a better ablation study. For computation delays, we focus primarily on those caused by gradient estimation, as this typically consumes most of the time in learning tasks. Delays from model mixing and updating are not considered in our model due to their minimal impact compared to gradient calculations.

% % Thus, the mean of the computation-delay distribution represents the average time of gradient estimation.

% For communication, we assume that the agents are full-duplex, capable of multicasting to their neighbors, and engaging in serial sending\footnote{Serial sending refers to that the agents only start the next sending after finishing the current one.}. Those assumptions are generally conservative and frequently met in practice. Common wireless communication protocols, like 4G, WiFi, and Zigbee, are full-duplex and support multicasting. Wired local area networks (LANs) with switches also offer these functionalities. We opt for serial transmission because it expedites the availability of part of the information to asynchronous agents under bandwidth constraints.  

\textbf{Tasks description.} We carry out experiments with 9 agents connected by a grid network. We first conduct logistic regression with non-convex regularization on the MNIST dataset \cite{lecun1998gradient}. A modified VGG11 \cite{simonyan2014very} is then trained over CIFAR-10 \cite{krizhevsky2009learning}. For both tasks, performance is evaluated under varying degrees of data heterogeneity. In the homogeneous setting, datasets are randomly partitioned. For heterogeneous cases, a mixed partitioning scheme is employed: a subset of local data is label-skewed while the remainder is uniformly distributed. Due to space constraints, we focus on the most challenging scenario (fully label-partitioned data) in the main text, with other configurations provided in the appendix. Notably, ADSGD demonstrates even stronger performance advantages in simpler (less heterogeneous) scenarios. For each task \& algorithm, there are 5 test cases as summarized in Table \ref{tab:test_cases}. Note comm. and comp. on average takes 1 unit of time in case 1) and comm. is twice as slow in case 2) for VGG (slower communication delays convergence and increases computational overhead, thus not investigated). A detailed description is deferred to the appendix.

% \begin{table}[ht]
% \caption{Test Cases for All Algorithms}
% \label{tab:test_cases}
% % \vskip -0.in
% \centering
% \begin{tabular}{@{}lp{5cm}@{}}
% \toprule
% Case & Description \\
% \midrule
% 1) Base & Equal delays for all. \\
% 2) Slow Comm. & Comm. 10x (2x) slower for all. \\
% 3) Comp. Straggler & An agent 10x slower in comp. \\
% 4) Comm. Straggler & An agent 10x slower in comm. \\
% 5) Comb. Straggler & An agent 10x slower in both. \\
% \bottomrule
% \end{tabular}
% \vskip -0.1in
% \end{table}

% \begin{table}[ht]
% \caption{Test Cases for All Algorithms}
% \label{tab:test_cases}
% \centering
% \begin{tabular}{@{}ll@{}}
% \toprule
% Case & Description \\
% \midrule
% 1) Base & Uniform delays \\
% 2) Slow Comm. & 10$\times$/2$\times$ slower comm. \\
% 3) Comp. Str. & 1 agent 10$\times$ slower comp. \\
% 4) Comm. Str. & 1 agent 10$\times$ slower comm. \\
% 5) Both Str. & 1 agent 10$\times$ slower both \\
% \bottomrule
% \end{tabular}
% \end{table}

\begin{table}[ht]
\caption{Test Cases for All Algorithms}
\label{tab:test_cases}
\centering
\begin{tabular}{@{}llll@{}}
\toprule
Case & Description & Case & Description \\
\midrule
1) Base & Uniform delays & 2) Slow Comm. & 10$\times$(2$\times$) slower comm. \\
3) Comp. Str. &  1 agent 10$\times$ slower comp. & 4) Comm. Str. & 1 agent 10$\times$ slower comm. \\
5) Comb. Str. & 1 agent 10$\times$ slower both &  &  \\
\bottomrule
\end{tabular}
% \vskip -0.1in
\end{table}

\textbf{Parameter selection.} A fixed step size of 0.01 is adopted for both tasks and algorithms, except for RFAST, which uses $\frac{0.01}{w_{ii}}$ to match an identical effective step size. Local batch size for non-convex logistic regression and VGG training is set to 32 and 8, respectively.

\begin{figure*}[h]
\vspace{-0.3cm}
\centering
\begin{subfigure}[t]{0.45\textwidth}
\centering
\includegraphics[width=\textwidth]{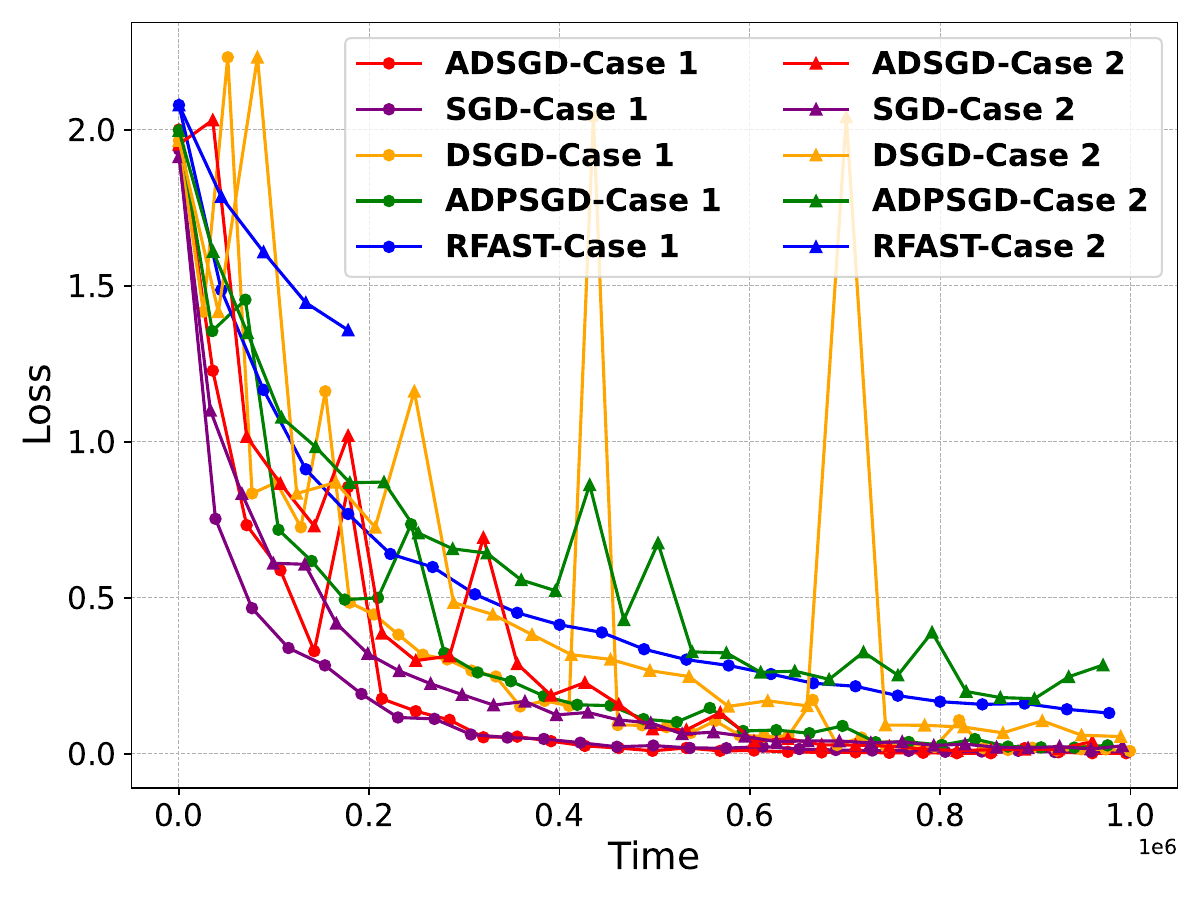}
% \vspace{-0.6cm}
\caption{No Straggler} %
\label{fig:VGG_no_stra_loss}
\end{subfigure}
% \hfill
% \hspace{-0.2cm}
\begin{subfigure}[t]{0.45\textwidth}
\centering
\includegraphics[width=\textwidth]{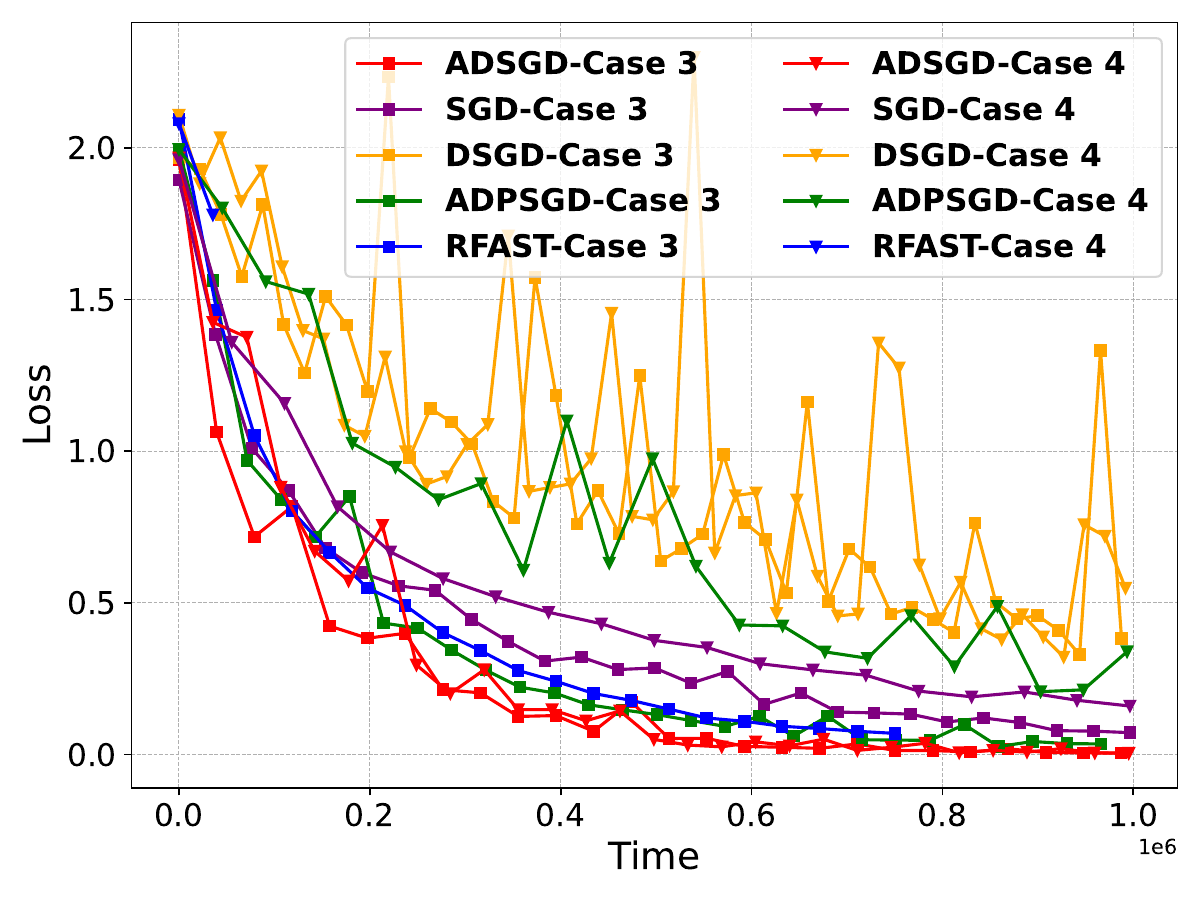}
% \vspace{-0.6cm}
\caption{One Straggler}
\label{fig:VGG_stra_loss}
\end{subfigure}
% \hfill
% \vspace{-0.25cm}
\caption{Loss plot of VGG on CIFAR10. Case 5 is excluded for clarity. See the appendix for details.}
\label{fig:VGG_loss}
  % \vspace{-0.4cm}
\end{figure*}

\subsection{Non-convex Logistic Regression on MNIST}
For logistic regression, the loss curves of the average model $f(\bar x)$ are presented in Fig. \ref{fig:LOG_loss} (case 5 and test accuracy in Appendix). 

Without stragglers (Fig. \ref{fig:LOG_no_stra_loss}), ADSGD outperforms all other algorithms and uniquely surpasses the baselines. While asynchronous methods typically show slower iteration-wise convergence from stale information but faster per-iteration runtime, RFAST and ADPSGD fail to benefit from this trade-off. ADPSGD suffers from partial synchronization requirements that slow its runtime, while RFAST's neighbor-specific communications and doubled gradient tracking overhead make it particularly delay-sensitive.

Under straggler scenarios (Fig. \ref{fig:LOG_stra_loss}), ADSGD maintains its lead. With communication stragglers, DSGD beats ADPSGD and RFAST, while ADSGD shows significantly stronger robustness. In computation delay cases (case 3), all asynchronous methods outperform the synchronous baseline, demonstrating their advantage under such conditions.

\subsection{VGG11 on CIFAR-10}
This section demonstrates ADSGD's capabilities on more complex tasks. Following the previous methodology, Fig. \ref{fig:VGG_no_stra_loss} presents the average model's loss curve, with additional results provided in the Appendix.

In Fig. \ref{fig:VGG_no_stra_loss}, all algorithms show convergence patterns similar to logistic regression. Parallel SGD emerges as the fastest due to non-convexity challenges with stale information and data heterogeneity. Among remaining methods, ADSGD maintains the fastest convergence.  Notably,  RFAST diverges in the case of Slow Comm., where the comm. delay is only twice as the comp. delay. Though unstable, RFAST does exhibit much smaller loss oscillation due to gradient tracking.

Fig. \ref{fig:VGG_stra_loss} confirms ADSGD's superiority. Asynchronous methods generally handle computation delays well, with ADSGD particularly robust against communication delays. RFAST diverges in all cases due to its delay sensitivity under non-convexity. Quantitatively, ADSGD achieves 85\% test accuracy 15-70\% faster than async methods (70\% for comm. stragglers), 30-85\% faster than sync methods, and 35-58\% faster than parallel SGD under stragglers (See Appendix for details.).

\subsection{Scalability}
While prior results demonstrate ADSGD’s effectiveness, computational constraints previously limited the number of agents tested. ADSGD's scalability is presented by measuring speedup across increasing agent counts under a ring topology and case 1 (Fig. \ref{fig:speedup}). Logistic Regression (Fig. \ref{fig:speedup_log}) scales to 128 agents, with loss curves showing consistent acceleration as parallelism grows. VGG11 Training (Fig. \ref{fig:speedup_vgg}) achieves comparable speedup up to 32 agents, beyond which hardware limitations prevent further testing. The near-linear trend in both tasks confirms ADSGD’s ability to leverage distributed training efficiently, revealing its potential to large scale applications.

\begin{figure*}[h]
\vspace{-0.3cm}
\centering
\begin{subfigure}[t]{0.45\textwidth}
\centering
\includegraphics[width=\textwidth]{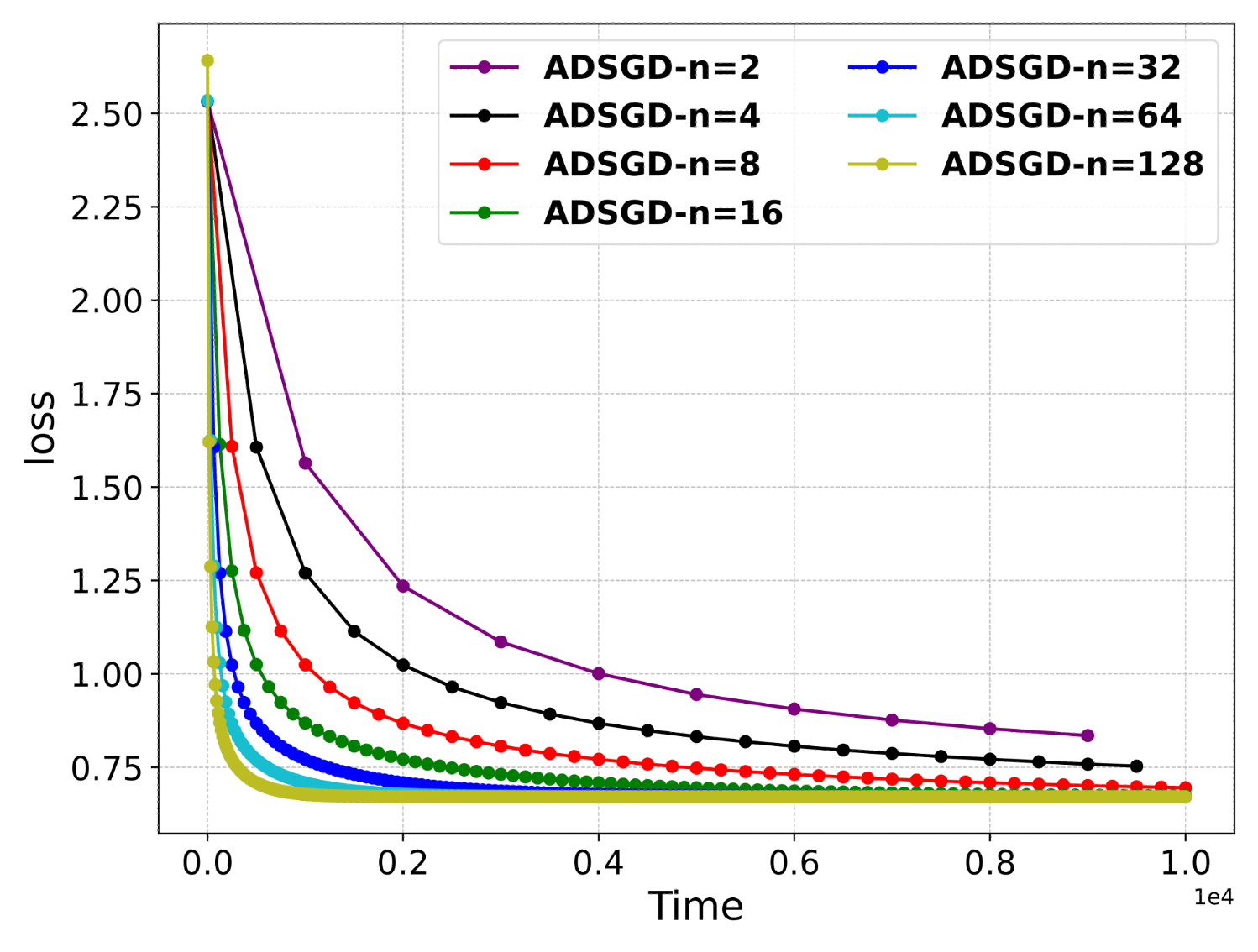}
% \vspace{-0.6cm}
\caption{Logistic Regression} %
\label{fig:speedup_log}
\end{subfigure}
% \hfill
% \hspace{-0.2cm}
\begin{subfigure}[t]{0.45\textwidth}
\centering
\includegraphics[width=\textwidth]{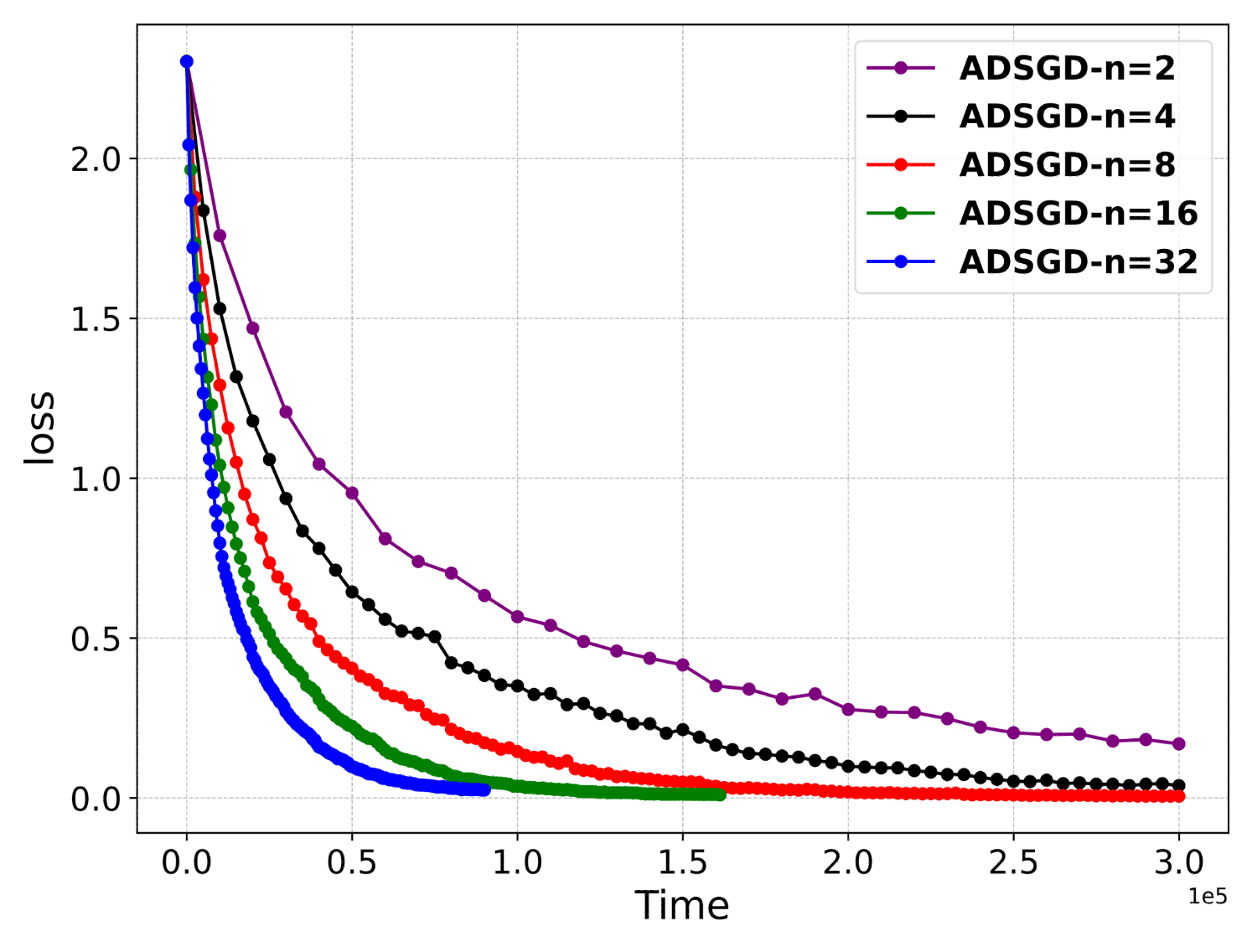}
% \vspace{-0.6cm}
\caption{VGG11 Training}
\label{fig:speedup_vgg}
\end{subfigure}
% \hfill
% \vspace{-0.25cm}
\caption{Speedup of ADSGD w.r.t. number of agents under a ring topology.}
\label{fig:speedup}
  % \vspace{-0.4cm}
\end{figure*}

\section{Conclusion}
This paper explores the convergence of asynchronous algorithms under bounded communication and computation delays, focusing on ASBCD and ADSGD. We show that ASBCD converges to the neighborhood of its stationary points in non-convex settings and achieves a rate of $\mc{O(}1/\sqrt{K})$ given a proper step size. Extending these results to ADSGD, we prove its convergence under non-convexity with a computation-delay-independent step size, without assuming data heterogeneity. The experimental results confirm the effectiveness of ADSGD on non-convex learning tasks. We highlight that the proposed approach is simple, memory-efficient, communication-efficient, and strikingly resilient to communication delays, providing greater flexibility and robustness in decentralized optimization scenarios.

% \nocite{langley00}

\bibliographystyle{apalike}
\bibliography{ref}

\begin{thebibliography}{}

\bibitem[Bornstein et~al., 2022]{bornstein2022swift}
Bornstein, M., Rabbani, T., Wang, E., Bedi, A.~S., and Huang, F. (2022).
\newblock Swift: Rapid decentralized federated learning via wait-free model communication.

\bibitem[Cannelli et~al., 2020]{cannelli2020asynchronous}
Cannelli, L., Facchinei, F., Scutari, G., and Kungurtsev, V. (2020).
\newblock Asynchronous optimization over graphs: Linear convergence under error bound conditions.
\newblock {\em IEEE Transactions on Automatic Control}, 66(10):4604--4619.

\bibitem[Kazemi and Wang, 2019]{kazemi2019asynchronous}
Kazemi, E. and Wang, L. (2019).
\newblock Asynchronous delay-aware accelerated proximal coordinate descent for nonconvex nonsmooth problems.
\newblock In {\em Proceedings of the AAAI Conference on Artificial Intelligence}, volume~33, pages 1528--1535.

\bibitem[Koloskova et~al., 2020]{koloskova2020unified}
Koloskova, A., Loizou, N., Boreiri, S., Jaggi, M., and Stich, S. (2020).
\newblock A unified theory of decentralized sgd with changing topology and local updates.
\newblock In {\em International conference on machine learning}, pages 5381--5393. PMLR.

\bibitem[Koloskova et~al., 2022]{koloskova2022sharper}
Koloskova, A., Stich, S.~U., and Jaggi, M. (2022).
\newblock Sharper convergence guarantees for asynchronous sgd for distributed and federated learning.

\bibitem[Krizhevsky et~al., 2009]{krizhevsky2009learning}
Krizhevsky, A., Hinton, G., et~al. (2009).
\newblock Learning multiple layers of features from tiny images.

\bibitem[Kungurtsev et~al., 2023]{kungurtsev2023decentralized}
Kungurtsev, V., Morafah, M., Javidi, T., and Scutari, G. (2023).
\newblock Decentralized asynchronous non-convex stochastic optimization on directed graphs.
\newblock {\em IEEE Transactions on Control of Network Systems}.

\bibitem[Leblond et~al., 2017]{leblond2017asaga}
Leblond, R., Pedregosa, F., and Lacoste-Julien, S. (2017).
\newblock Asaga: Asynchronous parallel saga.
\newblock In {\em Artificial Intelligence and Statistics}, pages 46--54. PMLR.

\bibitem[LeCun et~al., 1998]{lecun1998gradient}
LeCun, Y., Bottou, L., Bengio, Y., and Haffner, P. (1998).
\newblock Gradient-based learning applied to document recognition.
\newblock {\em Proceedings of the IEEE}, 86(11):2278--2324.

\bibitem[Lian et~al., 2018]{lian2018asynchronous}
Lian, X., Zhang, W., Zhang, C., and Liu, J. (2018).
\newblock Asynchronous decentralized parallel stochastic gradient descent.
\newblock In {\em International Conference on Machine Learning}, pages 3043--3052. PMLR.

\bibitem[Liu and Wright, 2015]{liu2015asynchronous}
Liu, J. and Wright, S.~J. (2015).
\newblock Asynchronous stochastic coordinate descent: Parallelism and convergence properties.
\newblock {\em SIAM Journal on Optimization}, 25(1):351--376.

\bibitem[Nedic et~al., 2017]{nedic2017achieving}
Nedic, A., Olshevsky, A., and Shi, W. (2017).
\newblock Achieving geometric convergence for distributed optimization over time-varying graphs.
\newblock {\em SIAM Journal on Optimization}, 27(4):2597--2633.

\bibitem[Niwa et~al., 2021]{niwa2021asynchronous}
Niwa, K., Zhang, G., Kleijn, W.~B., Harada, N., Sawada, H., and Fujino, A. (2021).
\newblock Asynchronous decentralized optimization with implicit stochastic variance reduction.
\newblock In {\em International Conference on Machine Learning}, pages 8195--8204. PMLR.

\bibitem[Peng et~al., 2016]{peng2016arock}
Peng, Z., Xu, Y., Yan, M., and Yin, W. (2016).
\newblock Arock: an algorithmic framework for asynchronous parallel coordinate updates.
\newblock {\em SIAM Journal on Scientific Computing}, 38(5):A2851--A2879.

\bibitem[Pu et~al., 2020]{pu2020push}
Pu, S., Shi, W., Xu, J., and Nedi{\'c}, A. (2020).
\newblock Push--pull gradient methods for distributed optimization in networks.
\newblock {\em IEEE Transactions on Automatic Control}, 66(1):1--16.

\bibitem[Samarakoon et~al., 2019]{samarakoon2019distributed}
Samarakoon, S., Bennis, M., Saad, W., and Debbah, M. (2019).
\newblock Distributed federated learning for ultra-reliable low-latency vehicular communications.
\newblock {\em IEEE Transactions on Communications}, 68(2):1146--1159.

\bibitem[Simonyan and Zisserman, 2014]{simonyan2014very}
Simonyan, K. and Zisserman, A. (2014).
\newblock Very deep convolutional networks for large-scale image recognition.
\newblock {\em arXiv preprint arXiv:1409.1556}.

\bibitem[Spiridonoff et~al., 2020]{spiridonoff2020robust}
Spiridonoff, A., Olshevsky, A., and Paschalidis, I.~C. (2020).
\newblock Robust asynchronous stochastic gradient-push: Asymptotically optimal and network-independent performance for strongly convex functions.
\newblock {\em Journal of machine learning research}, 21(58):1--47.

\bibitem[Sun et~al., 2017]{sun2017asynchronous}
Sun, T., Hannah, R., and Yin, W. (2017).
\newblock Asynchronous coordinate descent under more realistic assumptions.
\newblock {\em Advances in Neural Information Processing Systems}, 30.

\bibitem[Tang et~al., 2023]{tang2023fusionai}
Tang, Z., Wang, Y., He, X., Zhang, L., Pan, X., Wang, Q., Zeng, R., Zhao, K., Shi, S., He, B., et~al. (2023).
\newblock Fusionai: Decentralized training and deploying llms with massive consumer-level gpus.
\newblock {\em arXiv preprint arXiv:2309.01172}.

\bibitem[Tian et~al., 2020]{tian2020achieving}
Tian, Y., Sun, Y., and Scutari, G. (2020).
\newblock Achieving linear convergence in distributed asynchronous multiagent optimization.
\newblock {\em IEEE Transactions on Automatic Control}, 65(12):5264--5279.

\bibitem[Tseng, 1991]{tseng1991rate}
Tseng, P. (1991).
\newblock On the rate of convergence of a partially asynchronous gradient projection algorithm.
\newblock {\em SIAM Journal on Optimization}, 1(4):603--619.

\bibitem[Ubl and Hale, 2022]{ubl2022faster}
Ubl, M. and Hale, M.~T. (2022).
\newblock Faster asynchronous nonconvex block coordinate descent with locally chosen stepsizes.
\newblock In {\em 2022 IEEE 61st Conference on Decision and Control (CDC)}, pages 4559--4564. IEEE.

\bibitem[Wu et~al., ]{wu_delay-agnostic_2023}
Wu, X., Liu, C., Magnusson, S., and Johansson, M.
\newblock Delay-agnostic asynchronous distributed optimization.

\bibitem[Yuan et~al., 2016]{yuan2016convergence}
Yuan, K., Ling, Q., and Yin, W. (2016).
\newblock On the convergence of decentralized gradient descent.
\newblock {\em SIAM Journal on Optimization}, 26(3):1835--1854.

\bibitem[Zeng and Yin, 2018]{zeng2018nonconvex}
Zeng, J. and Yin, W. (2018).
\newblock On nonconvex decentralized gradient descent.
\newblock {\em IEEE Transactions on signal processing}, 66(11):2834--2848.

\bibitem[Zhang and You, 2019]{zhang2019fully}
Zhang, J. and You, K. (2019).
\newblock Fully asynchronous distributed optimization with linear convergence in directed networks.
\newblock {\em arXiv preprint arXiv:1901.08215}.

\bibitem[Zhou et~al., 2018]{zhou2018distributed}
Zhou, Y., Liang, Y., Yu, Y., Dai, W., and Xing, E.~P. (2018).
\newblock Distributed proximal gradient algorithm for partially asynchronous computer clusters.
\newblock {\em Journal of Machine Learning Research}, 19(19):1--32.

\bibitem[Zhu et~al., 2023]{zhu2023robust}
Zhu, Z., Tian, Y., Huang, Y., Xu, J., and He, S. (2023).
\newblock Robust fully-asynchronous methods for distributed training over general architecture.
\newblock {\em arXiv preprint arXiv:2307.11617}.

\end{thebibliography}

%%%%%%%%%%%%%%%%%%%%%%%%%%%%%%%%%%%%%%%%%%%%%%%%%%%%%%%%%%%%%%%%%%%%%%%%%%%%%%%
%%%%%%%%%%%%%%%%%%%%%%%%%%%%%%%%%%%%%%%%%%%%%%%%%%%%%%%%%%%%%%%%%%%%%%%%%%%%%%%
% APPENDIX
%%%%%%%%%%%%%%%%%%%%%%%%%%%%%%%%%%%%%%%%%%%%%%%%%%%%%%%%%%%%%%%%%%%%%%%%%%%%%%%
%%%%%%%%%%%%%%%%%%%%%%%%%%%%%%%%%%%%%%%%%%%%%%%%%%%%%%%%%%%%%%%%%%%%%%%%%%%%%%%
\newpage
\appendix
\onecolumn
\section*{Appendix}
\section{Algorithm Comparison}
Table \ref{tab:sto_alg} represents all decentralized stochastic gradient algorithms with provable convergence results under identical asynchrony assumptions. We present details of derivation in the following subsections.

\begin{table*}[htbp]
    \centering
    \caption{Decentralized stochastic gradient algorithms that converge under identical asynchrony assumptions. Memory and communication budget are normalized w.r.t. the model size. Comments: (1) SC and NC represent strong convexity and general non-convexity, respectively. (2) The step size in \cite{spiridonoff2020robust} is computation-delay-independent. However, they used diminishing step size, which will eventually be small enough for their convergence results. Readers may refer to the appendix for details.}
    \label{tab:sto_alg}
    
    \vskip 0.15in
        
    % \begin{tabular}{||c|c|c|c|c||}
    
    \begin{tabular}{cccccc}
    \toprule
    \label{tab:compare_alg}
    
     & Convexity$^{(1)}$ & Comp.-delay-ind. Step Size & Memory Cost & Comm. Budget\\
    % \hline
    \midrule
    \cite{spiridonoff2020robust} & SC & $\checkmark^{(2)}$ & $3|\mc{N}_i|+3$ & $|\mc{N}_i|$ \\
    % \hline
    \cite{zhu2023robust} & NC & &$4|\mc{N}_i|+5$ & $2|\mc{N}_i|$\\
    % \hline
    \cite{kungurtsev2023decentralized} & NC &  &$4|\mc{N}_i|+5$ & $2|\mc{N}_i|$\\
    % \hline
    Ours & NC & $\checkmark$ & $|\mc{N}_i|+2$ & $|\mc{N}_i|$ \\
    % \hline
    \bottomrule
    \end{tabular}
    \vskip -0.1in
\end{table*}
\subsection{On Step Sizes}
Only \cite{spiridonoff2020robust} and this work adopt a computation-delay-independent step size. However, 
in \cite{spiridonoff2020robust}, they use a diminishing step size rule to reach an asymptotic result, whereas we use a fixed step size.  

As mentioned in Section \ref{sec:related_work}, current gradient tracking methods require not only computation-delay-dependent step sizes but also extremely small ones. We simplify (increase) their required step sizes to illustrate how impractical their bounds are.

From Theorem III.1 in \cite{kungurtsev2023decentralized}, its initial step size is upper bounded by $\frac{1}{4(L_F+\tilde C_1 + \tilde C_2)}$, where $\tilde C_1 \triangleq \frac{2\tilde C_0}{1-\rho^2}$, $\tilde C_0 \triangleq \frac{2\sqrt{(D+2)n} (1+\underline{w} ^{-K_A})}{1-\underline{w}^{K_A}}$, $\rho\triangleq (1-\underline{w}^{K_A})^{1/K_A}$, and $K_A\triangleq (2n-1)B+nD$. For ease of illustration, we increase the upper bound to $\frac{1}{\tilde C_1} = \frac{(1-\rho)^2}{2\tilde C_0}$. Note that this is a very loose bound since $\tilde C_2$ ($C_2^t$ in \cite{kungurtsev2023decentralized}) can be magnitudes larger than $\tilde C_1$.

% From Theorem III.1 in \cite{kungurtsev2023decentralized}, its initial step size is upper bounded by $\frac{1}{4(L+C_2^c + C_2^t)}$. For the ease of illustration, we increase the upper bound to $\frac{1}{C_2^c}$. Note that this is a very loose bound, since $C_2^t$ can be magnitudes larger than $C_2^c$. From (11) in \cite{kungurtsev2023decentralized},
% $C_2^c\triangleq \frac{2C_0^c}{(1-\rho)^2}$, where $C_0^c \triangleq \frac{2\sqrt{(D+2)n} (1+\underline{w} ^{K_A})}{1-\underline{w}^{K_A}}$, $\rho\triangleq (1-\underline{w}^{K_A})^{1/K_A}$, and $K_A\triangleq (2n-1)B+nD$. Thus, the upperbound of the step size is $\frac{(1-\rho)^2}{2C_0^c}$.

From (39) in \cite{zhu2023robust}, the step size is upper bounded by $\frac{r\eta^2}{16(1+6\rho_t)K}$, where $r$ is the number of common roots between subgraphs for pushing and pulling. For the ease of illustration, we assume that $r\le 96$. Thus, the upper bound is now $\frac{\eta^2}{\rho_t K}$, where $\eta \triangleq \underline w^{K_A}$, $\rho_t \triangleq \frac{54\tilde C_3^2 C_L^2 [4\tilde C_0^2 + (1-\rho)^2]}{(1-\rho)^4}$, $\tilde C_3\triangleq\frac{2\sqrt{2S} (1+\underline w^{-K_A})}{\rho (1-\underline w^{K_A})}$, $C_L=\max\{L_i\}$, and $S\ge (D+1)n$. Again for illustration, we assume $216\tilde C_3^2C_L^2 \ge 1$, which should be satisfied for most applications. Therefore, we have $\rho_t \ge \frac{C_2^2}{(1-\rho)^4}$ and the step size must be smaller than $\frac{\eta^2(1-\rho^4)}{\tilde C_0^2K}$.

\subsection{On Memory and Communication Costs}

For \cite{spiridonoff2020robust}, from their Algorithm 3, each agent stores $\{\mbf{x}_i, \hat{\mbf{g}_i}, \pmb{\phi}_i^x\}$ for itself, and $\{\pmb{\phi}_i^x,\pmb{\rho}^{*x}_{ij}, \pmb{\rho}_{ij}^{x}\}$ for all neighbors. In each iteration, each agent sends $\{\pmb{\phi}_i^x\}$. Some scalars are ignored.

For \cite{zhu2023robust}, each agent stores $\{x_i^t, z_i^t, v_i^t, \nabla f_i(x_i^{t+1};\zeta_i^{t+1}), \nabla f_i(x_i^{t};\zeta_i^{t}) \}$ for itself, and $\{ v_{j}^{\tau_{v,ij}^t}, \rho_{ij}^{\tau_{\rho,ij}^t}, \tilde{\rho}_{ij}^t, \rho_{ji}^t\}$ for all neighbors. In each iteration, each agent sends $\{v_i^{t+1}, \rho_{ji}^{t+1}\}$. The calculation for \cite{kungurtsev2023decentralized} is identical.

In the proposed method, each agent only stores $\{x_i, g^{f_i}(x_i)\}$ for itself and $\{x_j\}$ for its neighbors. In each iteration, each agent sends $\{x_i\}$.

\subsection{Memory-efficient Implementation of ADSGD}
In Table \ref{tab:compare_alg}, ADSGD requires $|\mc{N}|_i +2$ units of memory, scaling linearly with the number of neighbors. We can remove such a dependency by a straightforward memory-efficient implementation, as shown in Algorithm \ref{alg:mem_eff_ADSGD}. The key idea is to store only the weighted sum of neighboring models, rather than each individual model. During each iteration, nodes exchange model updates instead of the full models. Consequently, each node only maintains $\{x_i, y_i, z_i, g^{f_i}(x_i)\}$, where $z_i$ is the model update. Note that since we introduced the additional variable $z_i$, the memory requirement becomes fixed at four times the model size, regardless of the number of neighbors $|\mc{N}|_i$. Such an implementation is beneficial when $|\mc{N}|_i > 2$.

\begin{algorithm}[tb]
\caption{Memory-efficient ADSGD}
\label{alg:mem_eff_ADSGD}
\begin{algorithmic}[1]
\STATE \textbf{Initialization:}  All the nodes agree on $\alpha>0$, and cooperatively set $w_{ij}$ $\forall \{i,j\}\in\mc{E}$.
\STATE Each node chooses $x_i\in\mathbb{R}^d$, creates a local buffer $\mc{B}_i$, shares $x_i$ with all neighbors in $\mc{N}_i$, and calculates $g^{f_i}(x_i)$. 
\STATE Each node stores $y_i = \sum_{j\in \mc{N}_i}w_{ij} x_j$ in $\mc{B}_i$.
\vspace{0.5\baselineskip} 
\STATE \textbf{All Nodes Do In Parallel:}
\WHILE{the termination criterion is not met}
    \REPEAT
        \STATE Keep receiving $z_j$ from neighbors.
        \STATE Update $\mc{B}_i$ by $y_i = y_i + w_{ij}z_j$.
    \UNTIL{$g^{f_i}(x_i)$ is available.}
    \STATE Send $z_i = (w_{ii}-1) x_i + y_i - \alpha g^{f_i}(x_i)$ to all neighbors $j\in\mc{N}_i$.
    \STATE $x_i = x_i + z_i$.
    \STATE Calculates $g^{f_i}(x_i)$.
\ENDWHILE
\end{algorithmic}
\end{algorithm}

\section{More on Experiments}
\label{sec:more_exp}
\subsection{Compute Resources}
\label{sec:com_resource}
The experiments are conducted on a server with 8 NVIDIA 3090 GPU, each with 24GB memory. The overall experiment takes roughly over 1200 GPU hours. The heavy computation can be attributed to two reasons: (1) We investigate 5 delay configurations and 3 levels of data heterogeneity. (2) RFAST is over twice slower than ADSGD due to its complicated update mechanism.

\subsection{Experiment Settings}
\label{sec:more_exp_setting}
As mentioned in Section \ref{sec:experiments}, we use random numbers to simulate computation and communication delay. The distributions for both delays are tailored based on real-world data, specifically fitting actual computation delays experienced by GPUs and communication delays between them. We only control the mean of the distributions, while other parameters are altered accordingly. In general, the communication delay distribution exhibits a larger variance compared to that of the computation delay.

We elaborate on the 5 test cases mentioned in Table \ref{tab:test_cases}. 1) the base case with identical communication and computation speeds across all agents, where the mean of each delay distribution is set to 1; 2) the slow communication case, where all agents' communication is 10 times slower (2 times for the case of VGG training) than their computation; 3) the computation straggler case with one agent computes 10 times slower than others, while the remaining agents have identical computation and communication speed; 4) the communication straggler case with one agent communicates 10 times slower than others; and 5) a combined straggler case where one agent is 10 times slower in both communication and computation. Here, having identical communication and computation speeds indicates that the delay distributions have identical means while being different in shape. 10x slower computation implies the mean of computation delay is 10 times larger than the mean of communication delay, and vice versa.

We evaluate ADSGD under varying data heterogeneity levels, quantified by parameter $\zeta$ . This parameter determines the fraction of dataset splits based on ordered labels (with the remainder sampled uniformly). Higher $\zeta$  values correspond to greater heterogeneity.

\subsection{Non-convex Logistic Regression on MNIST}
Fig. \ref{fig:log_no_stra_loss} and \ref{fig:log_stra_loss} present the plots of training loss as a function of runtime for different algorithms. We see that under the presence of a combined straggler (Fig. \ref{fig:log_combined_stra_loss}), parallel SGD performs the worst and the lead of ADSGD is considerable.

Fig. \ref{fig:log_base_slow_acc} and \ref{fig:log_straggler_acc} present the plots of test accuracy as a function of runtime for different algorithms. Under all cases, ADSGD consistently achieves higher test accuracy within the same time frame. Notably, in most cases, asynchronous algorithms other than ADSGD even fail to outperform DSGD, rendering them of limited practical value.

Fig. \ref{fig:LOG_relative_time} illustrates the relative time required to reach 89\% test accuracy across different algorithms. ADSGD maintains a substantial lead, converging faster than other asynchronous algorithms by at least 50\% in all cases. In case 4, ADSGD demonstrates its resilience to communication delay, saving over 85\% of the time compared to RFAST, which is more than 7 times faster. In the case of slow communication, RFAST, while not diverging, converges very slowly and fails to reach 89\% accuracy within $2 \times 10^5$ units of time (for comparison, ADSGD achieves this at $5.3 \times 10^4$ units of time). Moreover, ADSGD outpaces its synchronous counterpart by a margin of 28\% - 75\%.

It is important to note that ADPSGD cannot achieve 89\% test accuracy in the presence of a straggler. This is because it converges to the stationary point of the weighted sum of local cost functions, $\sum_i p_i f_i$, instead of the global objective $\sum_i f_i$, where $p_i$ represents the update frequency. This issue stems from its partially synchronized update $X_{k+1} = X_k W_k - \alpha g(\hat{\mbf{x}}_k)$. For instance, in an extreme case where one agent is connected to all others and continues updating, that agent will dominate the updates and drag the other agents toward its own stationary point. Therefore, the partial synchronization in ADPSGD not only slows down the protocol under communication delays because of stalls but also introduces some bias into the optimization process.

Fig. \ref{fig:diff_hete_log} shows convergence of Logistic Regression under varying heterogeneity levels ($\zeta\in{0,0.5}$). ADSGD demonstrates stronger performance with lower $\zeta$ values, consistently outperforming baselines across all scenarios.

\begin{figure}[htbp]
\vspace{-0.3cm}
\centering
\begin{subfigure}[t]{0.3\textwidth}
\centering
\includegraphics[width=\textwidth]{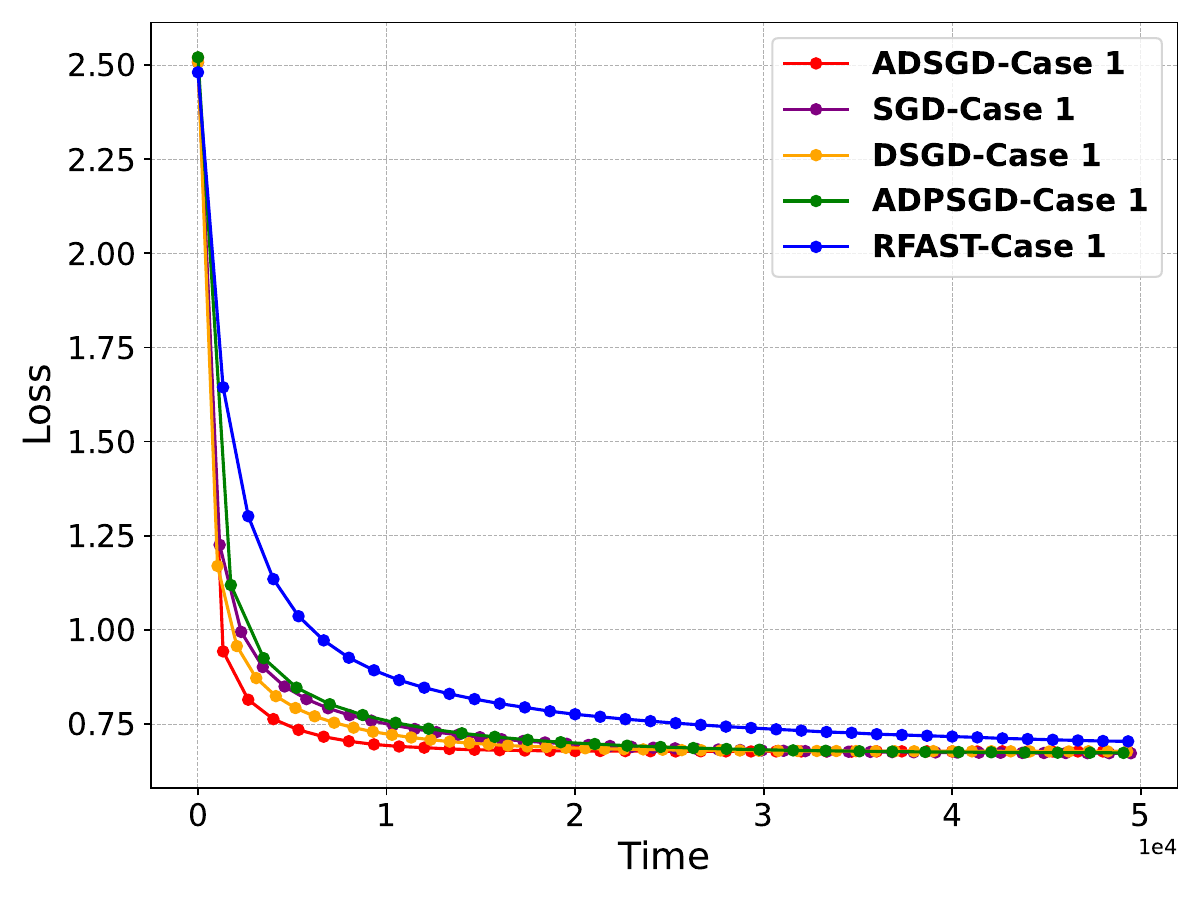}
\vspace{-0.6cm}
\caption{Base}
\label{fig:log_base_loss}
\end{subfigure}
\hspace{-0.15cm}
\begin{subfigure}[t]{0.3\textwidth}
\centering
\includegraphics[width=\textwidth]{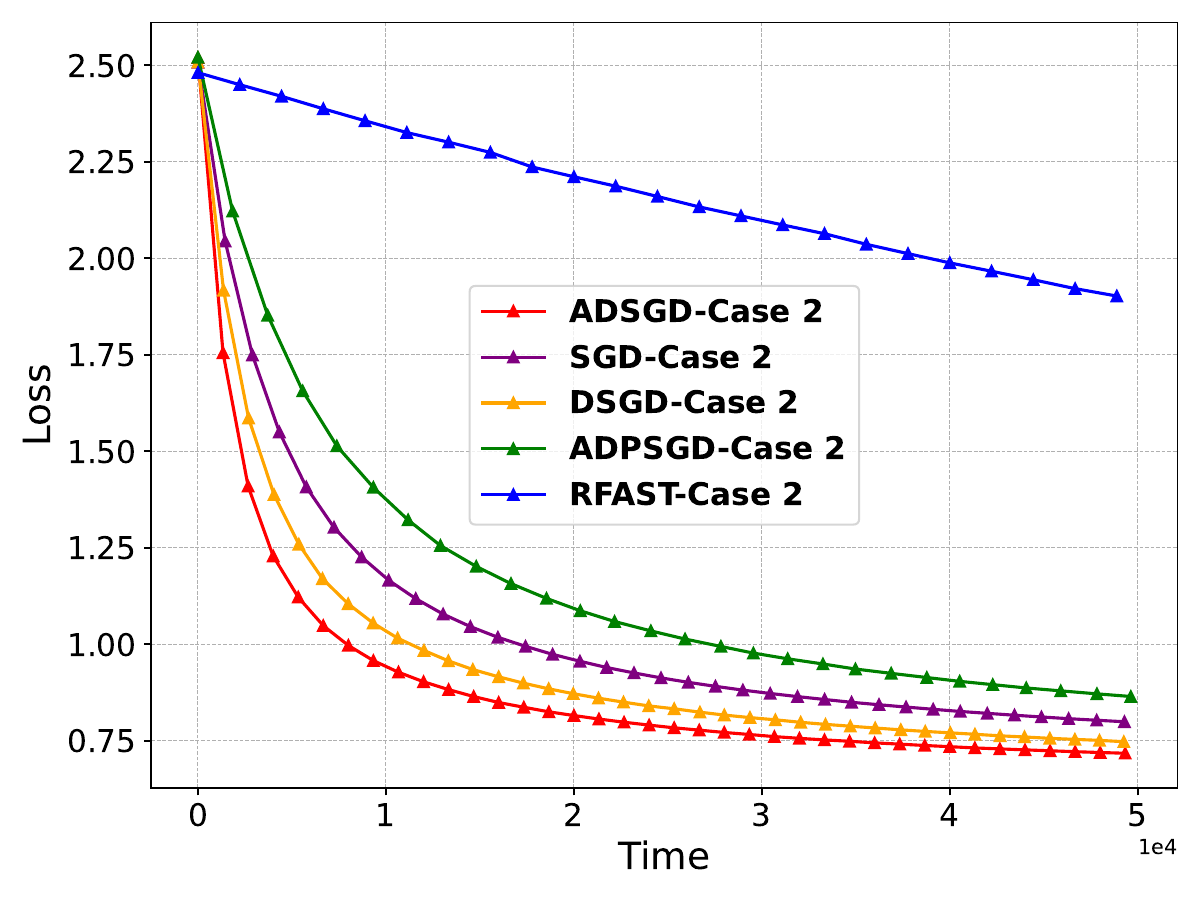}
\vspace{-0.6cm}
\caption{Slow Comm.}
\label{fig:log_slow_loss}
\end{subfigure}
\vspace{-0.25cm}
\caption{Logistic Regression - No Straggler - Training Loss - $\zeta=1$ }
\label{fig:log_no_stra_loss}
  % \vspace{-0.4cm}
\end{figure}

\begin{figure*}[htbp]
\vspace{-0.3cm}
\centering
\begin{subfigure}[t]{0.3\textwidth}
\centering
\includegraphics[width=\textwidth]{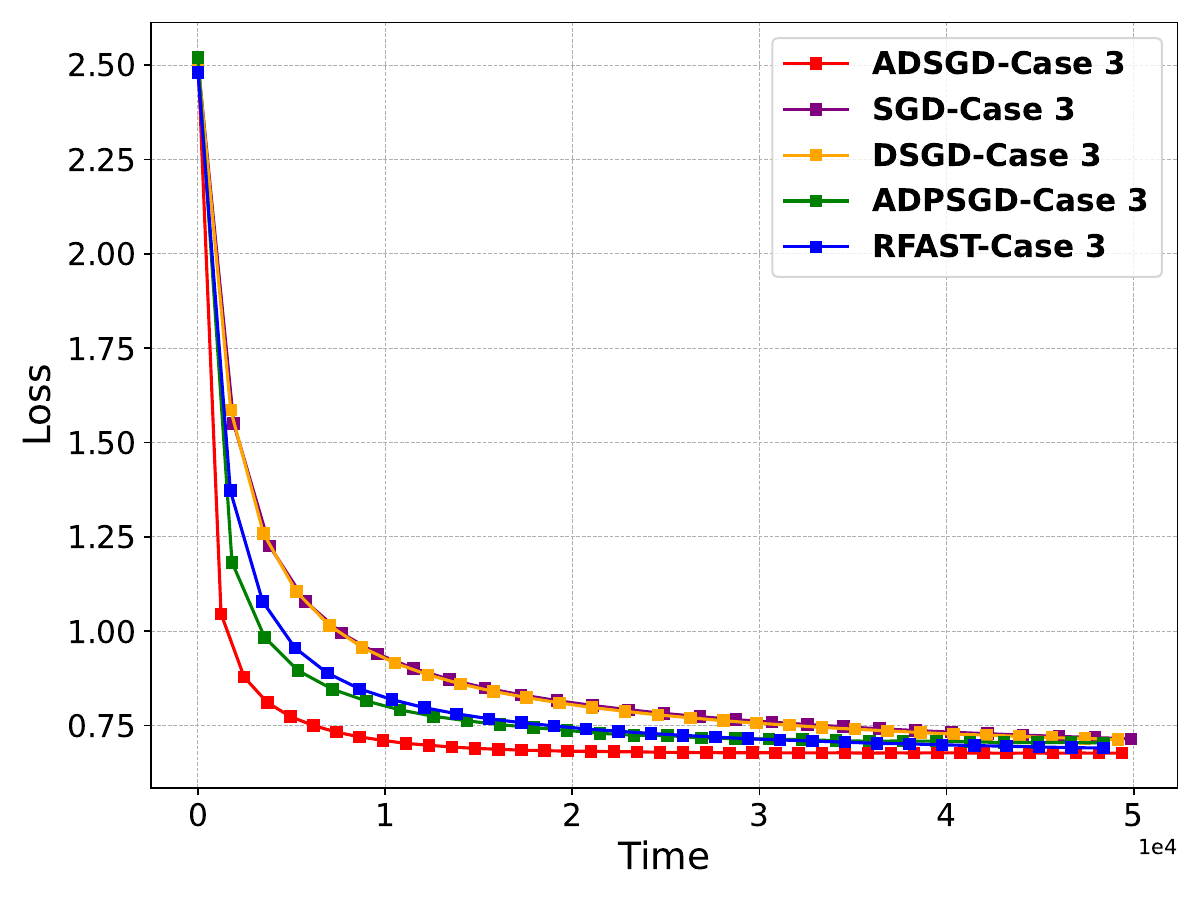}
\vspace{-0.6cm}
\caption{Comp. Straggler} %
\label{fig:log_comp_stra_loss}
\end{subfigure}
% \hfill
\hspace{0.4cm}
\begin{subfigure}[t]{0.3\textwidth}
\centering
\includegraphics[width=\textwidth]{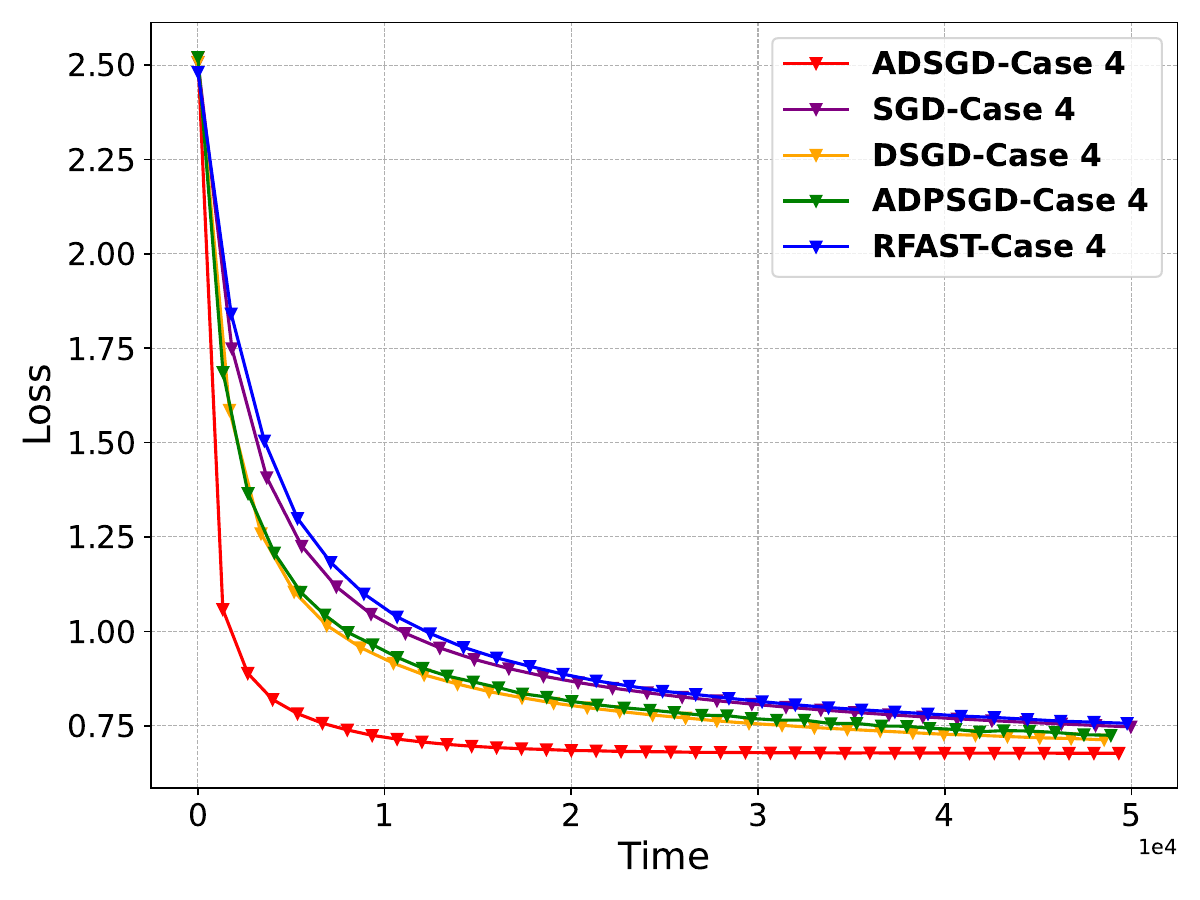}
\vspace{-0.6cm}
\caption{Comm. Straggler}
\label{fig:log_comm_stra_loss}
\end{subfigure}
% \hfill
\hspace{0.4cm}
\begin{subfigure}[t]{0.3\textwidth}
\centering
\includegraphics[width=\textwidth]{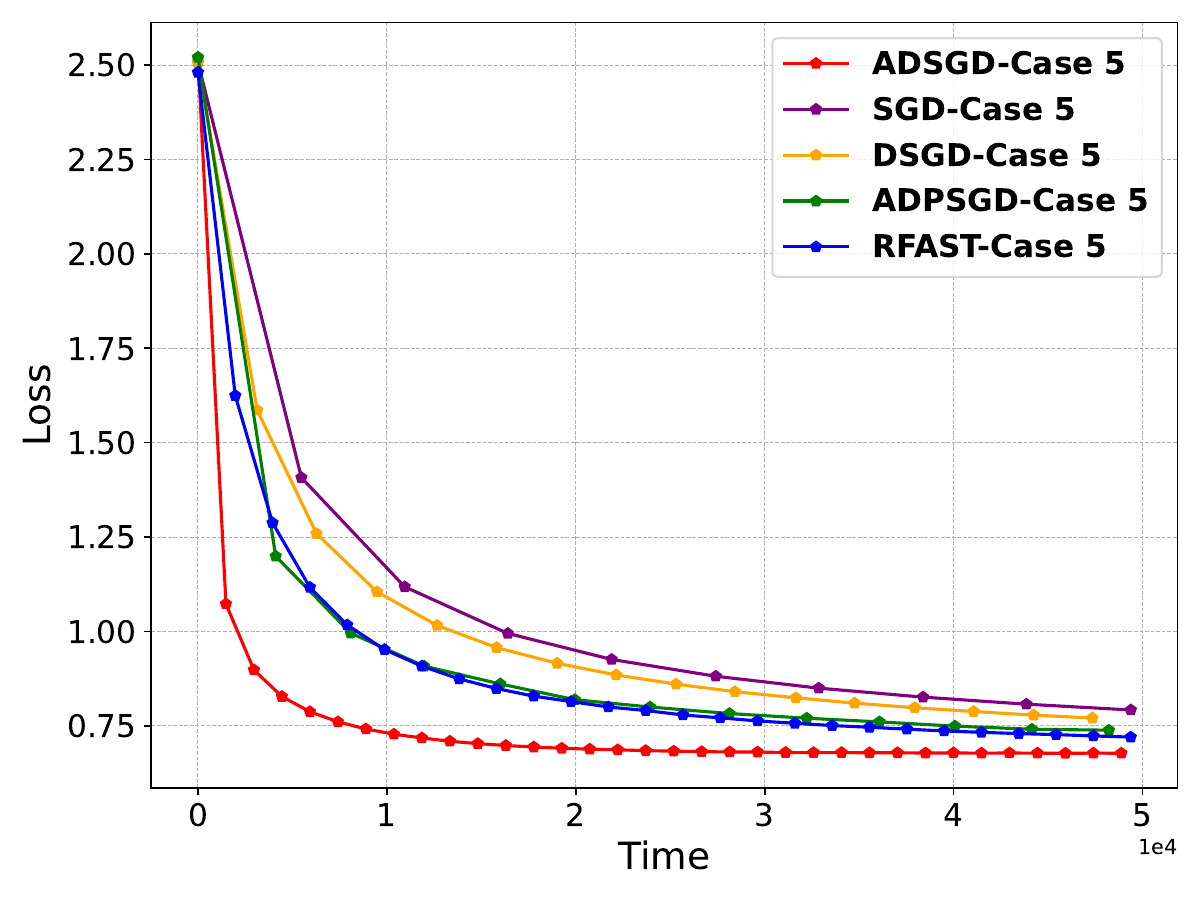}
\vspace{-0.6cm}
\caption{Combined Straggler}
\label{fig:log_combined_stra_loss}
\end{subfigure}
% \hfill
\vspace{-0.25cm}
\caption{Logistic Regression - One Straggler - Training Loss - $\zeta=1$}
\label{fig:log_stra_loss}
  % \vspace{-0.4cm}
\end{figure*}

\begin{figure}[htbp]
\vspace{-0.3cm}
\centering
\begin{subfigure}[t]{0.3\textwidth}
\centering
\includegraphics[width=\textwidth]{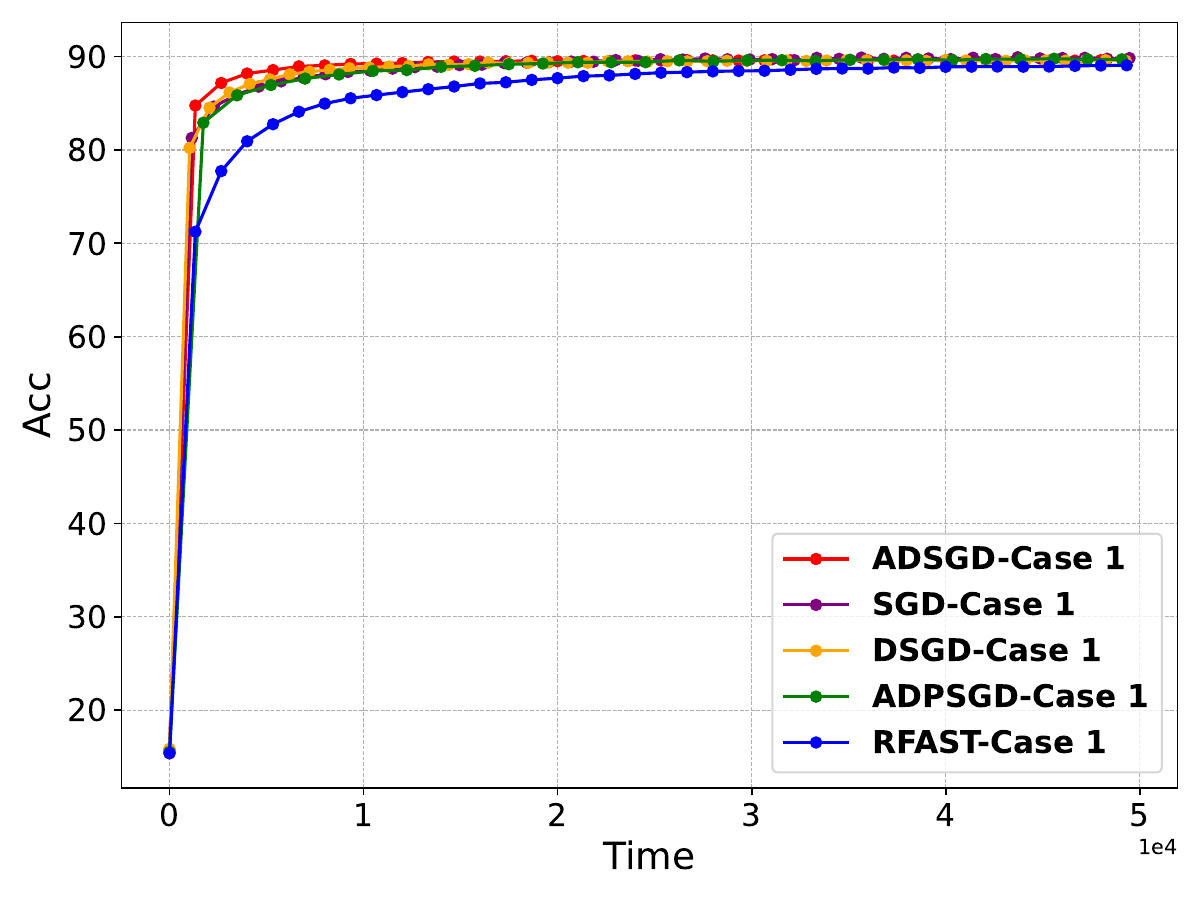}
\vspace{-0.6cm}
\caption{Base}
\label{fig:log_base_acc}
\end{subfigure}
\hspace{-0.15cm}
\begin{subfigure}[t]{0.3\textwidth}
\centering
\includegraphics[width=\textwidth]{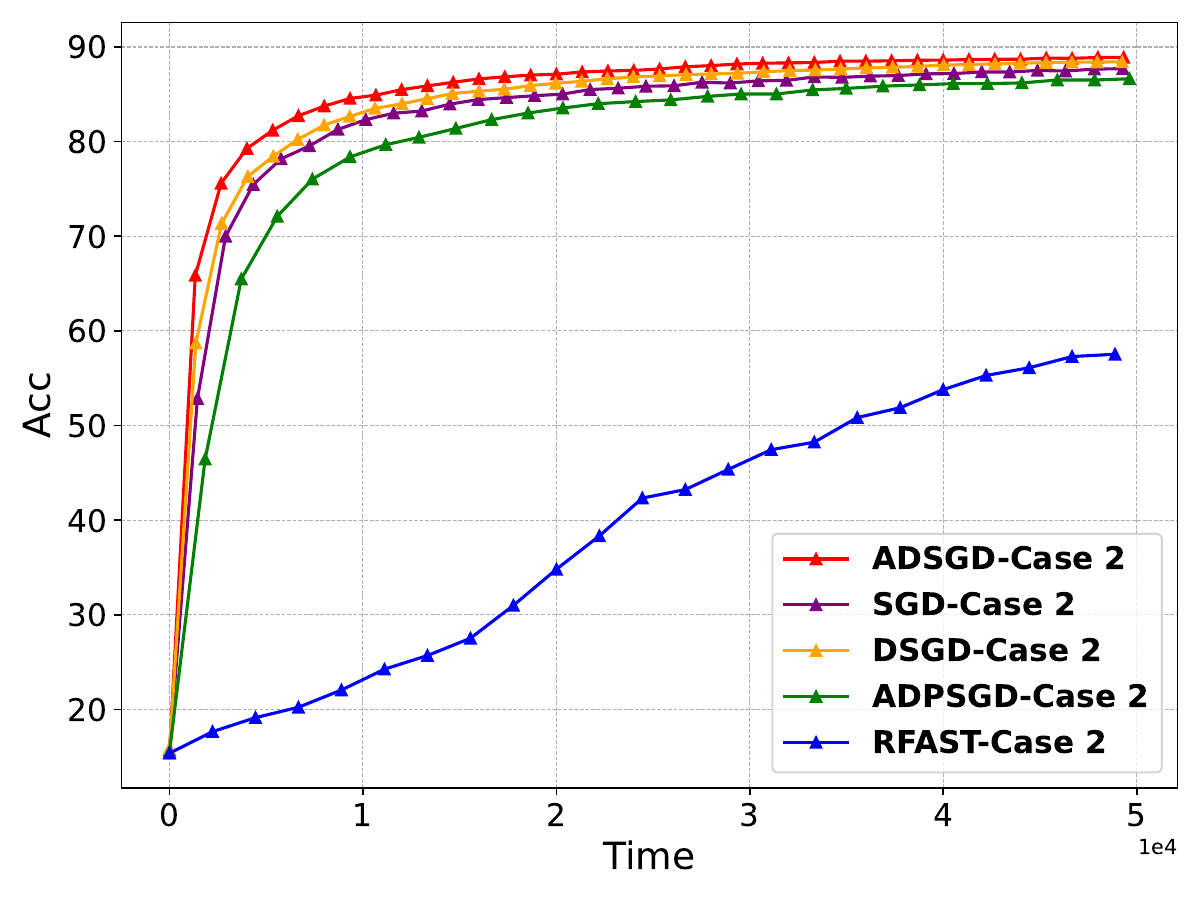}
\vspace{-0.6cm}
\caption{Slow Comm.}
\label{fig:log_slow_acc}
\end{subfigure}
\vspace{-0.25cm}
\caption{Logistic Regression - No Straggler - Test Accuracy - $\zeta=1$}
\label{fig:log_base_slow_acc}
  % \vspace{-0.4cm}
\end{figure}

\begin{figure*}[htbp]
\vspace{-0.3cm}
\centering
\begin{subfigure}[t]{0.3\textwidth}
\centering
\includegraphics[width=\textwidth]{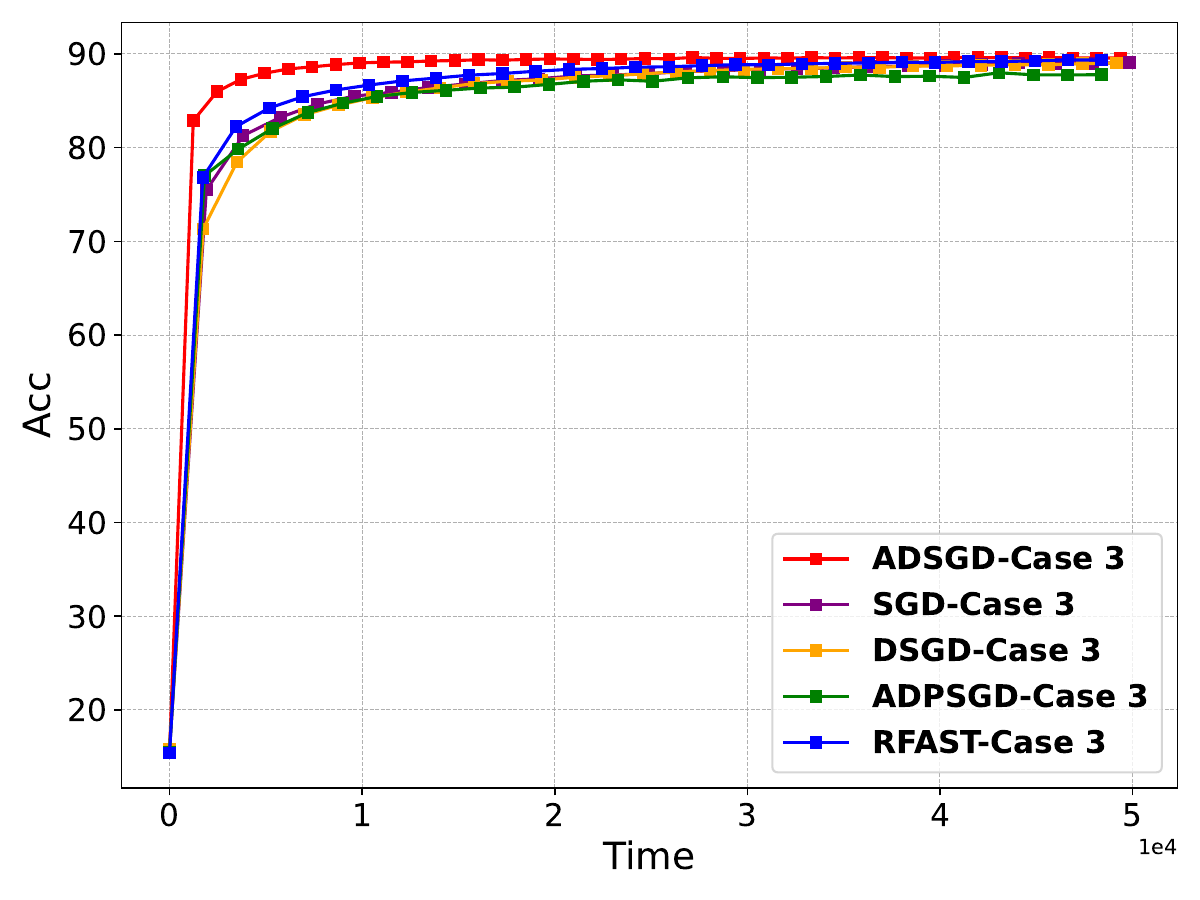}
\vspace{-0.6cm}
\caption{Comp. Straggler} %
\label{fig:log_comp_stra_acc}
\end{subfigure}
% \hfill
\hspace{0.4cm}
\begin{subfigure}[t]{0.3\textwidth}
\centering
\includegraphics[width=\textwidth]{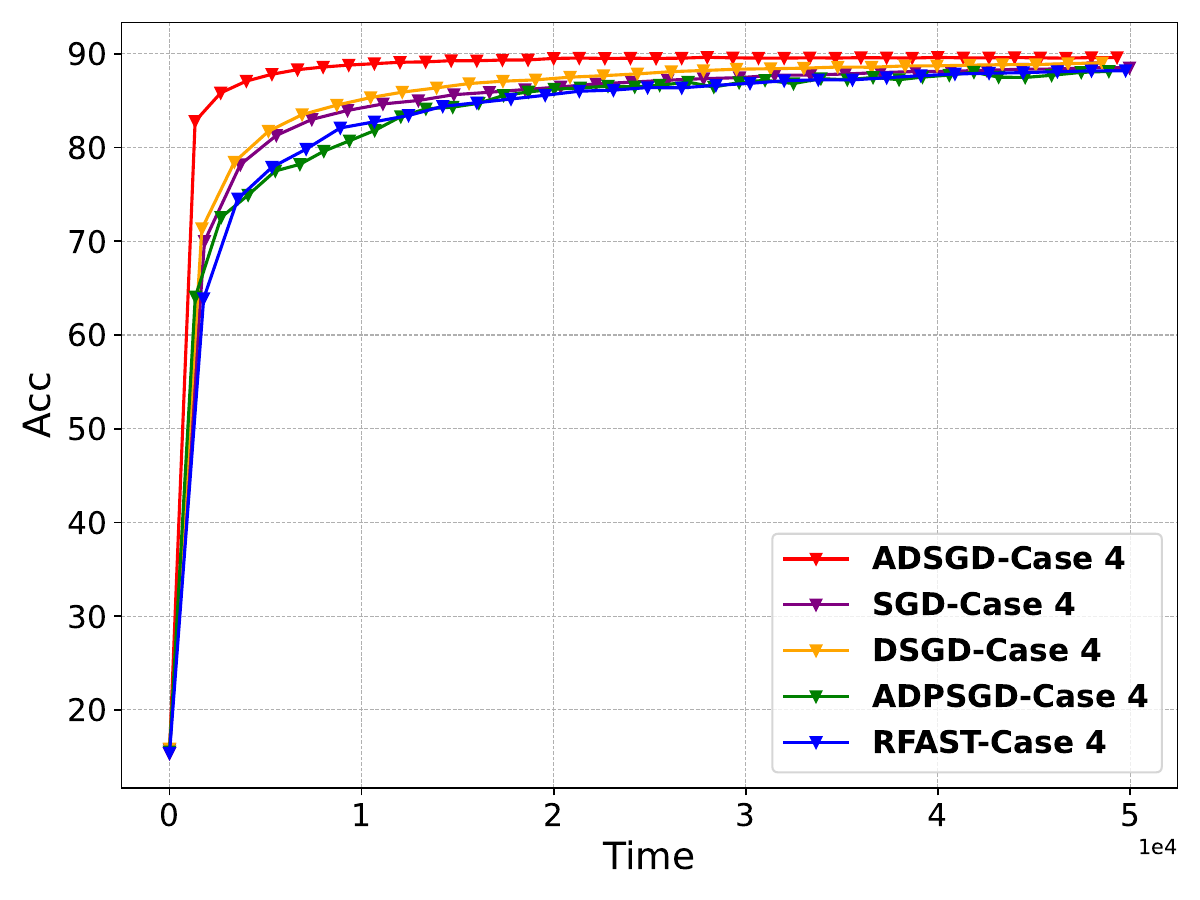}
\vspace{-0.6cm}
\caption{Comm. Straggler}
\label{fig:log_comm_stra_acc}
\end{subfigure}
% \hfill
\hspace{0.4cm}
\begin{subfigure}[t]{0.3\textwidth}
\centering
\includegraphics[width=\textwidth]{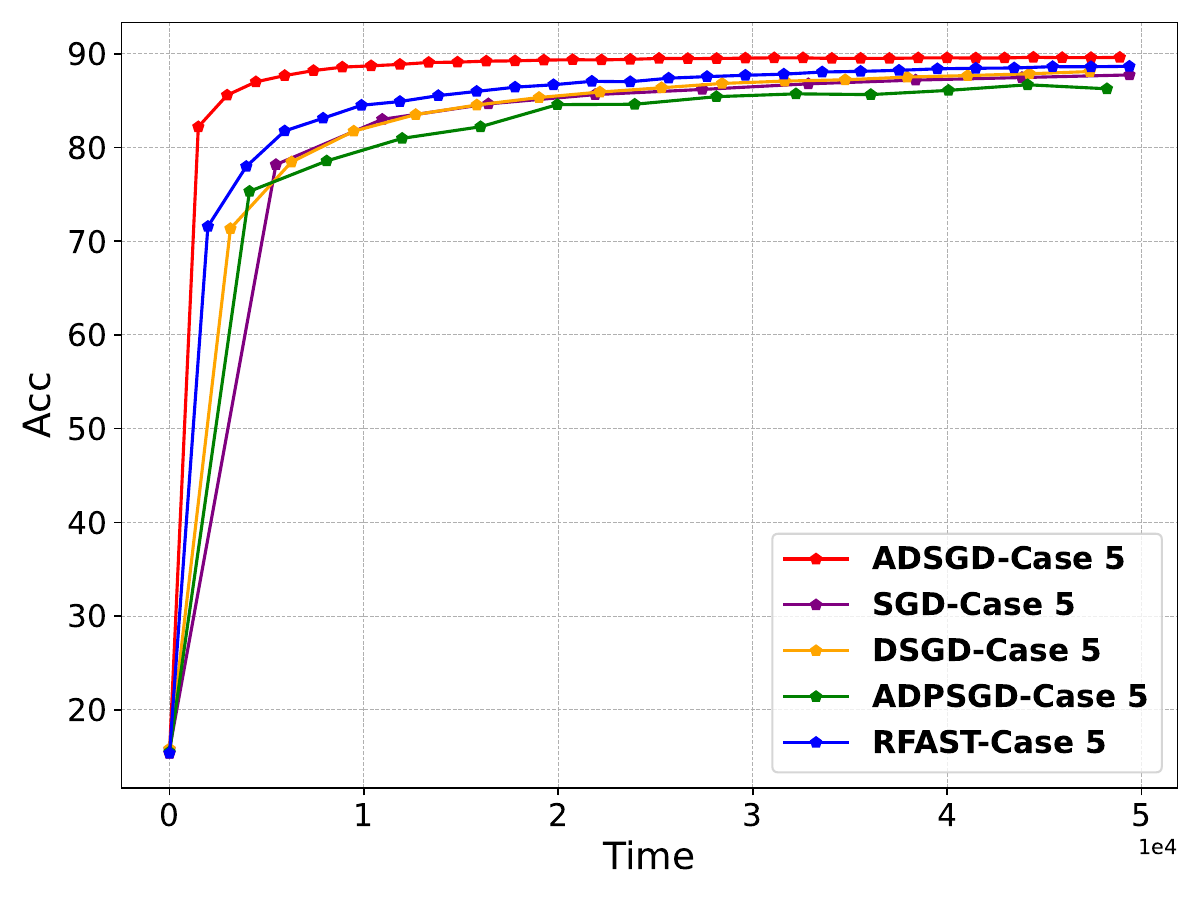}
\vspace{-0.6cm}
\caption{Combined Straggler}
\label{fig:log_combined_stra_acc}
\end{subfigure}
% \hfill
\vspace{-0.25cm}
\caption{Logistic Regression - One Straggler - Test Accuracy - $\zeta=1$}
\label{fig:log_straggler_acc}
  % \vspace{-0.4cm}
\end{figure*}

\begin{figure}[bhtp]
% \vskip 0.2in
\begin{center}
\centerline{\includegraphics[width=0.6\columnwidth]{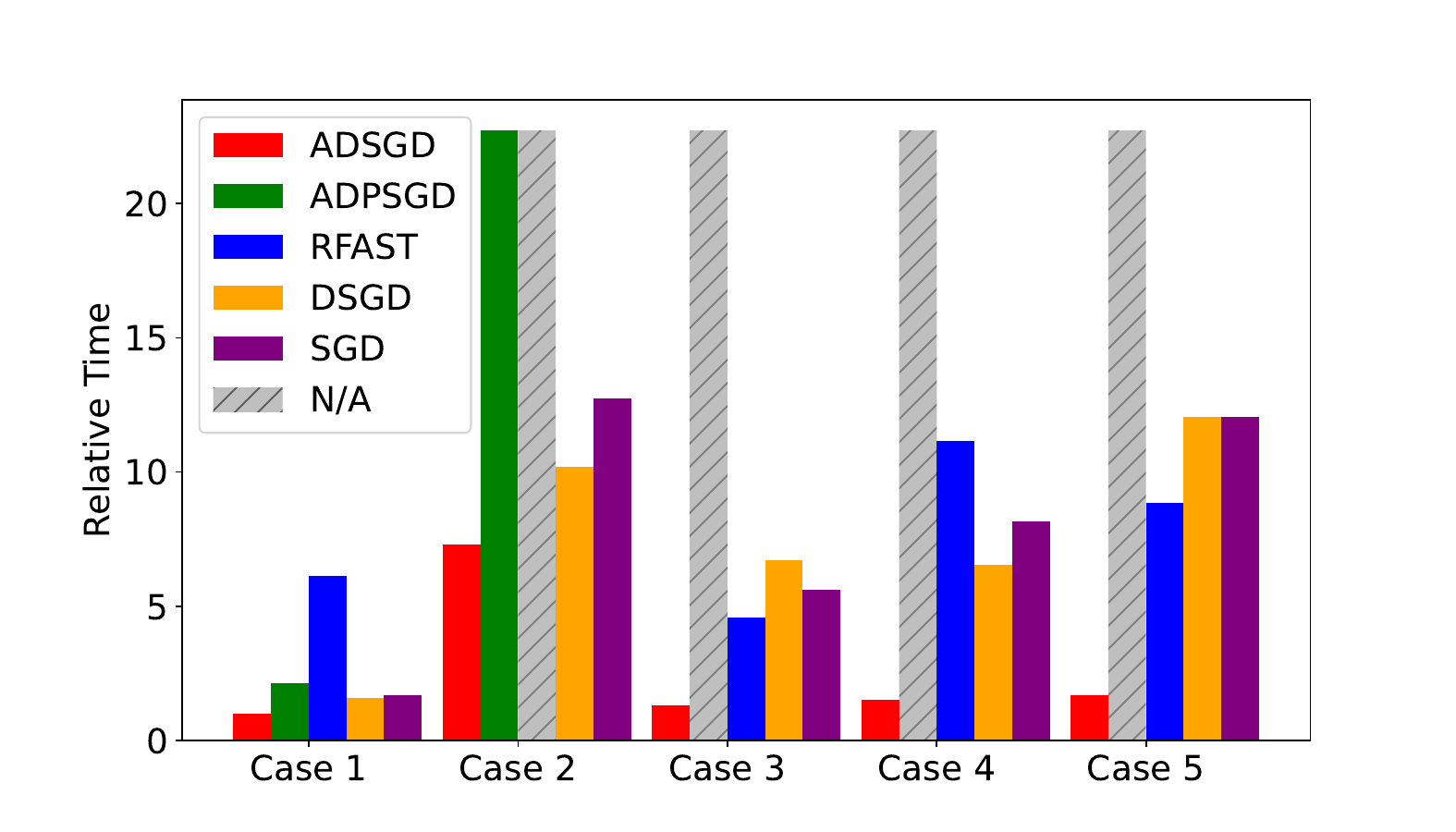}}
\vskip -0.2in
\caption{Relative time (\textbf{lower is better}) to achieve 89\% test accuracy for Non-convex Logistic Regression on MNIST, normalized w.r.t. the runtime of ADSGD Case 1). N/A indicates the algorithm did not reach 89\% accuracy. $\zeta=1$}
\label{fig:LOG_relative_time}
\end{center}
% \vskip -0.2in
\end{figure}

\begin{figure*}[htbp]
\vspace{-0.3cm}
\centering
\begin{subfigure}[t]{0.24\textwidth}
\centering
\includegraphics[width=\textwidth]{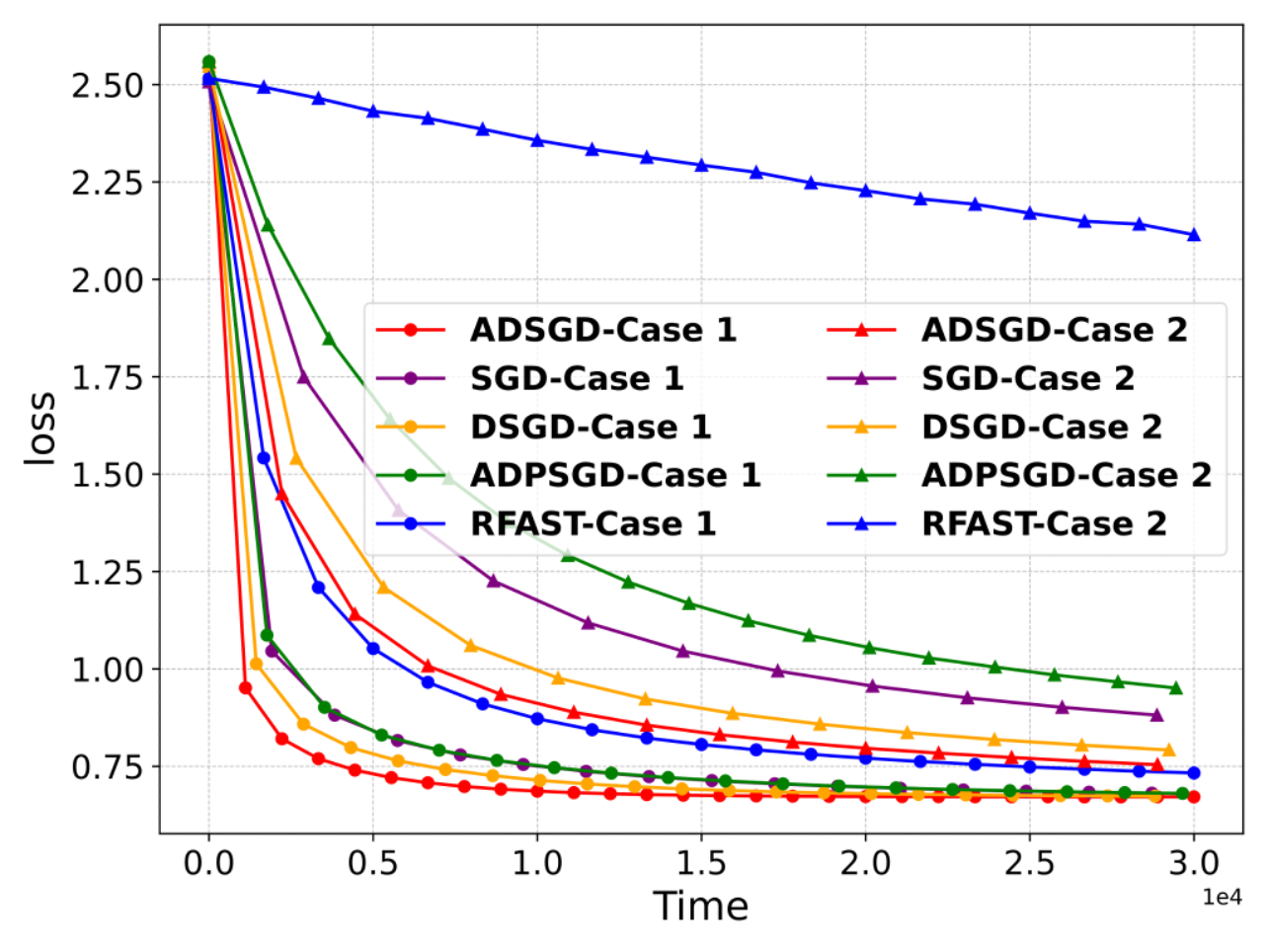}
\vspace{-0.6cm}
\caption{Case 1\&2 - $\zeta=0$} %
\end{subfigure}
% \hfill
% \hspace{0.4cm}
\begin{subfigure}[t]{0.24\textwidth}
\centering
\includegraphics[width=\textwidth]{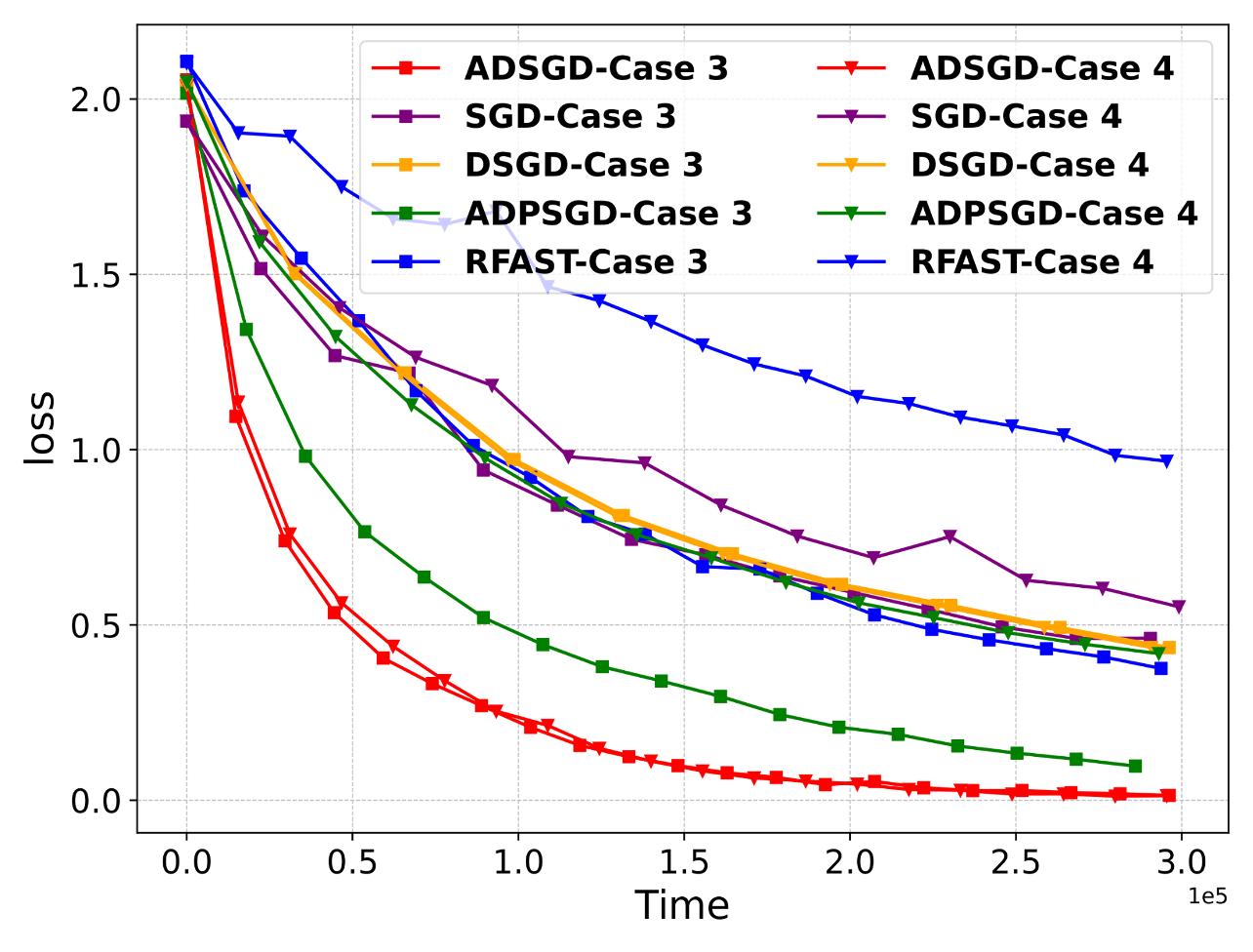}
\vspace{-0.6cm}
\caption{Case 3\&4 - $\zeta=0$}
\end{subfigure}
% \hfill
% \hspace{0.4cm}
\begin{subfigure}[t]{0.24\textwidth}
\centering
\includegraphics[width=\textwidth]{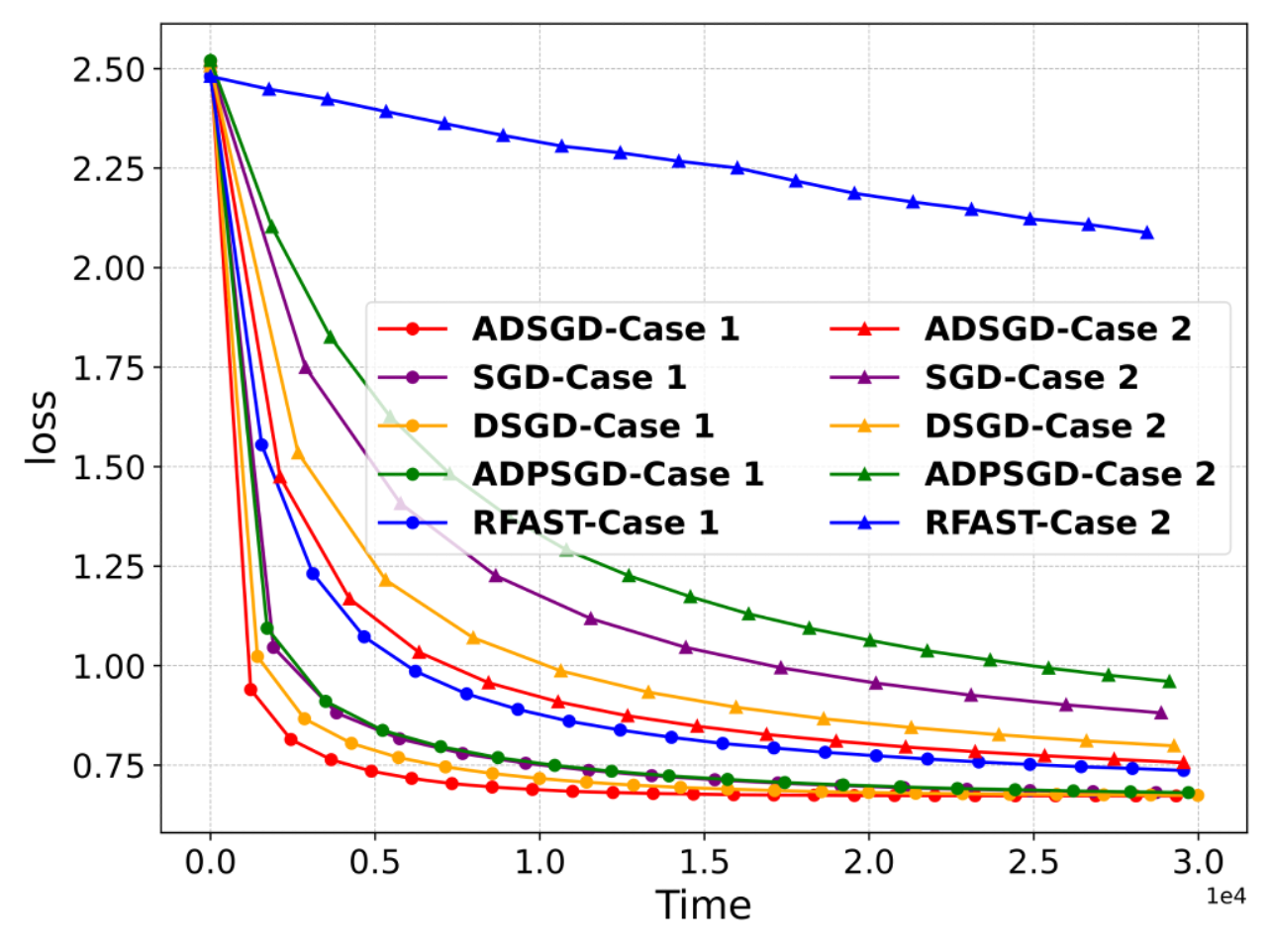}
\vspace{-0.6cm}
\caption{Case 1\&2 - $\zeta=0.5$}
\end{subfigure}
\begin{subfigure}[t]{0.24\textwidth}
\centering
\includegraphics[width=\textwidth]{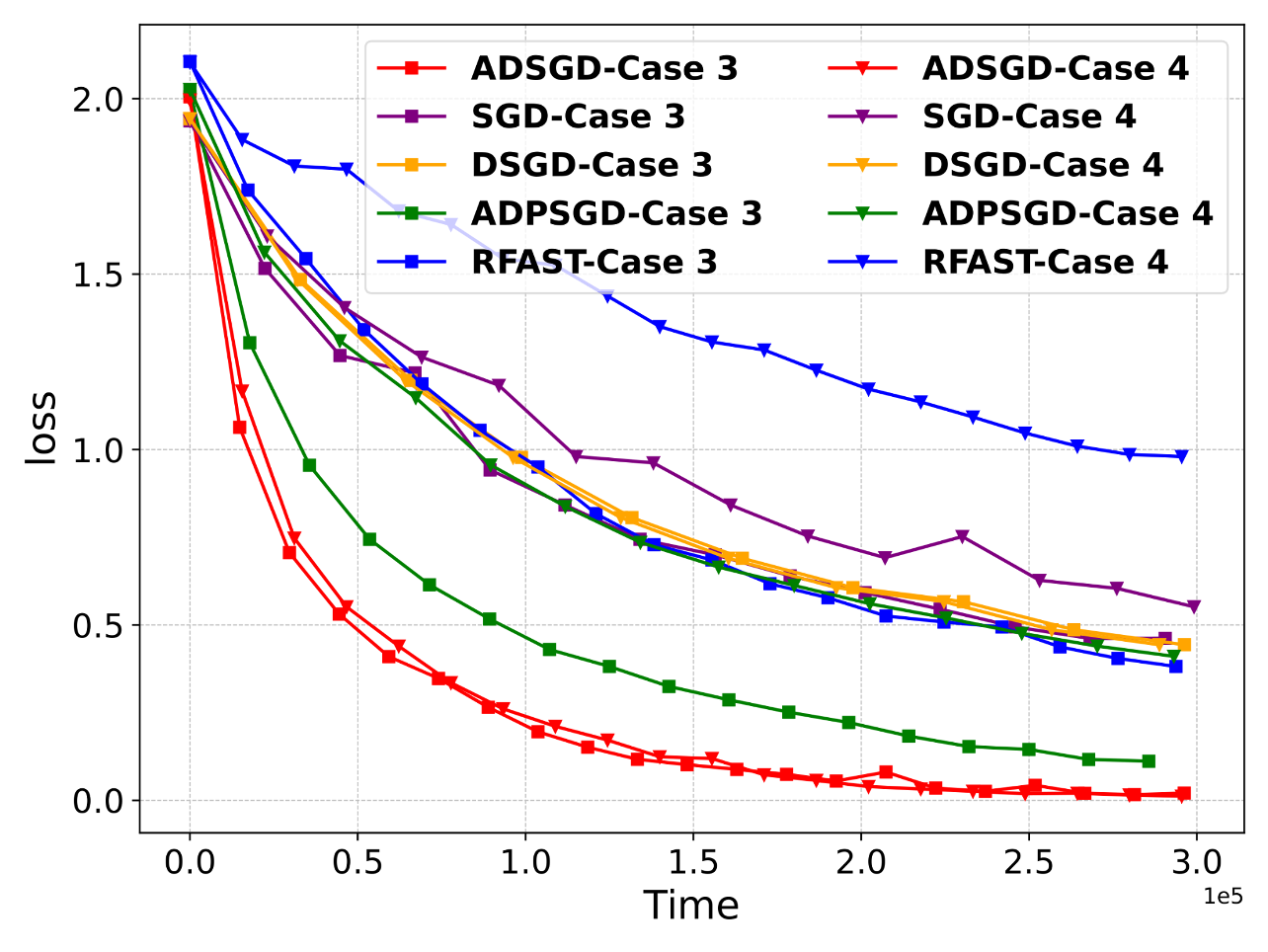}
\vspace{-0.6cm}
\caption{Case 3\&4 - $\zeta=0.5$}
\end{subfigure}
% \hfill
\vspace{-0.25cm}
\caption{Logistic Regression under smaller data heterogeneity levels}
\label{fig:diff_hete_log}
  % \vspace{-0.4cm}
\end{figure*}

\subsection{VGG11 on CIFAR-10}
Fig. \ref{fig:vgg_no_stra_loss} and \ref{fig:vgg_stra_loss} present the training loss w.r.t. runtime for different algorithms. With a combined straggler (Fig. \ref{fig:vgg_combined_stra_loss}), ADSGD maintains its lead by a significant amount, consistent with its performance in logistic regression.

Fig. \ref{fig:vgg_no_stra_acc} and \ref{fig:vgg_stra_acc} present the test accuracy w.r.t. runtime for different algorithms. When there is no straggler, parallel SGD reaches a higher test accuracy given the same amount of time, followed by ADSGD. Note that the loss function of VGG model is highly non-convex, resulting in additional difficulties on adopting stale information and dealing with data heterogeneity. Thus, parallel SGD beats all algorithms under certain cases. When there is a straggler, ADSGD consistently outperforms all other algorithms in most scenarios, except in the case of computation stragglers, where asynchronous algorithms show comparable performance. The accuracy drops observed in RFAST under several conditions highlight its sensitivity to delays.

We quantify the advantage of ADSGD using the relative runtime to achieve 85\% test accuracy. As shown in Fig. \ref{fig:VGG_relative_time}, ADSGD saves at least 15\% of the time compared to other asynchronous algorithms. In the case of comm. straggler, ADSGD saves over 70\% of the time compared to other asynchronous methods. ADSGD also saves from 30\% to 85\% of the time against its synchronous counterpart. Additionally, ADSGD outpaces parallel SGD by 35\% to 58\% under the straggler condition. 

Fig. \ref{fig:diff_hete_vgg} shows convergence for VGG training under different heterogeneity levels ($\zeta\in{0,0.5}$). Mirroring the logistic regression results, ADSGD achieves better performance with lower $\zeta$ values, consistently surpassing all baselines.

\begin{figure}[htbp]
\vspace{-0.3cm}
\centering
\begin{subfigure}[t]{0.3\textwidth}
\centering
\includegraphics[width=\textwidth]{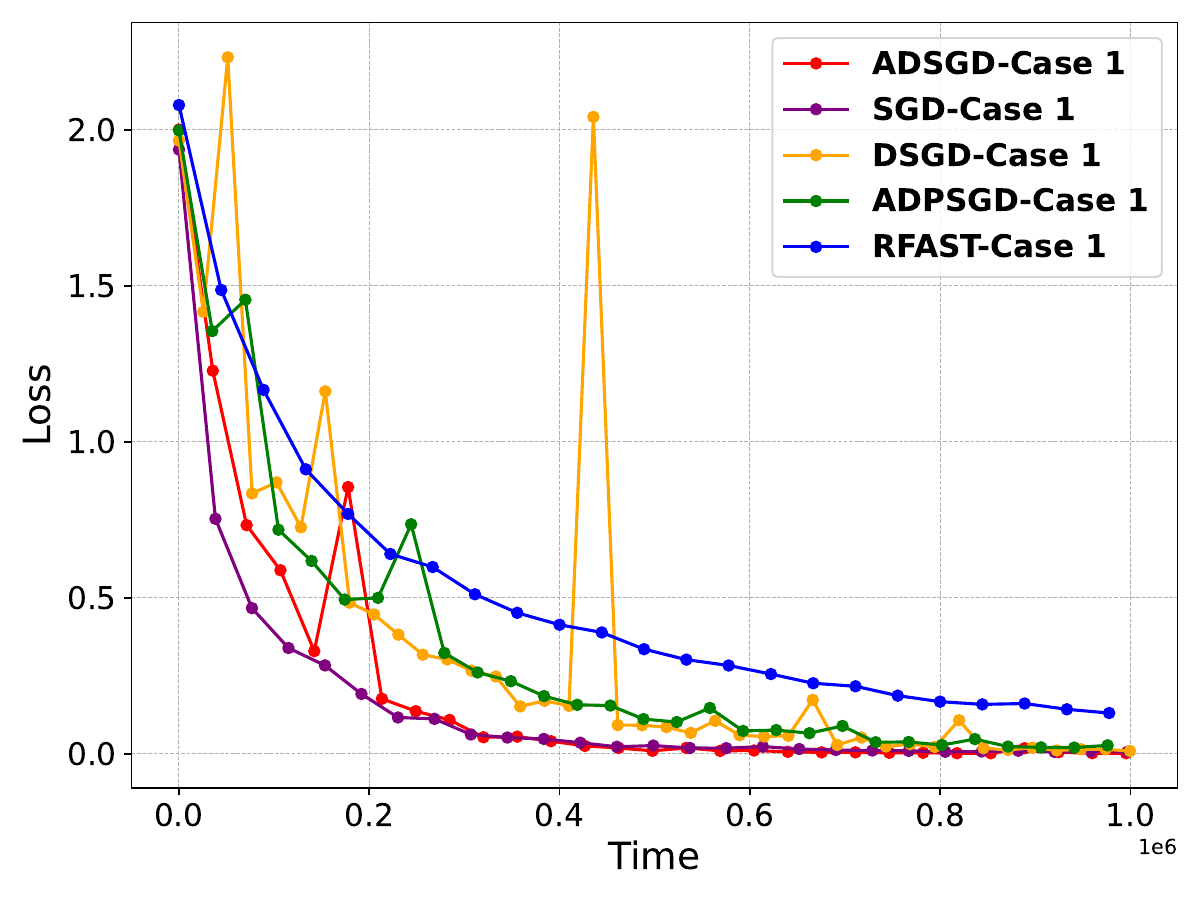}
\vspace{-0.6cm}
\caption{Base}
\label{fig:vgg_base_loss}
\end{subfigure}
\hspace{-0.15cm}
\begin{subfigure}[t]{0.3\textwidth}
\centering
\includegraphics[width=\textwidth]{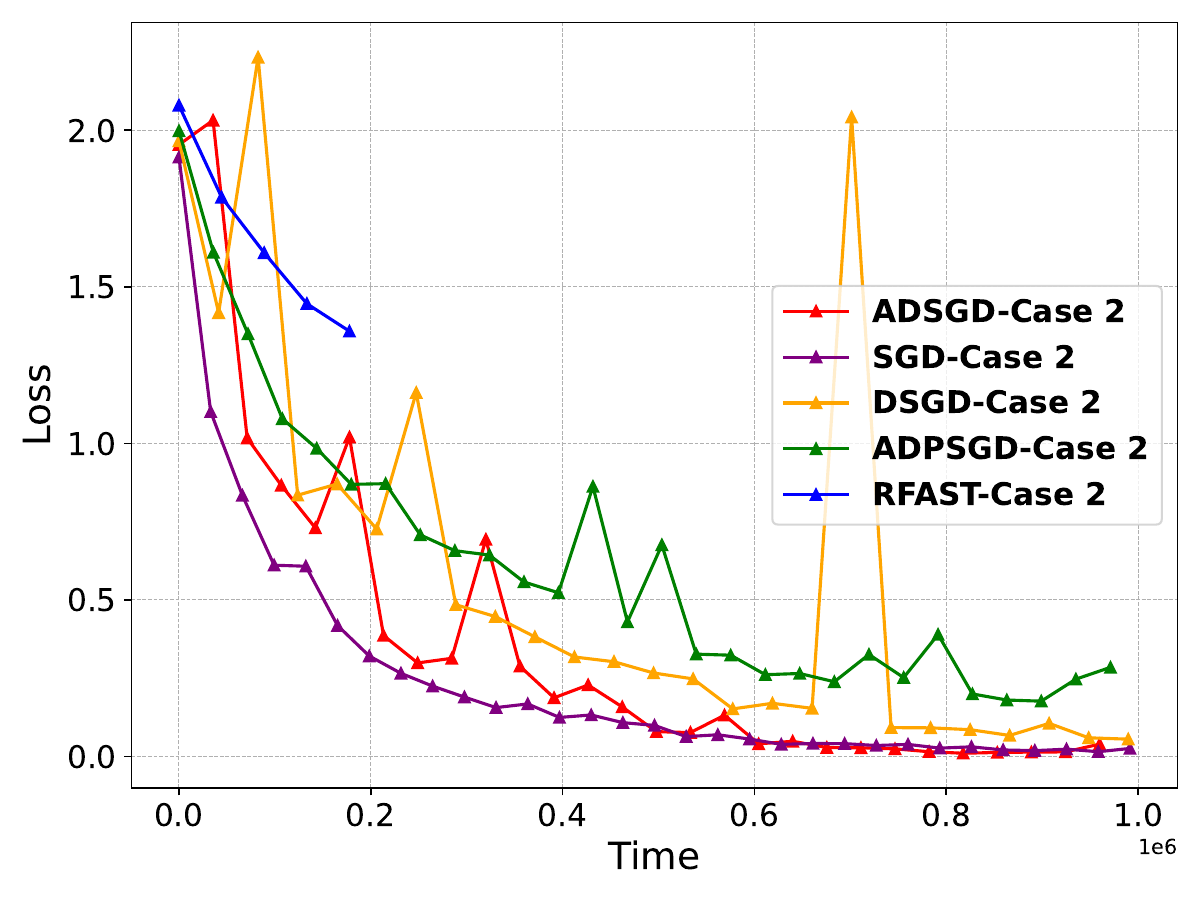}
\vspace{-0.6cm}
\caption{Slow Comm.}
\label{fig:vgg_slow_loss}
\end{subfigure}
\vspace{-0.25cm}
\caption{VGG - No Straggler - Training Loss - $\zeta=1$}
\label{fig:vgg_no_stra_loss}
  % \vspace{-0.4cm}
\end{figure}

\begin{figure*}[htbp]
\vspace{-0.3cm}
\centering
\begin{subfigure}[t]{0.3\textwidth}
\centering
\includegraphics[width=\textwidth]{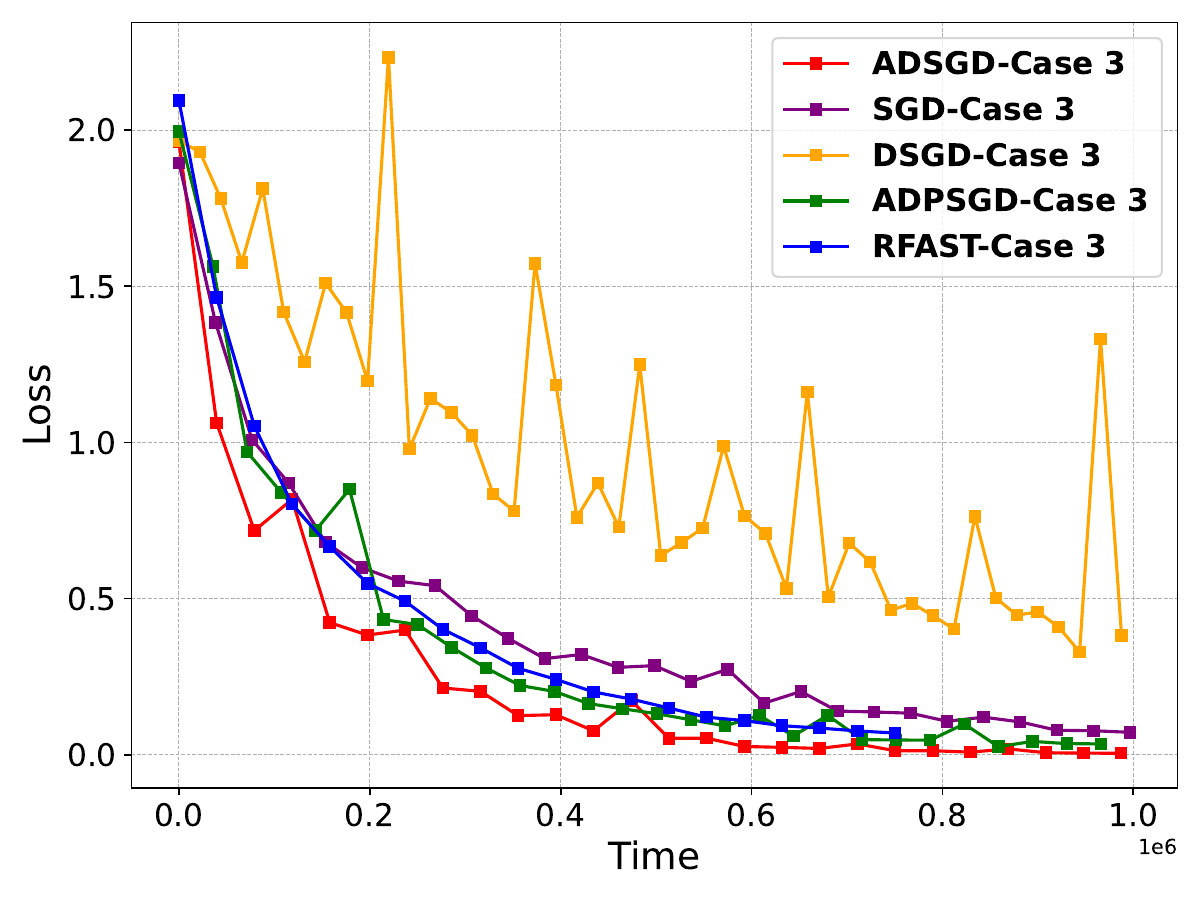}
\vspace{-0.6cm}
\caption{Comp. Straggler} %
\label{fig:vgg_comp_stra_loss}
\end{subfigure}
% \hfill
\hspace{0.4cm}
\begin{subfigure}[t]{0.3\textwidth}
\centering
\includegraphics[width=\textwidth]{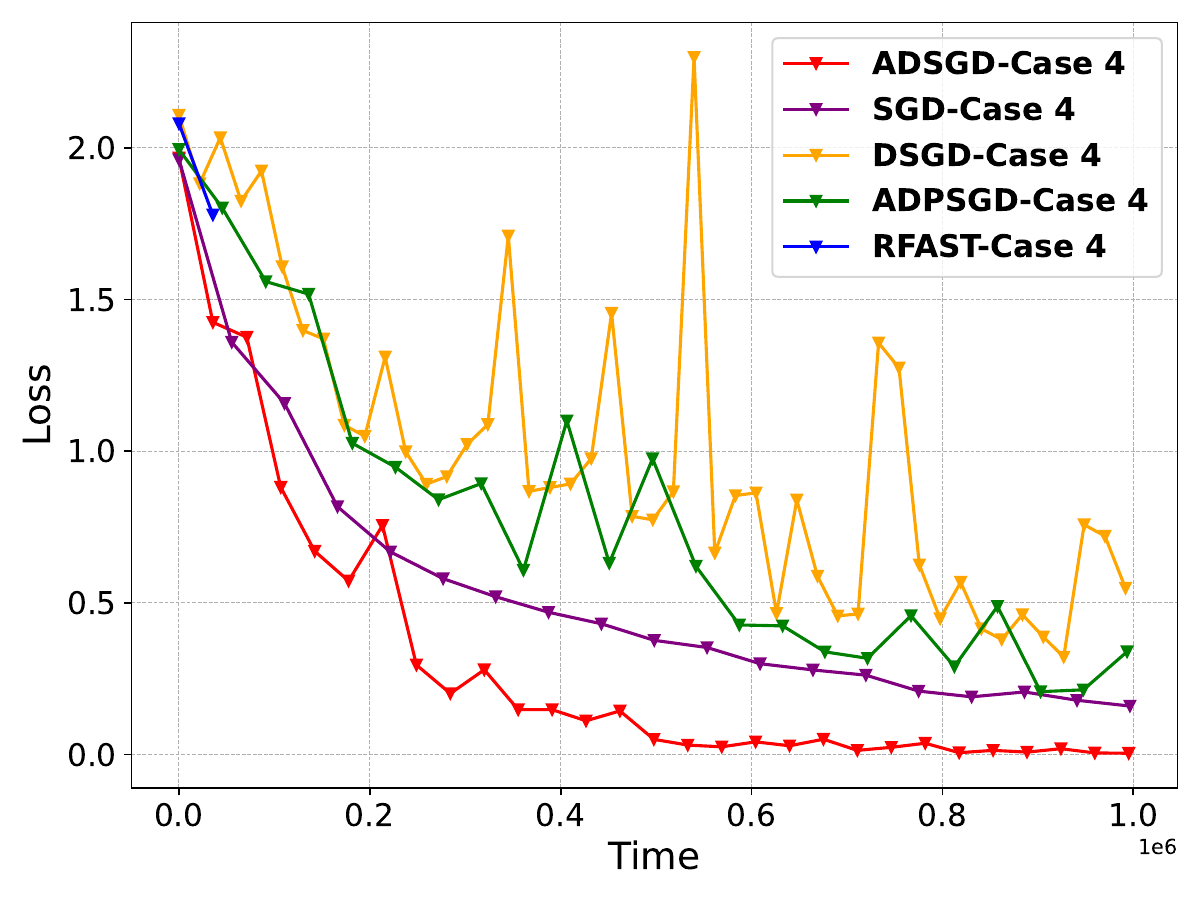}
\vspace{-0.6cm}
\caption{Comm. Straggler}
\label{fig:vgg_comm_stra_loss}
\end{subfigure}
% \hfill
\hspace{0.4cm}
\begin{subfigure}[t]{0.3\textwidth}
\centering
\includegraphics[width=\textwidth]{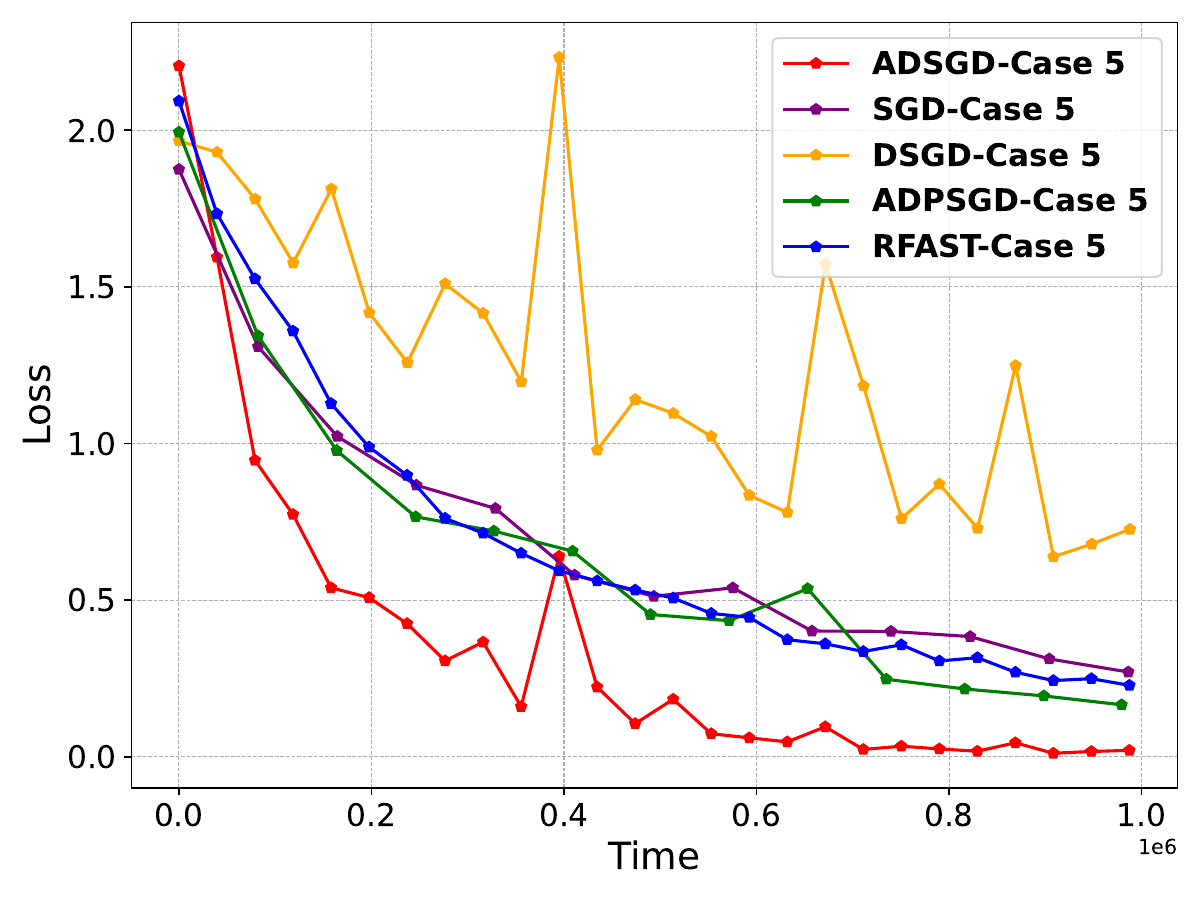}
\vspace{-0.6cm}
\caption{Combined Straggler}
\label{fig:vgg_combined_stra_loss}
\end{subfigure}
% \hfill
\vspace{-0.25cm}
\caption{VGG - One Straggler - Training Loss - $\zeta=1$}
\label{fig:vgg_stra_loss}
  % \vspace{-0.4cm}
\end{figure*}

\begin{figure}[htbp]
\vspace{-0.3cm}
\centering
\begin{subfigure}[t]{0.3\textwidth}
\centering
\includegraphics[width=\textwidth]{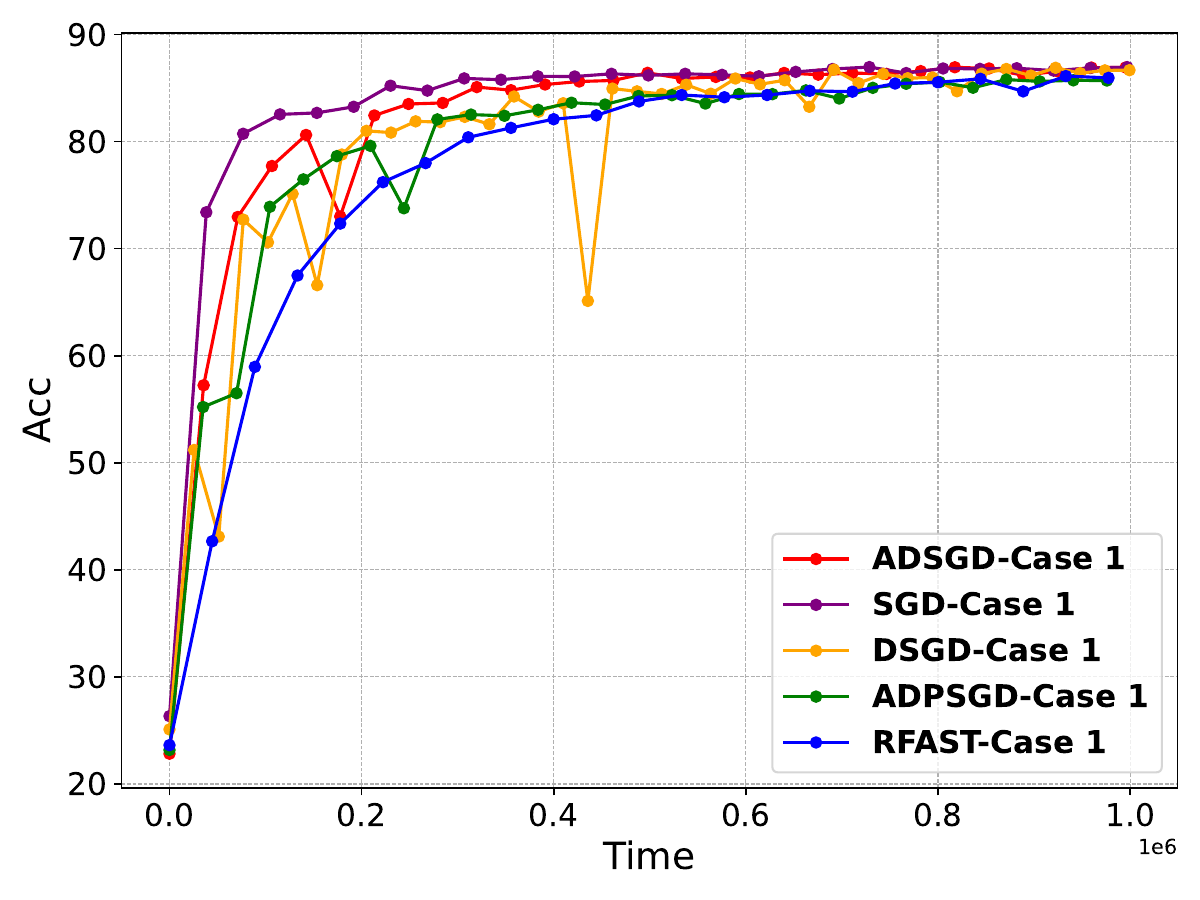}
\vspace{-0.6cm}
\caption{Base}
\label{fig:vgg_base_acc}
\end{subfigure}
\hspace{-0.15cm}
\begin{subfigure}[t]{0.3\textwidth}
\centering
\includegraphics[width=\textwidth]{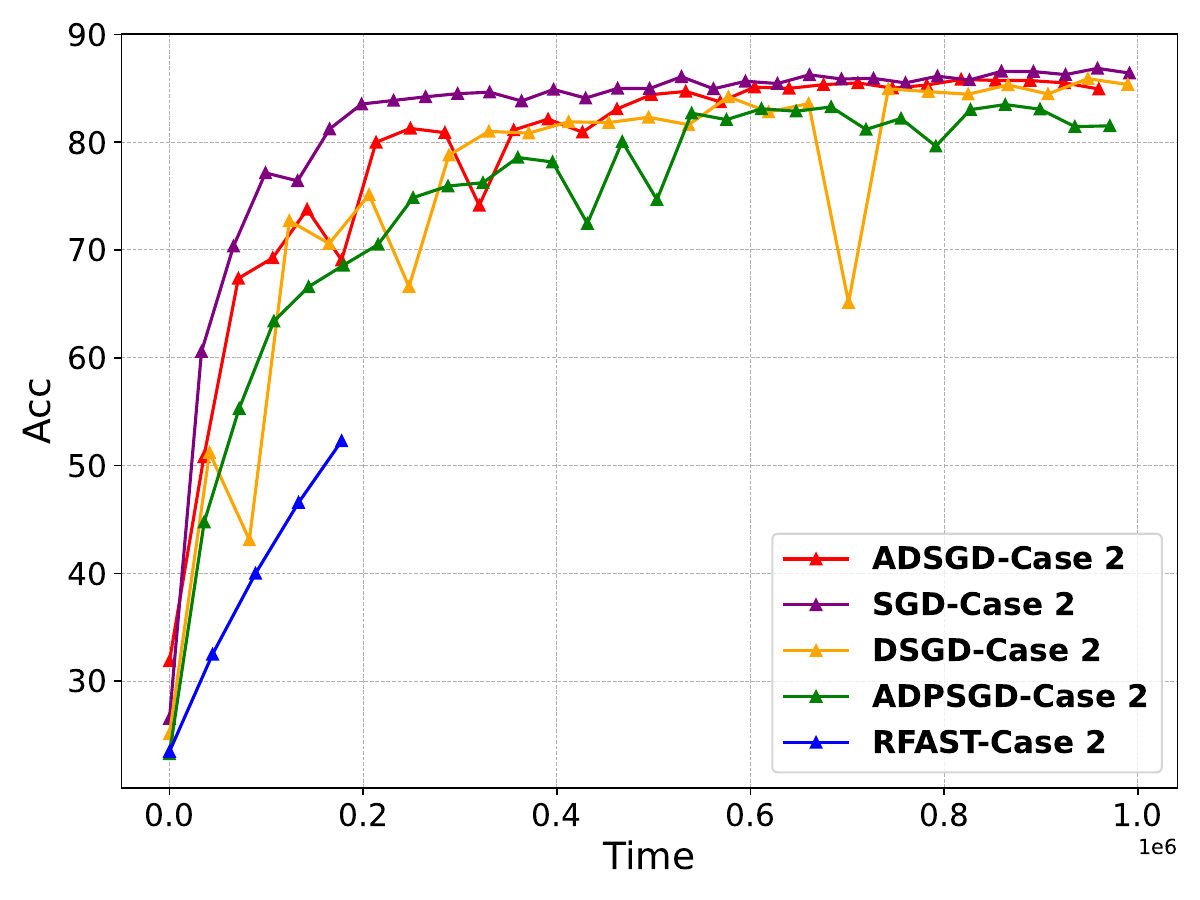}
\vspace{-0.6cm}
\caption{Slow Comm.}
\label{fig:vgg_slow_acc}
\end{subfigure}
\vspace{-0.25cm}
\caption{VGG - No Straggler - Test Accuracy - $\zeta=1$}
\label{fig:vgg_no_stra_acc}
  % \vspace{-0.4cm}
\end{figure}

\begin{figure*}[htbp]
\vspace{-0.3cm}
\centering
\begin{subfigure}[t]{0.3\textwidth}
\centering
\includegraphics[width=\textwidth]{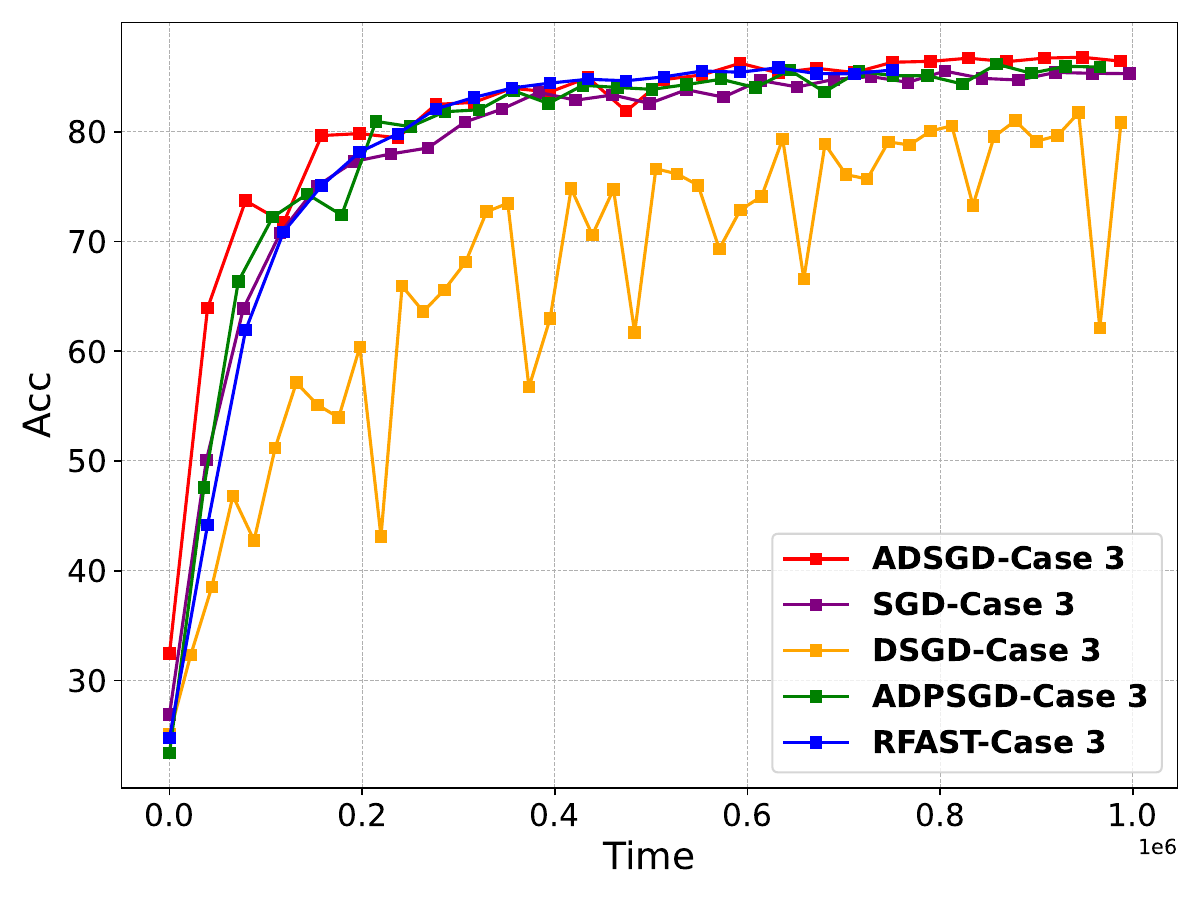}
\vspace{-0.6cm}
\caption{Comp. Straggler} %
\label{fig:vgg_comp_stra_acc}
\end{subfigure}
% \hfill
\hspace{0.4cm}
\begin{subfigure}[t]{0.3\textwidth}
\centering
\includegraphics[width=\textwidth]{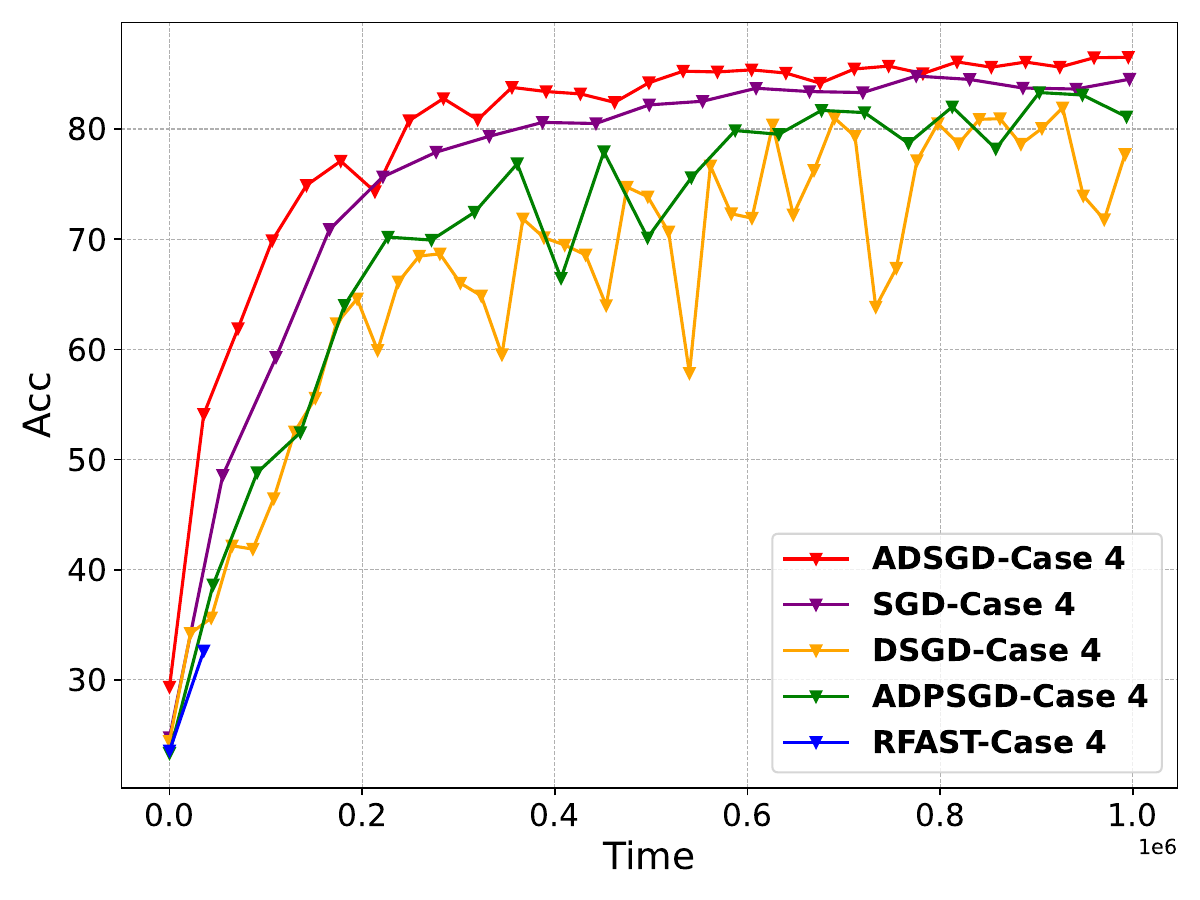}
\vspace{-0.6cm}
\caption{Comm. Straggler}
\label{fig:vgg_comm_stra_acc}
\end{subfigure}
% \hfill
\hspace{0.4cm}
\begin{subfigure}[t]{0.3\textwidth}
\centering
\includegraphics[width=\textwidth]{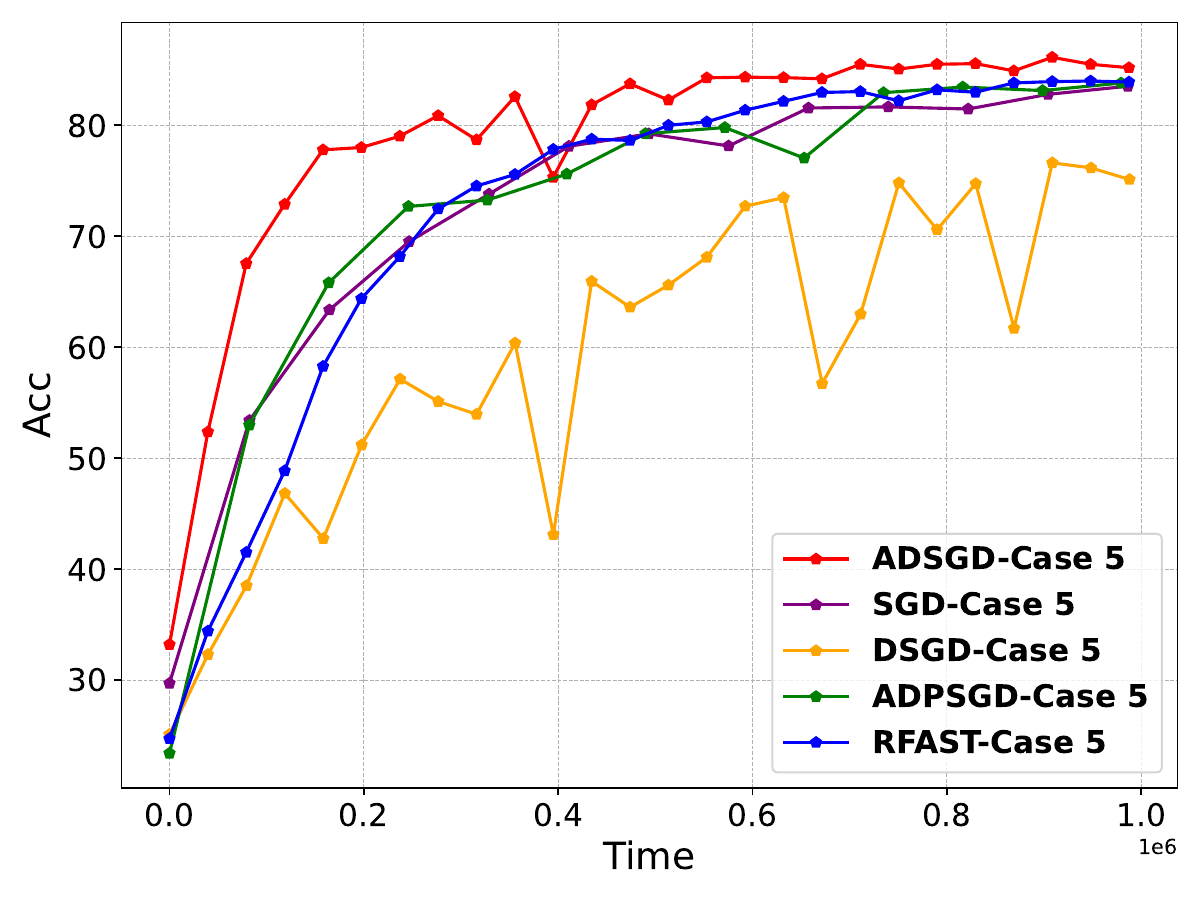}
\vspace{-0.6cm}
\caption{Combined Straggler}
\label{fig:vgg_combined_stra_acc}
\end{subfigure}
% \hfill
\vspace{-0.25cm}
\caption{VGG - One Straggler - Test Accuracy - $\zeta=1$}
\label{fig:vgg_stra_acc}
  % \vspace{-0.4cm}
\end{figure*}

\begin{figure}[ht]
% \vskip 0.2in
\begin{center}
\centerline{\includegraphics[width=0.6\columnwidth]{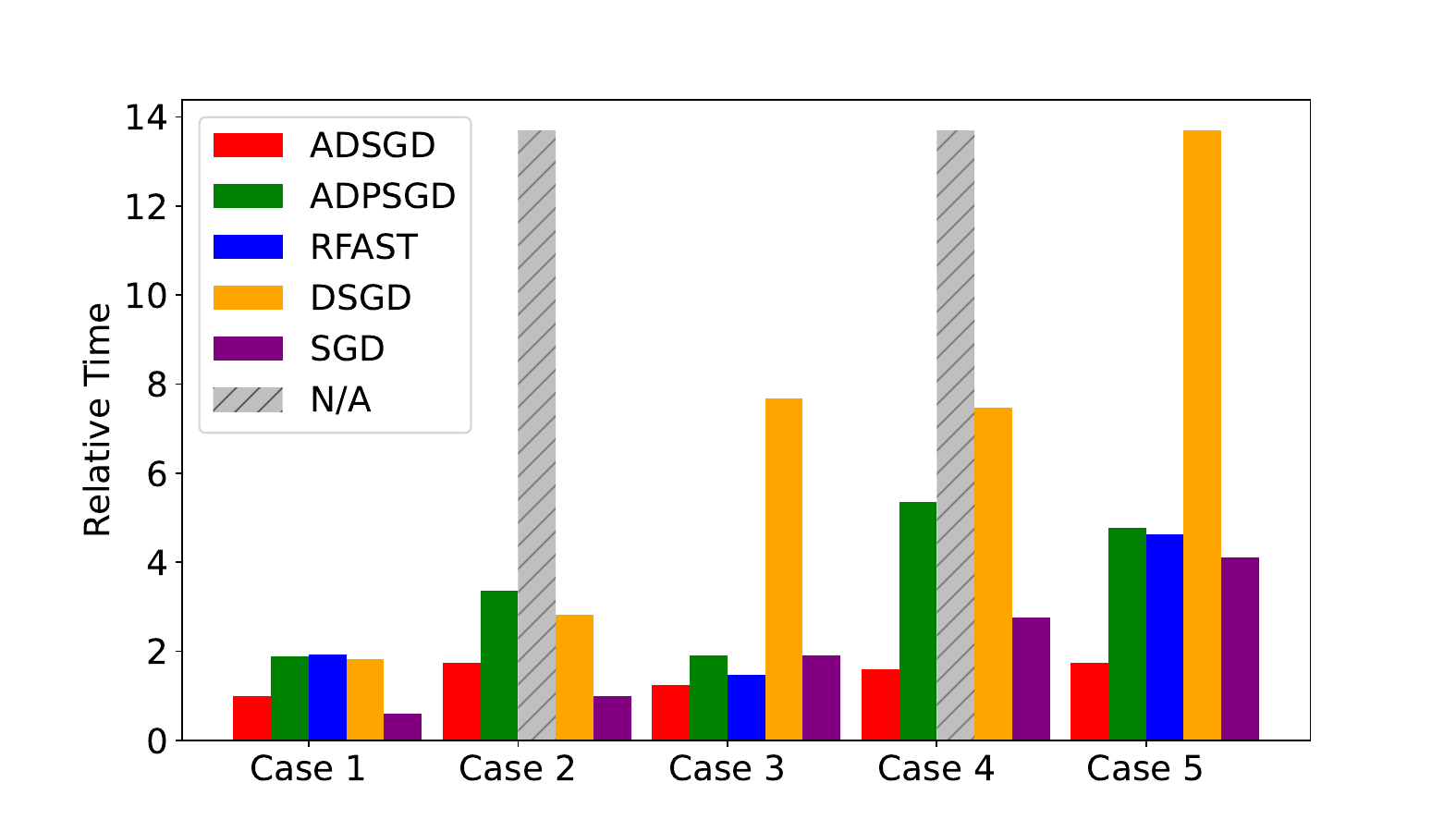}}
\vskip -0.2in
\caption{Relative time (\textbf{lower is better}) to achieve 85\% test accuracy for VGG11 on CIFAR-10, normalized w.r.t. the runtime of ADSGD Case 1). N/A indicates the algorithm did not reach 85\% accuracy. $\zeta=1$}
\label{fig:VGG_relative_time}
\end{center}
\vskip -0.2in
\end{figure}

\begin{figure*}[htbp]
\vspace{-0.3cm}
\centering
\begin{subfigure}[t]{0.24\textwidth}
\centering
\includegraphics[width=\textwidth]{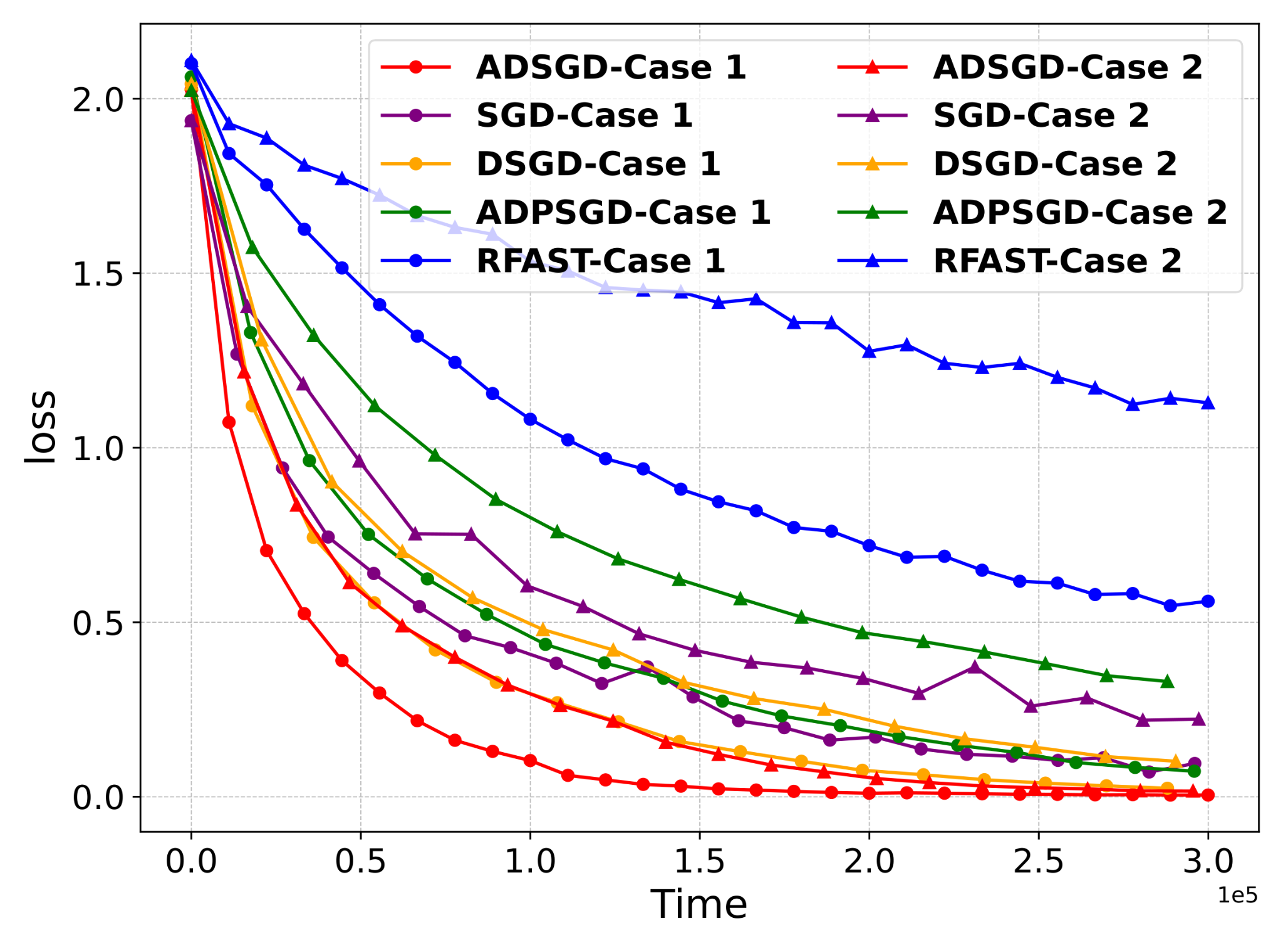}
\vspace{-0.6cm}
\caption{Case 1\&2 - $\zeta=0$} %
\end{subfigure}
% \hfill
% \hspace{0.4cm}
\begin{subfigure}[t]{0.24\textwidth}
\centering
\includegraphics[width=\textwidth]{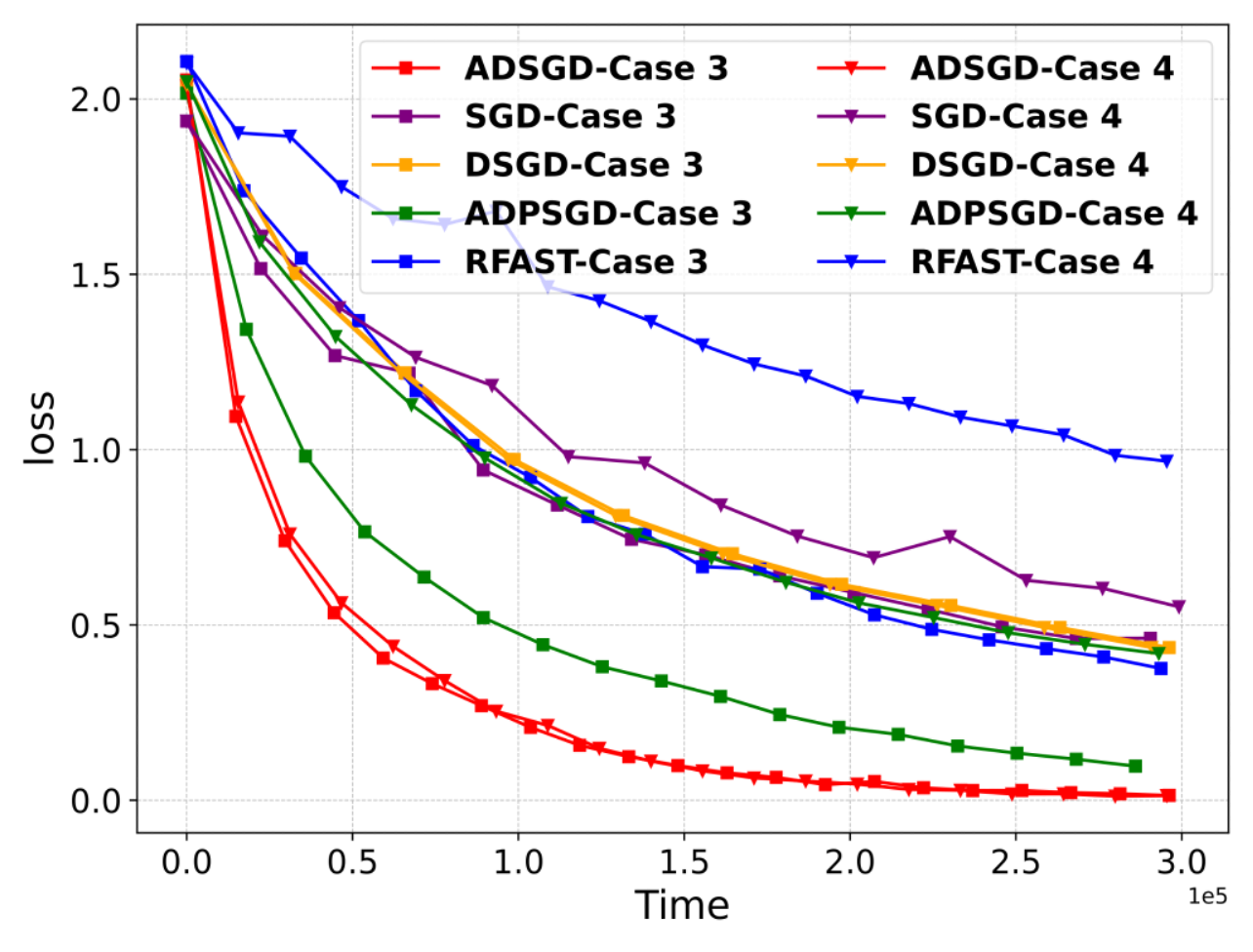}
\vspace{-0.6cm}
\caption{Case 3\&4 - $\zeta=0$}
\end{subfigure}
% \hfill
% \hspace{0.4cm}
\begin{subfigure}[t]{0.24\textwidth}
\centering
\includegraphics[width=\textwidth]{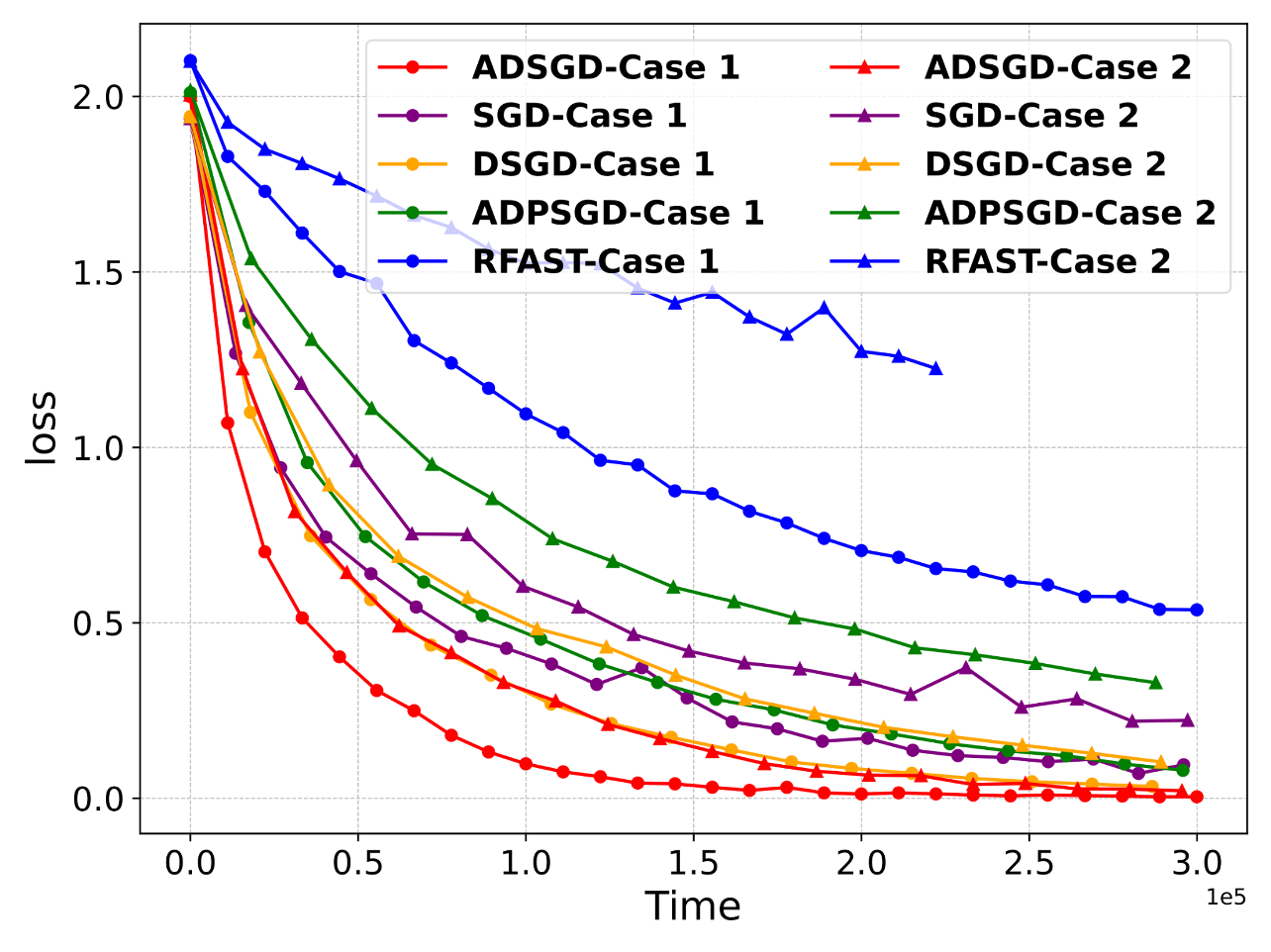}
\vspace{-0.6cm}
\caption{Case 1\&2 - $\zeta=0.5$}
\end{subfigure}
\begin{subfigure}[t]{0.24\textwidth}
\centering
\includegraphics[width=\textwidth]{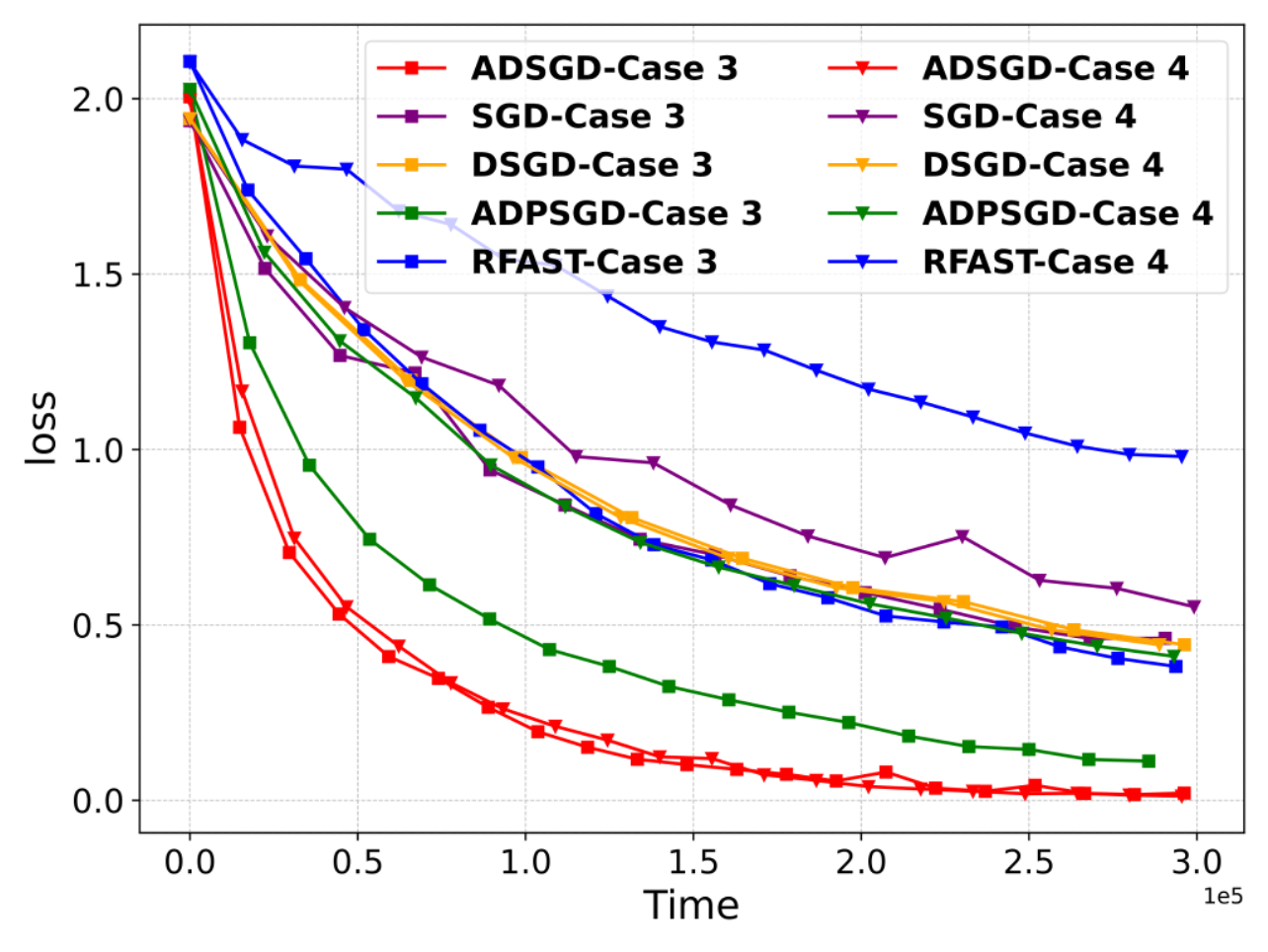}
\vspace{-0.6cm}
\caption{Case 3\&4 - $\zeta=0.5$}
\end{subfigure}
% \hfill
\vspace{-0.25cm}
\caption{VGG training under smaller data heterogeneity levels}
\label{fig:diff_hete_vgg}
  % \vspace{-0.4cm}
\end{figure*}

\section{Convergence Proof}
\label{sec:proof}
\subsection{Proof of Lemma \ref{the:ABCD_convergence}}
    Following a similar way as in \cite{sun2017asynchronous}, let 
    \begin{equation*}
        \xi^k_f \triangleq \mbf{f}(\mbf x^k) + \frac{L}{2\epsilon}\sum_{i=(k-D)^+}^{k-1}(i-(k-D)+1)\|\Delta^i\|^2,
    \end{equation*} 
    where $\Delta^k \triangleq \mbf x^{k+1} - \mbf x^k=-\alpha g^\mbf{f}_{i_k}(\hat{\mbf{x}}^k)$, and we define $d^k \triangleq \mbf x^{k} - \hat{\mbf{x}}^k$.

    We first characterize the relation between $\Delta^k$ and $d^k$. The following proof follows exactly the reasoning in \cite{zhou2018distributed}. However, we found that such a relation holds for a broader class of algorithms.
    \begin{lemma}\label{le:staleness_err_for_block_alg} Let $x_i^k \in \R^{d_i}$ and $\mbf x^k = (x_1^k,...,x_n^k) \in \R^{d'}$. For any algorithm that updates in the following form
    \begin{equation*}
    x_i^{k+1} = \begin{cases}
        T_i(\hat{\mbf{x}}^k), & i=i_k,\\
        x_i^k, & \text{otherwise},
    \end{cases}
    \end{equation*}
    where $\hat{\mbf{x}}^k = (x_1^{t_1^k},...,x_n^{t_n^k})$ and $T_i(\cdot)$ is some mapping from $\R^{d'}$ to $\R^{d_i}$, if $k-t_i^k \le D$ for all $i$ and $k$, we have
    \begin{equation*}
        \|\mbf{x}^k - \hat{\mbf{x}}^k \| \le \sum_{t=(k-D)+}^{k-1} \|\mbf x^{t+1} - \mbf x^t\|.
    \end{equation*}
    \end{lemma}
    \begin{proof}
        \begin{align*}
            \|\mbf{x}^k - \hat{\mbf{x}}^k\| ^2&=\sum_{i=1}^n \|x_i^k - x_i^{t_i^k} \|^2\\
            &\le \sum_{i=1}^n\left(\sum_{t=t_i^k}^{k-1} \|x_i^{t+1}  - x_i^t \|\right)^2\\
            &\le \sum_{i=1}^n\left(\sum_{t=(k-D)^+}^{k-1} \|x_i^{t+1}  - x_i^t \|\right)^2\\
            &= \sum_{i=1}^n \sum_{t=(k-D)^+}^{k-1}\sum_{t'=(k-D)^+}^{k-1} \|x_i^{t+1} - x_i^t \| \|x_i^{t'+1} - x_i^{t'} \|\\
            &=   \sum_{t=(k-D)^+}^{k-1}\sum_{t'=(k-D)^+}^{k-1} \sum_{i=1}^n\|x_i^{t+1} - x_i^t \| \|x_i^{t'+1} - x_i^{t'} \|\\
            &\le  \sum_{t=(k-D)^+}^{k-1}\sum_{t'=(k-D)^+}^{k-1} \|\mbf{x}^{t+1} - \mbf{x}^t \| \|\mbf{x}^{t'+1} - \mbf{x}^{t'} \|\\
            &= \left(\sum_{t=(k-D)^+}^{k-1} \|\mbf{x}^{t+1} - \mbf{x}^t\|  \right)^2
        \end{align*}
    \end{proof}
    \begin{remark}
        Lemma \ref{le:staleness_err_for_block_alg} bounds the staleness error by the sum of magnitude of updates. It holds for any algorithm that obeys a block-updating pattern with staled parameters, e.g. ASBCD and ADSGD.
    \end{remark}

    By Lemma \ref{le:staleness_err_for_block_alg}, we have 
    \begin{equation}
    \label{eq:ABCD_convergence_update_stale}
        \|d^k\| \le \sum_{t=(k-D)^+}^{k-1} \|\Delta^k\| 
    \end{equation}
    
    With the above relation, we are ready to prove the following.
    
    \begin{lemma}\label{the:ABCD_convergence_tool1}
     Given Assumption \ref{asm:b-l_bounded-smooth} - \ref{asm:partialasynchrony}, and $\alpha < \frac{1}{(D + \frac{1}{2})L}$, the sequence $\{\mbf x^k\}$ generated by (\ref{eq:ABCD_update}) satisfies the following relation:

     \begin{align*}
    \frac{\sum_{k=0}^{K-1}\E\|\nabla_{i_k} \mbf{f}(\hat{\mbf{x}}^k)\|^2}{K} 
    &\le \frac{\mbf{f}(\mbf x^0) - \mbf{f}^*}{\alpha(1 - (D+\frac{1}{2})L\alpha)K} + \frac{(D +1)L\alpha}{2(1 - (D+\frac{1}{2})L\alpha)} \sigma^2
\end{align*}
%     \begin{equation*}
%     \min_{0\le k\le K} \|\E \Delta^k\| \sim o\left(\frac{1}{\sqrt{K(\frac{1}{\alpha} - (D+\frac{1}{2})L)}}\right).
%     \end{equation*}
% Namely, $\sqrt{(\frac{1}{\alpha} - (D+\frac{1}{2})L)}\|\E \Delta^k\|$ is square-summable.
    \end{lemma}

\begin{proof}
            We have
    \begin{align}\label{eq:ABCD_lemma_temp1}
   \langle \E\Delta^k,\nabla f(\hat{\mbf{x}}^k)\rangle=-\alpha\| \nabla_{i_k} f(\hat{\mbf{x}}^k) \|^2 = -\frac{1}{\alpha} \|\E \Delta^k\|^2,
\end{align}
and 
\begin{align}\label{eq:ABCD_lemma_temp2}
    Var(\Delta^k) = \alpha^2 \sigma^2,
\end{align}
where the expectation $\E$ is taken over the randomness of the gradient estimator.

By $L$-smoothness, 
\begin{align}\label{eq:ABCD_lemma_L_smooth}
    \mbf{f}(\mbf x^{k+1})\leq \mbf{f}(\mbf x^k)+\langle \nabla \mbf{f}(\mbf x^k),\Delta^k\rangle+\frac{L} {2}\|\Delta^k\|^2_2.
\end{align}

We define the filtration $\mc{F}_{k}$ as a sequence of $\sigma$-algebra that captures all the randomness up to and including the $k$-th iteration.
% Define $\mc{F}_k \triangleq \{x^k, \hat{\mbf{x}}^k\}$.
Take conditional expectation of \eqref{eq:ABCD_lemma_L_smooth}, by Assumption \ref{asm:b-grad_est},
\begin{align*}
    \E[\mbf{f}(\mbf x^{k+1})|\mc{F}_{k}] - \mbf{f}(\mbf x^k)&\overset{\eqref{eq:ABCD_lemma_temp1}}{\leq} \langle \nabla \mbf{f}(\mbf x^k) - \nabla \mbf{f}(\hat{\mbf{x}}^k),\E[\Delta^k|\mc{F}_k]\rangle+\frac{L} {2}\E[\|\Delta^k\|^2|\mc{F}_k]-\frac{1}{\alpha}\|\E[\Delta^k|\mc{F}_k]\|^2 \\
    &\le L \|d^k\|\cdot \|\E[\Delta^k|\mc{F}_k]\| + \frac{L} {2}\E[\|\Delta^k\|^2|\mc{F}_k]-\frac{1}{\alpha}\|\E[\Delta^k|\mc{F}_k]\|^2 \\
    &\le \frac{L}{2\epsilon} \sum_{i=(k-D)^+}^{k-1}\|\Delta^i\|^2 + (\frac{\epsilon L D}{2} -\frac{1}{\alpha})\|E[\Delta^k|\mc{F}_k]\|^2 + \frac{L} {2}\E[\|\Delta^k\|^2|\mc{F}_k],
\end{align*}
where the last inequality is from \eqref{eq:ABCD_convergence_update_stale}.

Therefore, by definition of $\xi^k_f$,
\begin{align}
    \xi^k_f - \E [\xi^{k+1}_f|\mc{F}_k] &\ge \mbf{f}(\mbf x^k) - \E [\mbf{f}(\mbf x^{k+1})|\mc{F}_k]  + \frac{L}{2\epsilon}\sum_{i=(k-D)^+}^{k-1} \|\Delta^i\|^2 - \frac{LD}{2\epsilon} \E [\|\Delta_{k}\|^2|\mc{F}_k]\nonumber \\
    &\ge (\frac{1}{\alpha} - \frac{\epsilon LD}{2} )\|\E[\Delta^k| \mc{F}_k]\|^2 - (\frac{LD}{2\epsilon} + \frac{L}{2}) \E[\|\Delta^k\|^2|\mc{F}_k]\nonumber\\
    &= (\frac{1}{\alpha} - \frac{\epsilon LD}{2}- \frac{LD}{2\epsilon} - \frac{L}{2}) \|\E[\Delta^k| \mc{F}_k]\|^2  - (\frac{LD}{2\epsilon} + \frac{L}{2}) Var(\Delta^k)\nonumber\\
    &\overset{\epsilon=1}{\ge} (\frac{1}{\alpha} - \frac{L}{2} - D L)\alpha^2\|\nabla_{i_k} \mbf{f}(\hat{\mbf{x}}^k)\|^2 - \frac{L(D+1)}{2} \alpha^2 \sigma^2 \label{eq:ABCD_lemma_proof_xi_descent}
\end{align}

Take full expectation of the above and sum over $k$,

\begin{align*}
    (\frac{1}{\alpha} - \frac{L}{2} - D L)\alpha^2\frac{\sum_{k=0}^{K-1}\E\|\nabla_{i_k} \mbf{f}(\hat{\mbf{x}}^k)\|^2}{K}  &\le \frac{\xi^0_f - \xi^{K}_f}{K} + \frac{L(D +1)}{2} \alpha^2\sigma^2\\
    &\le \frac{\mbf{f}(\mbf x^0) - \mbf{f}^*}{K} + \frac{L(D +1)}{2} \alpha^2\sigma^2
\end{align*}

\end{proof}

Lemma \ref{the:ABCD_convergence_tool1} characterizes the convergence of $\E\|\nabla_{i_k} f(\hat{\mbf{x}}^k)\|^2$, which implies the convergence of the ABCD algorithm. Based on the lemma, we now derive the convergence rate of $\E\|\nabla f(\mbf x^k)\|^2$.

% Now, we follow similar steps in \cite{sun2017asynchronous} to show that $\frac{\sqrt{\frac{1}{\alpha} - (D + \frac{1}{2})L}}{\frac{1}{\alpha}+L}\|\E \nabla_i f(x^k)\|$ is square-summable given identical conditions in Lemma \ref{the:ABCD_convergence_tool1}.

Note that 
\begin{align}
    E\|\nabla_i \mbf{f}(\mbf x^k)\|^2 &\le 3(\E\|\nabla_i \mbf{f}(\hat{\mbf{x}}^{t_i(k})\|^2 + \E\|\nabla_i \mbf{f}(\mbf x^k) - \nabla_i \mbf{f}(\hat{\mbf{x}}^k)\|^2 +  \E \|\nabla_i \mbf{f}(\hat{\mbf{x}}^k) - \nabla_i \mbf{f}(\hat{\mbf{x}}^{t_i(k})\|^2)\nonumber\\
    &\le 3(\E\|\nabla_{i} \mbf{f}(\hat{\mbf{x}}^{t_i(k)})\|^2+ L^2 \E\|d^k\|^2 + L^2B\sum_{j=t_i(k)}^{k-1} \E\|\hat{\mbf{x}}^{j+1} - \hat{\mbf{x}}^j\|^2).\label{eq:ABCD_lemma_proof_split_nab_f}
\end{align}

For the second term on the RHS,
\begin{align}
    \E \|d^k\|^2 &\le D \sum_{j=(k-D)^+}^{k-1} \E \|\Delta^j\|^2\nonumber\\
    & =D \alpha^2 \sum_{j=(k-D)^+}^{k-1}(\sigma^2 + \E \|\nabla_{i_j} \mbf{f}(\hat{\mbf{x}}^j)\|^2)\label{eq:ABCD_lemma_proof_staleness}.
\end{align}

For the third term,
\begin{align}
    \sum_{j=t_i(k)}^{k-1} \E\|\hat{\mbf{x}}^{j+1} - \hat{\mbf{x}}^j\|^2 &\le \sum_{j=t_i(k)}^{k-1} \left( 3\E\|d^j\|^2 + 3\E\|d^{j+1}\|^2 + 3\E\|\Delta^j\|^2\right)\nonumber\\
    &\overset{\eqref{eq:ABCD_lemma_proof_staleness}}{\le}3\sum_{j=t_i(k)}^{k-1} \Bigl(D \sum_{l=(j-D)^+}^{j-1}(\E\|\Delta^l\|^2 + \E \|\Delta^{l+1}\|^2)  + \E \|\Delta^j\|^2 \Bigr)\nonumber \\
    &\le 3\sum_{j=t_i(k)}^{k-1} \Bigl(D \sum_{l=(j-D)^+}^{j-1}(\alpha^2\sum_{m=(l-D)^+}^{l-1}(2\sigma^2 + \E\|\nabla_{i_m}\mbf{f}(\hat{\mbf{x}}^m)\|^2 + \nonumber\\
    &\E\|\nabla_{i_{m+1}} \mbf{f}(\hat{\mbf{x}}^{m+1})\|^2))  +\alpha^2\sum_{l=(j-D)^+}^{j-1}(\sigma^2 + \E\|\nabla_{i_l}\mbf{f}(\hat{\mbf{x}}^l)\|^2)\Bigr)\label{eq:ABCD_lemma_proof_sum_hat_difference}
\end{align}

By \eqref{eq:ABCD_lemma_proof_staleness} and the following inequality,
\begin{equation*}
    \sum_{k=0}^{K-1} \sum_{j=(k-D)^+}^{k-1} a_j \le D \sum_{k=0}^{K-1} a_k,
\end{equation*}
we have 
\begin{equation}
    \sum_{k=0}^{K-1} \E \|d^k\|^2 \le D^2 \alpha^2 \left( \sum_{k=0}^{K-1}\E \|\nabla_{i_k} \mbf{f}(\hat{\mbf{x}}^k)\|^2 + K\sigma^2  \right).\label{eq:ABCD_lemma_proof_staleness_nab}
\end{equation}

Similarly, from \eqref{eq:ABCD_lemma_proof_sum_hat_difference}, 
\begin{align}
    \sum_{k=0}^{K-1}\sum_{j=t_i(k)}^{k-1} \E\|\hat{\mbf{x}}^{j+1} - \hat{\mbf{x}}^j\|^2 &\le 3B\Bigl( 2D^3 \alpha^2 (\sum_{k=0}^{K-1}\E\|\nabla_{i_k} \mbf{f}(\hat{\mbf{x}}^k)\|^2 + K\sigma^2)\nonumber \\
    &+ D \alpha^2 (\sum_{k=0}^{K-1}\E\|\nabla_{i_k} \mbf{f}(\hat{\mbf{x}}^k)\|^2 + K\sigma^2)\Bigr)\nonumber \\
    &= 3B(D+2D^3)\alpha^2\left(\sum_{k=0}^{K-1}\E\|\nabla_{i_k} \mbf{f}(\hat{\mbf{x}}^k)\|^2 + K\sigma^2\right)\label{eq:ABCD_lemma_proof_sum_hat_difference_nab}
\end{align}

Moreover, since each agent updates at least once every $B$ steps, we have
\begin{equation}
\label{eq:ABCD_lemma_proof_nab_i_nab_i_k}
    \sum_{k=0}^{K-1} \|\nabla_i \mbf{f}(\hat{\mbf{x}}^{t_i(k)})\|^2 \le B \sum_{k=0}^{K-1} \|\nabla_{i_k} \mbf{f}(\hat{\mbf{x}}^k)\|^2.
\end{equation}

 Combining \eqref{eq:ABCD_lemma_proof_split_nab_f} and \eqref{eq:ABCD_lemma_proof_staleness_nab} to \eqref{eq:ABCD_lemma_proof_nab_i_nab_i_k},
\begin{align*}
    \frac{\sum_{k=0}^{K-1}\E\|\nabla \mbf{f}(\mbf x^k)\|^2}{K} &= \frac{\sum_{k=0}^{K-1}\sum_{i=1}^n \E\|\nabla_i \mbf{f}(\mbf x^k)\|^2}{K}\\
    &\le 3n\left( \frac{\sum_{k=0}^{K-1}\E\|\nabla_i \mbf{f}(\hat{\mbf{x}}^{t_i(k)}\|^2}{K} + C_0L^2\alpha^2 (\frac{\sum_{k=0}^{K-1}\E\|\nabla_{i_k} \mbf{f}(\hat{\mbf{x}}^k)\|^2}{K}+ \sigma^2)\right)\\
    &\le 3n\left((B+ C_0L^2\alpha^2)\frac{\sum_{k=0}^{K-1}\E\|\nabla_{i_k} \mbf{f}(\hat{\mbf{x}}^k)\|^2}{K} + C_0L^2\alpha^2\sigma^2 \right),
\end{align*}
where $C_0 = D^2+3B^2(D + 2D^3)$.

By Lemma \ref{the:ABCD_convergence_tool1},
\begin{align*}
    \frac{\sum_{k=0}^{K-1}E\|\nabla \mbf{f}(\mbf x^k)\|^2}{K} &\le  \frac{3n(B+C_0L^2 \alpha^2)}{\alpha(1 - (D+\frac{1}{2})L\alpha)} \frac{\mbf{f}(\mbf x^0) - \mbf{f}^*}{K} + \alpha \left(3nC_0L^2\alpha + \frac{L(D+1)}{2(1 - (D+\frac{1}{2})L\alpha)}\right)\sigma^2.
\end{align*}

% By the above relation,
% \begin{align*}
%     \frac{\sqrt{\frac{1}{\alpha} - (D + \frac{1}{2})L}}{\frac{1}{\alpha}+L}\|\E \nabla_i f(x^k)\| \le \sqrt{\frac{1}{\alpha} - (D + \frac{1}{2})L} \left(\|\E \Delta^{t_i(k)}\| + \|\E d^k\| + \sum_{j=t_i(k)}^{k-1}\|\E(\hat{\mbf{x}}^{j+1} - \hat{\mbf{x}}^{j})\|\right).
% \end{align*}

% By Lemma \ref{the:ABCD_convergence_tool1}, the first term on the RHS is square-summable.

% Moreover, since 
% \begin{equation*}
%     \|\E d^k\| \le \sum_{j=(k-D)^+}^{k-1}\|\E \Delta^j\|
% \end{equation*}
% and 
% \begin{equation*}
%     \sum_{j=t_i(k)}^{k-1}\|\E(\hat{\mbf{x}}^{j+1} - \hat{\mbf{x}}^{j})\| = \sum_{j=t_i(k)}^{k-1} \left( \|\E d^j\| + \|\E d^{j+1}\| + \|\E \Delta^j\|\right),
% \end{equation*}
% the second and the third are also square-summable.

% Therefore, $\frac{\sqrt{\frac{1}{\alpha} - (D + \frac{1}{2})L}}{\frac{1}{\alpha}+L}\|\E \nabla_i f(x^k)\|$ is square-summable and so is $\frac{\sqrt{\frac{1}{\alpha} - (D + \frac{1}{2})L}}{\frac{1}{\alpha}+L}\|\E \nabla f(x^k)\|$.

% Again by [Lemma 3, \cite{glowinski2017splitting}], we have 
% \begin{align*}
%     \min_{0\le k\le K} \|\E \nabla f(x^k)\| &= o\left(\frac{\frac{1}{\alpha} + L}{\sqrt{K(\frac{1}{\alpha} - (D+ \frac{1}{2})L)}}\right)\\
%     &=o\left(\frac{\frac{1}{\alpha} + L}{\sqrt{K(\frac{1}{\alpha} - L)}}\right)
% \end{align*}

\subsection{A corollary on Lemma \ref{the:ABCD_convergence_tool1}}
We provide a corollary on Lemma \ref{the:ABCD_convergence_tool1}, specifying its convergence rate under a specific step size.
\begin{corollary}
    \label{coro:ASBCD}
    For problem \ref{eq:optimization}, given Assumption \ref{asm:b-l_bounded-smooth} - \ref{asm:partialasynchrony}, when $\alpha = \frac{1}{2(D + 1/2)L\sqrt{K}}$, the sequence $\{\mbf x^k\}$ generated by (\ref{eq:ABCD_update}) satisfies the following relations:
    \begin{align*}
    &\frac{\sum_{k=0}^{K-1}\E\|\nabla \mbf{f}(\mbf x^k)\|^2}{K} \le nL C_1\frac{\mbf{f}(\mbf{x}^0) - \mbf{f}^*}{\sqrt{K}} + \left(\frac{1}{\sqrt{K}} + \frac{3nC_2}{4K} \right) \sigma^2,
\end{align*}
where $C_1 = 6B^2D^2 + 3B^2 + 3BD + 3B +D$ and $C_2 = 6B^2D + 3B^2/D + 1$.    
\end{corollary}

\begin{proof}
When $\alpha = \frac{1}{2(D+1/2)L\sqrt{K}}$, we have $L\alpha \le \frac{1}{2(D+1/2)}$ and $1-(D+1/2)L\alpha \le 1/2$.

Thus the first term on the RHS of Lemma \ref{the:ABCD_convergence} can be upper bounded by 
\begin{equation}
\label{eq:ub_RHS_ASBCD_1}
    3n\left(B + \frac{C_0}{3(D+1/2)^2}\right)(D+1/2)L \frac{\mbf{f}(\mbf{x}^0) - \mbf{f}^*}{\sqrt{K}}.
\end{equation}

Likewise, its second term can be upper bounded by
\begin{equation}
\label{eq:ub_RHS_ASBCD_2}
    \left(\frac{3nC_0}{4(D+1/2)^2K} + \frac{L(D+1)}{2(D+1/2)L\sqrt{K}} \right)\sigma^2.
\end{equation}

Substituting (\ref{eq:ub_RHS_ASBCD_1}) and (\ref{eq:ub_RHS_ASBCD_2}) to Lemma \ref{the:ABCD_convergence}, we have 
\begin{align}
\label{eq:coro_ASBCD_proof}
    \frac{\sum_{k=0}^{K-1}\E\|\nabla \mbf{f}(\mbf{x}^k)\|^2}{K} &\le  3n\left(B + \frac{C_0}{3(D+1/2)^2}\right)(D+\frac{1}{2})L \frac{\mbf{f}(\mbf{x}^0) - \mbf{f}^*}{\sqrt{K}}+\left(\frac{3nC_0}{4(D+1/2)^2K} + \frac{L(D+1)}{2(D+1/2)L\sqrt{K}} \right)\sigma^2\nonumber\\
    &\le 3nL\left(B(D+1) + \frac{C_0}{3D} \right) \frac{\mbf{f}(\mbf{x}^0) - \mbf{f}^*}{\sqrt{K}} + \left(\frac{3nC_0}{4D^2 K} + \frac{1}{\sqrt{K}}\right) \sigma^2\nonumber\\
    &\le nL C_1\frac{\mbf{f}(\mbf{x}^0) - \mbf{f}^*}{\sqrt{K}} + \left(\frac{1}{\sqrt{K}} + \frac{3nC_2}{4K} \right) \sigma^2,
\end{align}
where $C_1 = 6B^2D^2 + 3B^2 + 3BD + 3B +D$ and $C_2 = 6B^2D + 3B^2/D + 1$.
\end{proof}

\subsection{Proof of Theorem \ref{the:ADSGD}}
As mentioned, ADSGD with double step size $\{\alpha,\beta\}$ on problem \ref{eq:dec_opt} can be viewed as ABCD with step size $\beta$ on the function $L_\alpha(\mbf x) = F(\mbf x) + \frac{\mbf x^T(I-\mbf{W})\mbf x}{2\alpha}$, where $F(\mbf x) = \sum_{i=1}^n f_i(x_i)$.

W.L.O.G., we assume identical initialization. Note that $L_\alpha(\mbf x)$ is $L_L$-smooth. By Lemma \ref{the:ABCD_convergence}, when $\beta < \frac{1}{(D +\frac{1}{2})L_L}$,
\begin{align}
    \frac{\sum_{k=0}^{K-1}E\|\nabla L_\alpha(\mbf x^k)\|^2}{K} &\le \frac{3n(B+C_0L_L^2 \beta^2)}{\beta(1 - (D+\frac{1}{2})L_L \beta)} \frac{\sum_{i=1}^n(f_i(x^0) - f_i^*)}{K}\nonumber\\
    &+\beta \left(3nC_0L_L^2\beta + \frac{L_L(D+1)}{2(1 - (D+\frac{1}{2}) L_L\beta)}\right)\sigma^2\label{eq:ADSGD_proof_nab_L},
\end{align}
    
% \begin{align*}
%     \min_{0\le k\le K} \|\E \nabla L_\alpha(x)\| &= o\left(\frac{\frac{1}{\beta} + L_L}{\sqrt{K(\frac{1}{\beta} - L_L)}}\right)\\
%     &\overset{L_L\sim O(\frac{1}{\alpha})}{=}o\left(\frac{\frac{1}{\beta} + \frac{1}{\alpha}}{\sqrt{K(\frac{1}{\beta} - \frac{1}{\alpha})}}\right).
% \end{align*}

We now bound $\E\|\nabla f(\bar x^k)\|^2$ by the relations between $f, F,$ and $L_\alpha$.
\begin{align}
    \E \| \nabla f (\bar x^k)\|^2 &= \E\| \sum_{i=1}^n\nabla f_i(\bar x^k)\|^2\nonumber \\
    &\le 2\E\| \sum_{i=1}^n \nabla f_i(x_i^k)\|^2 + 2\E\| \sum_{i=1}^n (\nabla f_i(\bar x^k) - \nabla f_i(x_i^k))\|^2\nonumber\\
    &\le 2\E\| \sum_{i=1}^n \nabla f_i(x_i^k)\|^2  + 2n(\max_i L_i)^2 \E \|1_n \otimes \bar x^k - \mbf x^k\|^2\label{eq:ADSGD_bound_nab_f}
\end{align}

For the first term,
\begin{align}
    \E\| \sum_{i=1}^n \nabla f_i(x_i^k)\|^2 &=\E \| \sum_{i=1}^n \nabla_i F(\mbf x^k)\|^2\nonumber\\
    &= \E\|\sum_{i=1}^n(\nabla_i L_\alpha(\mbf x^k) - \sum_{j=1}^n [I-\mbf W]_{ij}x^k_j)\|^2\nonumber\\
    &= \E\|\sum_{i=1}^n \nabla_i L_\alpha(\mbf x^k)]\|^2\nonumber\\
    &\le n\E\|\nabla L_\alpha(\mbf x^k)\|^2\label{eq:nabf_nabL},
\end{align}
where the third equality is from the doubly stochasticity of $\mbf{W}$.

For the second term, since $1_n \otimes \bar x^k - \mbf x^k$ is in the range space of $I-\mbf W$,
\begin{align}
    \E\| 1_n \otimes \bar x^k - \mbf x^k\|^2 
    &\le \E\left[\frac{(\mbf x^k)^T(I-\mbf{W})\mbf x^k}{\lambda_{\min}(I-\mbf{W})}\right]\nonumber\\
    & = \frac{2\alpha}{1-\lambda_2(W)}\E\left[ L_\alpha(\mbf x^k) - F(\mbf x^k)\right]\nonumber\\
    &\le \frac{2\alpha}{1-\lambda_2(W)}\E\left[ \xi_{L_\alpha}^k - F^*\right]\nonumber\\
    &\overset{\eqref{eq:ABCD_lemma_proof_xi_descent}}{\le} \frac{2\alpha}{1-\lambda_2(W)}\left(\sum_{i=1}^n (f_i(x^0) - f_i^*) + \frac{D+1}{2} KL_L\beta^2\sigma^2\right)\label{eq:concensus_err},
\end{align}
where $\lambda_{\min}(\cdot)$ is the minimal positive eigenvalue. 

Bringing \eqref{eq:concensus_err} and \eqref{eq:nabf_nabL} back to \eqref{eq:ADSGD_bound_nab_f}, 
\begin{align*}
    \E \|\nabla f(\bar x^k)\|^2 &\le 2n\E\| \nabla L_\alpha(\mbf x)\|^2 + \frac{4n(\max_i L_i)^2\alpha}{1-\lambda_2(W)}\left(\sum_{i=1}^n (f_i( x^0) - f_i^*) + \frac{D+1}{2} KL_L\beta^2\sigma^2\right)
\end{align*}

Sum the above over $k$ and combine with \eqref{eq:ADSGD_proof_nab_L},

\begin{align*}
    \frac{\sum_{k=0}^{K-1}E\|\nabla f(\bar x^k)\|^2}{K} &\le \frac{6n^2(B+C_0L_L^2 \beta^2)}{\beta(1 - (D+\frac{1}{2})L_L\beta)} \frac{\sum_{i=1}^n(f_i(x^0) - f_i^*)}{K}\nonumber\\
    &+ L_L \beta\left(6n^2C_0L_L\beta + \frac{n(D+1)}{1 - (D+\frac{1}{2})L_L\beta}\right)\sigma^2\\
    &+ \frac{4n(\max_i L_i)^2\alpha}{1-\lambda_2(W)}\left(\sum_{i=1}^n (f_i(x^0) - f_i^*) + \frac{D+1}{2} KL_L\beta^2\sigma^2\right).
\end{align*}

\subsection{Proof of Corollary \ref{coro:ADSGD}}
When $\alpha=\frac{2}{L_F K^{1/3}}, \beta=\frac{1}{4L_F(D+1/2)K^{2/3}}$, we have 
\begin{align*}
    L_L \beta &\le (L_F + \frac{2}{\alpha})\beta\\
    &\le \frac{1}{4(D+1/2)K^{2/3}} + \frac{1}{8(D+1/2)K^{1/3}} \\
    &\le \frac{1}{2(D+1/2)K^{1/3}}
\end{align*}

Thus, the first term on the RHS of Theorem \ref{the:ADSGD} can be upper bounded by
\begin{equation}
\label{eq:ADSGD_coro_ub1}
    48n^2 \left(B+\frac{C_0}{4(D+1/2)^2}\right)(D+1/2)L_F\frac{\sum_{i=1}^n(f_i(x^0)-f_i^*)}{K^{1/3}}.
\end{equation}
Similarly, its second and third terms can be respectively upper bounded by
\begin{equation}
\label{eq:ADSGD_coro_ub2}
    \left(\frac{3n^2C_0}{2(D+1/2)K^{2/3}} + \frac{n(D+1)}{(D+1/2)K^{1/3}}\right) \sigma^2
\end{equation}
and
\begin{equation}
\label{eq:ADSGD_coro_ub3}
    \frac{8n(\max_i L_i)^2}{(1-\lambda_2(W)) L_F K^{1/3}} \left(\sum_{i=1}^{n} (f_i(x^0)) - f_i^*) + \frac{D+1}{16(D+1/2)^2 L_F} \sigma^2\right).
\end{equation}

Taking \eqref{eq:ADSGD_coro_ub1}, \eqref{eq:ADSGD_coro_ub2}, and \eqref{eq:ADSGD_coro_ub3} into Theorem \ref{the:ADSGD}, we have
\begin{align*}
    \frac{\sum_{k=0}^{K-1} \E\|\nabla f(\bar x^k)\|^2}{K} &\le 48n^2 \left(B+\frac{C_0}{4(D+1/2)^2}\right)(D+1/2)L_F\frac{\sum_{i=1}^n(f_i(x^0)-f_i^*)}{K^{1/3}} \\
    &+  \left(\frac{3n^2C_0}{2(D+1/2)K^{2/3}} + \frac{n(D+1)}{(D+1/2)K^{1/3}}\right) \sigma^2\\
    &+ \frac{8n(\max_i L_i)^2}{(1-\lambda_2(W)) L_F K^{1/3}} \left(\sum_{i=1}^{n} (f_i(x^0)) - f_i^*) + \frac{D+1}{16(D+1/2)^2 L_F} \sigma^2\right)\\
    &\le \left(48n^2(B(D+1)+\frac{C_0}{3D}) + \frac{8n}{1-\lambda_2(W)}\right) L_F \frac{\sum_{i=1}^n(f_i(x^0)-f_i^*)}{K^{1/3}} \\
    &+ \left(\frac{n}{D(1-\lambda_2(W))} + 2n\right)\frac{\sigma^2}{K^{1/3}} + \frac{3n^2C_0}{2D}\frac{\sigma^2}{K^{2/3}}\\
    &\le \left(16n^2C_1 + \frac{8n}{1-\lambda_2(W)}\right) L_F \frac{\sum_{i=1}^n(f_i(x^0)-f_i^*)}{K^{1/3}} \\
    &+ \left(\frac{n}{D(1-\lambda_2(W))} + 2n\right)\frac{\sigma^2}{K^{1/3}} + \frac{3n^2C_0}{2D}\frac{\sigma^2}{K^{2/3}},
\end{align*}
where $C_1$ is defined in \eqref{eq:coro_ASBCD_proof}.

% \section{Proof of Old Corollary \ref{coro:ADSGD}}
% Since $L_L = \sum_{i=1}^n L_i + \frac{1-\lambda_n}{\alpha}$, when $\frac{\beta}{\alpha}\sim o(1)$, 
% \begin{equation*}
%     \frac{6n^2(B+C_0L_L^2 \beta^2)}{\beta(1 - (D+\frac{1}{2})L_L)\beta} = O(\frac{1}{\beta}),
% \end{equation*}
% and
% \begin{align*}
%     L_L \beta\left(6n^2C_0L_L\beta + \frac{n(D+1)}{1 - (D+\frac{1}{2})L_L\beta}\right)=O(\frac{\beta}{\alpha} ).
% \end{align*}

% Moreover, when $\frac{\beta^2}{\alpha}\le O(\frac{1}{K})$, 
% \begin{equation*}
%     \frac{4n(\max_i L_i)^2\alpha}{(1-\lambda_2(W))}\left(\sum_{i=1}^n (f_i( x^0) - f^*) + \frac{D+1}{2} KL_L\beta^2\sigma^2\right) = O(\alpha).
% \end{equation*}

% Therefore, given $\frac{\beta}{\alpha}\sim o(1)$, $\frac{\beta^2}{\alpha}\le O(\frac{1}{K})$ and all conditions in Theorem \ref{the:ADSGD},
% \begin{equation*}
%     \frac{\sum_{k=0}^{K-1}E\|\nabla f(\bar x^k)\|^2}{K} \sim O(\frac{1}{K\beta} + \frac{\beta}{\alpha} + \alpha).
% \end{equation*}

% Specifically, given $\alpha = \frac{1}{K^{1/3}}$ and $\beta = \frac{1}{K^{2/3}}$ and $\beta<\frac{1}{(D+1/2)L_L}$, 
% \begin{equation*}
%     \frac{\sum_{k=0}^{K-1}E\|\nabla f(\bar x^k)\|^2}{K} \sim O(\frac{1}{K^{1/3}}).
% \end{equation*}

% The last condition holds when $K>(2D + 1)^3$.
%%%%%%%%%%%%%%%%%%%%%%%%%%%%%%%%%%%%%%%%%%%%%%%%%%%%%%%%%%%%%%%%%%%%%%%%%%%%%%%
%%%%%%%%%%%%%%%%%%%%%%%%%%%%%%%%%%%%%%%%%%%%%%%%%%%%%%%%%%%%%%%%%%%%%%%%%%%%%%%
%%%%%%%%%%%%%%%%%%%%%%%%%%%%%%%%%%%%%%%%%%%%%%%%%%%%%%%%%%%%

\newpage

\end{document}